\PassOptionsToPackage{svgnames}{xcolor}
\documentclass[12pt]{article}
\usepackage{amsmath}
\usepackage{graphicx}
\usepackage{enumerate}
\usepackage{natbib}
\usepackage{graphicx} 

\usepackage[left=1in,top=1.in,right=1in,bottom=1.in,head=.2in,nofoot]{geometry}

\usepackage[utf8]{inputenc}
\usepackage[bookmarksopen=true,bookmarks=true,pdfencoding=auto,psdextra,colorlinks]{hyperref}
\hypersetup{ 
    colorlinks,
    linkcolor={blue!80!black},
    citecolor={green!40!black},
    urlcolor={red!50!black}
}
\usepackage{bookmark}

\usepackage{mystyle}
\usepackage{natbib}
\usepackage{bibunits}

\allowdisplaybreaks
\usepackage{xspace}

\usepackage[font=small,labelfont=bf]{caption}
\usepackage{tabularx}
\usepackage{multirow}
\usepackage{array}
\newcolumntype{Y}{>{\centering\arraybackslash}X}
\newcolumntype{P}{>{\raggedleft\arraybackslash}X}
\newcolumntype{C}{ >{\centering\arraybackslash} m{4cm} }
\newcolumntype{D}{ >{\raggedright\arraybackslash} m{10.5cm} }

\usepackage{algorithm}
\usepackage{algpseudocode}
\usepackage{chngcntr}

\usepackage{tikz}

\newcommand{\footremember}[2]{%
    \footnote{#2}
    \newcounter{#1}
    \setcounter{#1}{\value{footnote}}%
}
\newcommand{\footrecall}[1]{%
    \footnotemark[\value{#1}]%
}
\usepackage{setspace}
\usepackage{wrapfig}

\usepackage{etoolbox}
\colorlet{shadecolor}{lightgray!12}
\usepackage{framed}

\newif\ifaftersection
\aftersectionfalse
\pretocmd{\section}{\aftersectiontrue}{}{}
\pretocmd{\subsection}{\aftersectiontrue}{}{}

\renewenvironment{quote}
  {\list{}{\rightmargin=.0cm \leftmargin=.0cm}%
   \item\relax}
  {\endlist}
\usepackage[xcolor,leftbars]{changebar}
{\begin{quote}%
  \begin{changebar}\cbcolor{lightgray}\color{Teal}}%
  {\end{changebar}%
\end{quote}}

\title{Disentangled Feature Importance}
\author{Jin-Hong Du\footremember{hkusaas}{Department of Statistics and Actuarial Science, The University of Hong Kong, Hong Kong SAR, China.}\footremember{hkuids}{Musketeers Foundation Institute of Data Science, The University of Hong Kong, Hong Kong SAR, China.}
\and 
Kathryn Roeder\footremember{cmustats}{Department of Statistics and Data Science, Carnegie Mellon University, Pittsburgh, PA 15213, USA.}\footremember{cmucbd}{Computational Biology Department, Carnegie Mellon University, Pittsburgh, PA 15213, USA.}
\and 
Larry Wasserman\footrecall{cmustats} \footremember{cmumld}{Machine Learning Department, Carnegie Mellon University, Pittsburgh, PA 15213, USA.}
}
\date{\today}

\begin{document}

\maketitle

\begin{abstract}
    When predictors are statistically dependent, the appropriate definition of feature importance depends on the operational goal.
    Conditional-incremental measures are well-suited for feature selection, acquisition, and compression, where shared predictive information is treated as redundancy. 
    For post-hoc interpretation, however, the goal is often to attribute predictive signals across correlated measurement channels. 
    We introduce Disentangled Feature Importance (DFI), a population-level attribution framework for this setting. DFI maps covariates to an independent latent representation under a specified entropic optimal transport geometry, computes latent importance, and attributes it back to the original covariates through barycentric sensitivities. We show that broad conditional-incremental FI functionals target conditional incremental predictive value under squared-error loss, and therefore answer a different question from attribution of shared predictive signal under dependence. Under fixed transport cost, reference law, and regularization level, DFI defines a well-specified family of estimands. Latent scores admit a functional ANOVA interpretation, and in the Gaussian linear case, the attributed DFI recovers the classical \(R^2\) decomposition for correlated regressors. 
    We derive influence-function-based inference under nuisance-rate and smoothness conditions, and show in simulations and an HIV-1 neutralization-resistance analysis that DFI yields stable, interpretable, uncertainty-quantified attributions of shared predictive signal.
\end{abstract}

\section{Introduction}\label{sec:intro}

    Feature importance quantification lies at the heart of interpretable machine learning and statistics, yet dependence among predictors makes the target estimand inherently goal-dependent. 
    For feature selection, acquisition, or compression, it is often appropriate to treat shared predictive information as redundancy; for post-hoc interpretation of a fitted predictor, however, one may instead want to attribute predictive signals across the observed measurement channels through which the model operates \citep{owen2017shapley,verdinelli2024feature}. 
    Such attribution is useful for auditing modern black-box models, monitoring which correlated inputs carry predictive signal, diagnosing failure modes or proxy behavior, and communicating model behavior in scientific applications.
    This distinction is critical in domains ranging from genomics \citep{du2025causal,du2025causarray} to language-model interpretation \citep{kang2025spex}, where features naturally exhibit complex dependencies and several observed variables may reflect the same underlying signal.

    Consider the regression task of predicting an outcome $Y \in \RR$ from covariates $X \in \RR^d$. 
    The regression function \(f(x)=\EE[Y\mid X=x]\) captures the population predictive relationship, while a fitted black-box predictor \(\hat f\) provides an empirical approximation or deployed model whose behavior motivates interpretation.
    Feature importance (FI) methods aim to decompose predictive signal across the input dimensions, revealing which features contribute most to prediction.
    The literature distinguishes between \emph{statistical feature importance}, which measures each feature's role in the population regression function \(f\), and \emph{algorithmic feature importance}, which quantifies contributions to a specific fitted predictor \(\hat f\).
    We adopt the statistical perspective: \(\hat f\) may be used to estimate \(f\), but the target estimands are population-level quantities determined by the joint law of \((X,Y)\).

    Rather than modifying conditional-incremental methods to answer a different question, we introduce a complementary FI estimand for post-hoc attribution under dependence.
    The central idea is to compute importance after passing to an independent latent representation, and then attribute the resulting predictive signal back to the observed covariates.
    Our framework operates in three steps.
    First, under a specified disentangling geometry, we map the original covariates \(X\) to a representation \(Z\sim K(\cdot\mid X)\) with independent coordinates via a Markov kernel \(K\).
    Second, we compute FI in this latent space, where resampling-based importance targets are not distorted by dependence among the latent coordinates.
    Third, we attribute the latent importance back to the original covariates through the sensitivities of the barycentric projection.
    The resulting scores should be interpreted as feature-importance attribution scores under correlation, not as measures of additional predictive utility for feature selection.
    This disentanglement strategy is motivated by a classical result in statistics: the decomposition of \(R^2\) for linear regression with correlated covariates.

    \begin{example}[Linear regression with correlated covariates]\label{ex:R2}
        Consider a multiple linear model:
        \begin{align}
            Y = \beta^{\top} X + \epsilon \label{model:linear}
        \end{align}
        with $X\sim \cN_d(0, \Sigma)$ for some covariance matrix $\Sigma\in\SSS_+^d$ and $\epsilon$ being an independent noise of $X$.
        Without loss of generality, suppose $Y$ and $X_j$'s all have zero mean and unit variance. 
        In this case, $\Sigma$ is also the correlation matrix of $X$, and $\Sigma_{jj}=\sum_{l=j}^d(\Sigma^{\frac{1}{2}})_{jl}^2 = 1$ for $j=1,\ldots,d$.
        The data model can be characterized through two parameters ($\Sigma$ and $\beta$) and the \emph{coefficient of determination} is given by $R^2 = \beta^{\top}\Sigma\beta$.
        \citet{genizi1993decomposition} showed a decomposition:
        \[
            R^2 = \sum_{j=1}^dc_j^2 = \sum_{l=1}^dv_l^2,
        \]
        where $c_j^2=(e_j^{\top}\Sigma^{\frac{1}{2}}\beta)^2$ is the component assigned to the decorrelated feature $Z_j$ and $v_l^2 = \sum_{j=1}^d(\Sigma^{\frac{1}{2}})_{jl}^2c_j^2$ is the component assigned to feature $X_l$.
    \end{example}

    In the above example, $R^2$ also represents the intrinsic variation $\VV[\beta^{\top}X]$.
    Hence, the decomposition attributes each feature an importance that sums up to the intrinsic variation.
    Our disentangled FI framework extends this classical decomposition to nonparametric settings, providing a principled way to attribute FI in the presence of arbitrary dependencies.

\subsection{Motivating applications}
    The applications below share a common structure: a flexible predictive model is trained on correlated measurement channels, and the scientific or operational goal is to understand how predictive signal is expressed through those observed channels. This is distinct from selecting a minimal nonredundant subset of covariates. When several features measure related biological, linguistic, clinical, or behavioral factors, an attribution analysis may legitimately assign importance to multiple correlated inputs, because each may carry or reflect the same predictive signal in the fitted prediction problem.

    \begin{example}[Genomics and systems biology]
        In genetic perturbation and systems-biology studies, modern deep learning models and foundation models predict post-perturbation expression or cellular state from high-dimensional molecular measurements \citep{du2022robust,moon2025augmented,cui2024scgpt}. Genes often exhibit strong co-expression and pathway-level dependence, so the interpretive question is not only which gene would improve prediction beyond all others, but which genes or gene groups carry predictive signal through correlated biological modules.
    \end{example}

    \begin{example}[Natural language processing]
        In text-generation and language-understanding tasks, input tokens, phrases, or embeddings are strongly dependent through syntax, semantics, and context. Attributing a model output to the observed input components is central to interpreting black-box language models \citep{kang2025spex}, but conditional incremental importance may treat predictable tokens or phrases as redundant even when they participate in the model's predictive signal.
    \end{example}

    \begin{example}[Causal inference]
        After identification assumptions such as unconfoundedness are imposed, fitted models of conditional treatment effects, \(\tau(x)=\EE[Y(1)-Y(0)\mid X=x]\), are often interpreted to understand treatment-effect heterogeneity \citep{hines2022variable}. When clinical, demographic, or socioeconomic covariates are correlated, the relevant post-hoc question may be which observed covariates carry heterogeneity signal in the fitted predictor, rather than which covariates provide unique incremental value for selecting a sparse adjustment or targeting rule.
    \end{example}

\subsection{Contributions}
    Our work makes the following contributions.

\begin{itemize}

    \item \textbf{Estimand clarification under dependence.} In \Cref{sec:FI}, we clarify that different FI estimands answer different operational questions. Under squared-error loss, LOCO and CPI target the same population functional, \( \EE\!\left[\VV\{\EE[Y\mid X]\mid X_{-j}\}\right], \) and therefore quantify the conditional incremental value of \(X_j\) beyond \(X_{-j}\) (\Cref{lem:null-feature,lem:equiv}). This target is well aligned with feature selection, acquisition, or compression, but differs from feature-importance attribution, where the goal is to describe how predictive signal is carried by observed covariates under dependence.
    
    \item \textbf{Disentangled Feature Importance for attribution.} In \Cref{sec:DFI}, we introduce \emph{Disentangled Feature Importance (DFI)}, a population FI estimand for post-hoc attribution under correlated inputs. DFI first maps the observed covariates to an independent latent representation through a specified disentangling kernel, computes importance in the latent space, and then attributes latent importance back to the original covariates through barycentric sensitivities. The resulting scores should be interpreted as attribution scores under the chosen disentangling geometry, not as measures of unique incremental predictive utility.

    \item \textbf{Semiparametric inference.} In \Cref{sec:stat-est-inf}, we develop inference theory for both latent and attributed DFI. For entropic optimal transport (EOT) kernels, we derive efficient influence-function representations and establish asymptotic normality under second-order remainder terms (\Cref{thm:error-decom-eot,thm:error-decom-X-eot}). The remainders are negligible when the regression function, transport kernel, and barycentric sensitivity estimators converge at \(\op(n^{-1/4})\) rates. The same framework also covers transport-map special cases, including Bures--Wasserstein and Knothe--Rosenblatt maps (\Cref{app:sec:ot-maps}).

    \item \textbf{Nonparametric extension of the classical \(R^2\) decomposition.} In \Cref{subsec:R2-extension}, we show that latent DFI admits a functional ANOVA interpretation: the latent scores aggregate all variance components involving each latent coordinate. For additive regression functions in the disentangled space, the scores sum to the total predictive variability; more generally, they sum to interaction-order-weighted ANOVA variances (\Cref{lem:latent-importance-decomp} and \Cref{prop:decomp-phi}). In the Gaussian linear case, attributed DFI recovers the classical \(R^2\) decomposition of \citet{genizi1993decomposition}, providing a nonparametric extension to nonlinear regression under dependent covariates.
    
    \item \textbf{Computational and empirical assessment.} DFI avoids repeated reduced-model fitting and conditional covariate distribution estimation, two major computational bottlenecks in LOCO- and CPI-type procedures. Through simulations in \Cref{sec:simu} and an HIV-1 neutralization-resistance analysis in \Cref{sec:real-data}, we show that DFI provides stable feature-importance attribution under correlation, supports uncertainty quantification, and remains computationally practical.
\end{itemize}

\subsection{Notation}

For an integer $d\in\NN$, define $[d] = \{1,2,\ldots,d\}$.
For a vector $X\in \RR^d$, $X_{\cS}\in \RR^{|\cS|}$ denotes a sub-vector of $X$ indexed by $\cS\subseteq[d]$, and $X_{-j} := X_{[d]\setminus j}$. The $j$-th standard basis vector is denoted by $e_j$. The cardinality of a set $\cS$ is denoted by $|\cS|$. The indicator function is denoted by $\ind\{\cdot\}$.
For a symmetric matrix $\Sigma$, its (unique) positive semi-definite square root and its inverse are denoted by $\Sigma^{\frac{1}{2}}$ and $\Sigma^{-\frac{1}{2}}$, respectively. The set of symmetric and positive definite $d \times d$ matrices is denoted by $\SSS_+^d$. The Löwner order is denoted by $\preceq$; that is, $A \preceq B$ means that $B-A$ is positive semi-definite. The operator norm of a matrix is denoted by $\|\cdot\|_{\oper}$, and its trace is denoted by $\tr(\cdot)$.


For a tuple of random vectors $O =(X,Y)$, the expectation and probability over the joint distribution $\PP$ are denoted by $\EE(\cdot)$ and $\PP(\cdot)$, respectively. The marginal measure of $X$ is denoted by $\PP_X$, and its cumulative distribution function is $F_X$. For a measurable space $(\Omega,\cF)$, let $\cP(\Omega)$ denote the set of all probability measures on $(\Omega,\cF)$.
For a statistical estimand $\phi$, we write $\phi(\PP)$ to emphasise its dependence on the underlying distribution $\PP$. 
For (potentially random) measurable functions $f$, we denote expectations with respect to the data-generating distribution of $O$ alone by $\PP f(O) = \int f \rd\PP$, while $\EE [f(O)]$ marginalizes out all randomness from both $O$ and any nuisance functions that $f$ is dependent on. The empirical expectation over $n$ samples is denoted by $\PP_n f(O) = \frac{1}{n}\sum_{i=1}^nf(O_i)$. The $L_2$ norm of a function with respect to a measure $\PP_X$ is denoted by $\|\cdot\|_{L_2(\PP_X)}$.
The population and empirical variances are denoted by $\VV$ and $\VV_n$. Statistical independence is denoted by $\indep$. The $d$-dimensional multivariate normal distribution with mean $\mu$ and covariance $\Sigma$ is denoted by $\cN_d(\mu, \Sigma)$.
We use ``$o$'' and ``$\cO$'' to denote the little-o and big-O notations; ``$\op$'' and ``$\Op$'' are their probabilistic counterparts. For sequences $\{a_n\}$ and $\{b_n\}$, we write $a_n\lesssim b_n$ if $a_n=\cO(b_n)$; and $a_n\asymp b_n$ if $a_n=\cO(b_n)$ and $b_n=\cO(a_n)$. 
Convergence in distribution is denoted by ``$\dto$''.

\section{Feature importance measure}\label{sec:FI}

Feature importance has been explored from several complementary perspectives that differ in both their inferential targets and the assumptions they impose on the data-generating process.  Broadly speaking, \emph{algorithmic} approaches, including Shapley value and its fast approximations \citep{shapley1953value,lundberg2017unified}, attribute the prediction of a fixed (possibly black-box) model $f$ to its input coordinates; the distribution of $(X,Y)$ appears only through the empirical sample on which $f$ was trained.  
In contrast, \emph{statistical} or \emph{population} approaches regard $f$ as the true regression function and seek functionals of the joint law of $(X,Y)$ that decompose predictive risk.

\providecommand{\tabyes}{\ensuremath{\checkmark}}
\providecommand{\tabno}{\ensuremath{\times}}
\providecommand{\tabpart}{$\circ$}
\providecommand{\costL}{$\ast$}
\providecommand{\costM}{$\ast\ast$}
\providecommand{\costH}{$\ast\ast\ast$}

\begin{table}[t]
    \centering
    \scriptsize
    \begin{tabularx}{\textwidth}{@{}
        >{\raggedright\arraybackslash}m{2.1cm}
        >{\raggedright\arraybackslash}m{3.5cm}
        >{\centering\arraybackslash}p{0.07\textwidth}
        >{\centering\arraybackslash}p{0.07\textwidth}
        >{\centering\arraybackslash}p{0.07\textwidth}
        >{\centering\arraybackslash}p{0.07\textwidth}
        >{\centering\arraybackslash}p{0.06\textwidth}
        >{\centering\arraybackslash}p{0.105\textwidth}
    @{}}
    \toprule
    \textbf{Method} & \textbf{Reference}
    & \textbf{Inference}
    & \textbf{Submodel fit}
    & \textbf{Cond. dist.}
    & \textbf{Sampling}
    & \textbf{Cost}
    & \textbf{No corr. distortion} \\
    \midrule
    
    LOCO / nLOCO &
    \citep{lei2018distribution,rinaldo2019bootstrapping,verdinelli2024feature}
    & \tabyes
    & \tabyes
    & \tabno
    & \tabno
    & \costH
    & \tabno \\
    
    Marginal PFI / MDA &
    \citep{hooker2021unrestricted,benard2022mean,gan2022model}
    & \tabpart
    & \tabno
    & \tabno
    & \tabyes
    & \costL
    & \tabno \\
    
    CPI / conditional PFI &
    \citep{strobl2008conditional,chamma2023statistically,reyero2025sobol}
    & \tabyes
    & \tabno
    & \tabyes
    & \tabyes
    & \costH
    & \tabno \\
    
    Shapley VIM / SPVIM &
    \citep{shapley1953value,owen2017shapley,williamson2020efficient,williamson2023general}
    & \tabyes
    & \tabyes
    & \tabpart
    & \tabpart
    & \costH
    & \tabpart \\
    
    Sobol / fANOVA &
    \citep{sobol1990sensitivity,owen2017shapley,herbinger2024decomposing}
    & \tabpart
    & \tabno
    & \tabpart
    & \tabyes
    & \costM
    & \tabpart \\
    
    Projected covariance measure &
    \citep{lundborg2024projected}
    & \tabyes
    & \tabpart
    & \tabno
    & \tabno
    & \costM
    & \tabno \\
    
    dLOCO &
    \citep{verdinelli2024decorrelated}
    & \tabpart
    & \tabyes
    & \tabyes
    & \tabyes
    & \costH
    & \tabpart \\
    
    DFI & this paper
    & \tabyes
    & \tabno
    & \tabno
    & \tabyes
    & \costM
    & \tabyes \\\bottomrule
    \end{tabularx}
    \caption{Comparison of population-level feature-importance functionals. 
    Symbols: \(\checkmark\) = yes, \(\times\) = no, and \(\circ\) = partial or formulation-dependent; cost is qualitative, from \(*\) low to \(*{*}{*}\) high.
    ``Submodel fit'', ``Cond. dist.'', and ``Sampling'' indicate whether score construction requires reduced-model fitting, conditional covariate simulation, or permutation/Monte Carlo sampling, respectively.
    ``No corr. distortion'' refers to post-hoc attribution of predictive signal under dependent covariates, not to feature-selection validity.
    }\label{tab:stat-fi-method-comparison}
\end{table}

Classical representatives include leave-one-covariate-out (LOCO) and its Shapley-value reinterpretation \citep{verdinelli2024feature}, Conditional Permutation Importance (CPI) and its block- or Sobol-style refinements \citep{chamma2023statistically,reyero2025sobol,chamma2024variable}, as well as doubly robust, influence-function based estimators of variable importance (VIM, SPVIM) under general loss functions \citep{williamson2021nonparametric,williamson2023general,williamson2020efficient}.  
For causal inference, extensions to local FI recast these ideas, either through LOCO-type refitting \citep{hines2022variable,dai2024moving,lundborg2024projected} or CPI-type permutation schemes that circumvent nuisance estimation \citep{paillard2025measuring}.
A parallel literature in global sensitivity analysis \cite{sobol1990sensitivity} distributes variance via functional ANOVA or Sobol indices of independent inputs and has recently been connected to partial dependence \citep{owen2017shapley,herbinger2024decomposing}.
See \Cref{tab:stat-fi-method-comparison} for a summary of statistical FI functionals and their computational requirements.
For literature review, the readers are referred to \citet{hooker2021unrestricted}.

Three conceptually simple yet widely used importance measures, LOCO, CPI, and Shapley value, serve as building blocks for many of the above methods.  
The remainder of this section formalizes their definitions and highlights their connections.
Consider a regression problem when $Y\in\RR$ is the response and $X\in\RR^d$ is the feature.
Let $\mu(x) = \EE[ Y \mid X=x]$ denote the regression function using all features, and $\mu_{\cS}(x_{\cS}) := \EE[Y \mid X_{\cS}=x_{\cS}]$ denote the regression function using a subset of features $X_{\cS}$ for any $\cS\subseteq[d]$.
In particular, for a given coordinate $j \in [d]$, define the regression function of $Y$ given $X_{-j}$ as $\mu_{-j}(x_{-j}) = \EE[ Y \mid X_{-j}=x_{-j}]$.

\subsection{LOCO}
    
    Leave-One-Covariate-Out (LOCO) metric and assumption-lean inference framework such as conformal prediction for LOCO have been studied by \citet{lei2018distribution,rinaldo2019bootstrapping}.
    Consider a loss function $\ell$ (e.g., the negative squared error loss, classification loss, etc.), the LOCO parameter for $X_j$ is defined as
    \[ \psi^{\loco}_{X_j}= \EE[ \ell(\mu_{-j}(X_{-j}),Y) - \ell(\mu(X),Y)]
    ,\]
    which quantifies the change of expected loss when dropping $X_j$ from the full model.
    Up to scaling, $\psi_{loco}$ is a nonparametric version of the usual $R^2$ from standard regression.

    Consider fitting estimated regression function $\hat{\mu}$ and $\hat{\mu}_{-j}$ from $m$ independent and identically distributed samples of $(X,Y)$.
    We then compute an estimate of the LOCO parameter on $n$ additional samples of $(X,Y)$:
    \[ \hat{\psi}^{\loco}_{X_j} = \ell(\hat{\mu}_{-j}(X_{-j}),Y) - \PP_n\{ \ell(\hat{\mu}(X),Y) \}. \]
    For LOCO, there are two nuisance functions $(\mu,\mu_{-j})$ to be estimated from data.

\subsection{CPI}
    The traditional permutation importance approach ignores the correlation among features and may lead to uncontrolled type-1 errors.
    To mitigate this issue, conditional permutation importance has been proposed in different contexts:
    for random forests \citep{strobl2008conditional} and for general predictors under $\ell_2$ loss \citep{chamma2023statistically} and general loss functions \citep{reyero2025sobol}, by sampling from the conditional distribution $p(X_j \mid X _{-j})$.
    It can also be interpreted as replace-one-covariate or swap-one-covariate FI.

    Define a random vector $X^{(j)}$ such that $X^{(j)}_{-j}=X_{-j}$ and $X^{(j)}_j \sim p(X_j\mid X_{-j})$ is an independent copy of $X_j$ given all the other covariate $X_{-j}$, which is also independent of the outcome $Y$.
    The CPI parameter is then defined as:
    \[ 
    \psi^{\cpi}_{X_j} = \frac{1}{2}\EE[ \ell(\mu(X^{(j)}),Y) - \ell(\mu(X),Y)],
    \]
    Similarly, we can compute an estimate of the CPI parameter as:
    \[ \hat{\psi}^{\cpi}_{X_j} = \frac{1}{2}\PP_n\{ \ell(\hat{\mu}(X^{(j)}),Y) - \ell(\hat{\mu}(X),Y) \}, \]
    with $\hat{X}^{(j)}_j \sim \hat{p}(X_j \mid X_{-j})$.

    Instead of estimating the conditional density function $\hat{p}(X_j \mid X_{-j})$ directly, we can sample $X^{(j)}_j$ by estimating the regression function $\EE[X_j \mid X_{-j}]$ and bootstrapping the residuals. This approach requires that $X_j$ can be written as $X_j=\nu_j (X_{-j} ) + \epsilon_j$ where $\epsilon_j \indep X_{-j}$ and $\EE[\epsilon_j ] = 0$ \citep[Lemma 3.2]{reyero2025sobol}. Under this additive structure, one draws an independent sample $W$ of $X$ and sets $\hat{X}^{(j)}_j = \hat{\nu}_j(X_{-j}) + (W_j - \hat{\nu}_j(W_{-j}))$. While this independence assumption is restrictive in general, it holds naturally when features are independent or when covariates follow a multivariate Gaussian distribution. For CPI, this approach requires estimating two nuisance functions: the regression functions of the response and the $j$th covariate $ (\mu, \nu_{j})$.

\subsection{Shapley value}
    Originating from cooperative game theory \citep{shapley1953value}, the \emph{Shapley value} attributes a model’s predictive ability to individual covariates.
    Let $[d]=\{1,\dots,d\}$ index the features and let $v(\cS)$ be a value function that measures the expected utility obtained when only the subset $\cS\subseteq[d]$ is available.
    Throughout this paper, we adopt the squared error as the value function:
    \[
        v(\cS)\;:=\;\EE\left[\ell\left(\mu_{\cS}(X_{\cS}),Y\right)\right],
        \qquad
        \ell(\hat y,y)=(\hat y-y)^{2},
    \]
    where $\mu_{\cS}(x_{\cS})=\EE[Y\mid X_{\cS}=x_{\cS}]$ is the oracle regression trained on $\cS$ alone.
    The Shapley importance of feature $j$ is
    \[
        \psi^{\shap}_{X_j}
        \;=\;
        \sum_{\cS\subseteq[d]\setminus\{j\}}
        \frac{|\cS|!\,(d-|\cS|-1)!}{d!}\,
        \left\{\,v(\cS\cup\{j\})-v(\cS)\right\},
    \]
    or, equivalently, the expectation over a random permutation $\pi$ of the incremental gain achieved when $j$ enters the model immediately after its predecessors in~$\pi$.
    The vector $(\psi^{\shap}_{X_1},\dots,\psi^{\shap}_{X_d})$ is the unique attribution satisfying the four canonical Shapley axioms (efficiency, symmetry, dummy, additivity).
    
    As shown by \citet{verdinelli2024feature}, each term $v(\cS\cup\{j\})-v(\cS)$ coincides with LOCO for subset $\cS$, i.e., $\psi^{\loco}_{X_j}(\cS\cup\{j\}) :=  \EE[ (\mu_{\cS\cup\{j\}}(X_{\cS\cup\{j\}}) - \mu_{\cS}(X_{\cS}) )^2 ]$ under $\ell_2$ loss.
    Hence, the Shapley value is a weighted average of LOCO scores across all $2^{d-1}$ submodels: 
    \[
        \psi^{\shap}_{X_j}
        \;=\;
        \sum_{\cS\subseteq[d]\setminus\{j\}}
        w_{\cS}\,
        \psi^{\loco}_{X_j}(\cS\cup\{j\}),
        \qquad
        w_{\cS}:=\frac{(|\cS|+1)!\,(d-|\cS|)!}{d!}.
    \]
    Even under linear models, this averaging yields a highly non-linear functional of the joint distribution of ${(X,Y)}$.

    Exact evaluation of $\psi^{\shap}_{X_j}$ is infeasible beyond very small $d$; practical implementation relies on submodel subsampling \citep{williamson2020efficient}.
    Because the value function itself must be estimated for every sampled coalition, the resulting estimator is both statistically and computationally intensive, complicating further tasks such as variance estimation and formal inference.
    These difficulties motivate the search for importance measures that permit computationally tractable estimation and valid uncertainty quantification.
    For these reasons, we focus primarily on analyzing the properties of LOCO and CPI in the following subsection, while Shapley values are retained as a benchmark in the simulation study.

\subsection{Properties of LOCO and CPI}\label{subsec:properties-loco-cpi}

    Throughout the rest of the paper, we focus on the $\ell_2$ loss $\ell(\hat{y},y) = (\hat{y} - y)^2$. We begin by establishing two fundamental properties of LOCO and CPI.

    First, both measures correctly assign zero importance to truly conditionally irrelevant features.
    In fact, they characterize null features through a conditional mean independence property, as formalized in \Cref{lem:null-feature}.
    Since LOCO/CPI under $\ell_2$ loss are risk functionals for mean regression, their null hypothesis corresponds to conditional mean independence rather than the stronger conditional independence condition $X_j\indep Y\mid X_{-j}$.
    \begin{lemma}[Null features and conditional mean independence]\label[lemma]{lem:null-feature}
    Under the $\ell_2$ loss $\ell(\hat y,y)=(\hat y-y)^2$, $\psi^{\cpi}_{X_j}=\psi^{\loco}_{X_j}=0$ if and only if $\mu(X)=\mu_{-j}(X_{-j})$ almost surely.
    \end{lemma}

    Second, LOCO and CPI are equivalent under squared-error loss and can be expressed in terms of expected conditional variance.

    \begin{lemma}[Equivalence under $\ell_2$ loss]\label[lemma]{lem:equiv}
        For $\ell_2$ loss $\ell(\hat{y},y) = (\hat{y} - y)^2$, the two importance measures coincide: $ \psi^{\loco}_{X_j}= \psi^{\cpi}_{X_j} = 
        \EE[\VV[\mu(X) \mid X_{-j}]]$.
    \end{lemma}

    The conditional variance $\EE[\VV[\mu(X) \mid X_{-j}]]$ is known as the upper Sobol index \citep{sobol1990sensitivity,owen2017shapley}, which measures how much uncertainty in the prediction $\mu(X)$ remains after observing all features except $X_j$.
    This connection to sensitivity analysis provides additional interpretability to these importance measures.
    While \citet[Theorem 2]{hooker2021unrestricted} established a similar equivalence result\footnote{We note that the definition of ``conditional variable importance'' in \citet{hooker2021unrestricted} appears to have a factor of $1/2$ missing compared to CPI formulation.} using functional analysis, our proof relies only on elementary probability theory.
    Specifically, we exploit the alternative variance representation $\VV[W] = \EE[(W- \EE[W])^2] = \EE[(W- W')^2]/2$, where $W'$ is an independent copy of $W$, for any random variable $W$.
    This representation reveals a fundamental equivalence under $\ell_2$ loss: LOCO's approach of refitting the response submodel $Y\mid X_{-j}$ yields the same importance measure as CPI's approach of resampling from the conditional covariate distribution $X_j\mid X_{-j}$, with CPI having an additional factor of $1/2$ to account for the doubled variability from comparing two random realizations. 
    Lastly, the two approaches may yield different importance measures under other loss functions.

    The notion of FI is used for different objectives in the literature. 
    When the goal is feature selection, LOCO and CPI are well-aligned with the objective of identifying features that provide incremental predictive value beyond the remaining covariates, because of the above appealing properties.
    In this paper, however, our focus is feature-importance attribution under dependence: how much predictive signal is generated by independent latent factors and expressed through an observed covariate, even when that signal is shared with correlated covariates.
    Under this attribution objective, conditional-incremental measures can exhibit correlation distortion.
    As noted by \citet{verdinelli2024feature}, ``\emph{its value depends not just on how strongly $\mu(X)$ depends on $X_j$, but also on the correlation between $X_j$ and the other features. In the extreme case of perfect dependence, it is 0. This could lead users to erroneously conclude that some features are irrelevant even when $\mu(X)$ strongly depends on $X_j$.}'' The following examples illustrate this mismatch.

    \begin{example}[Dependent features with a bijective mapping]\label{ex:perfect-dependent}
        Consider a simple linear model: 
        \begin{align} Y = \beta_1 X_1 + \beta_2 X_2 + \epsilon \label{eq:bivar-linear-model} 
        \end{align} 
        with $X_1 = X_2 = U $. Under the $\ell_2$ loss, 
        \[\psi^{\loco}_{X_1}=\psi^{\loco}_{X_2}=\psi^{\cpi}_{X_1}=\psi^{\cpi}_{X_2}=0.
        \]
        Thus, LOCO/CPI consistently declare each feature conditionally (incrementally) unimportant given the other, regardless of the value of $(\beta_1,\beta_2)$. This near-zero score reflects redundancy, that $X_1$ provides no additional predictive utility once $X_2$ is available, but it does not reflect the attribution of predictive signal of $X_1$ to the regression function $\mu(X)$.
    \end{example}

    \begin{example}[Dependent features with an injective mapping]\label{ex:perfect-dependent-inj}
        Now suppose $X_{1}=X_{2}^{2}=U^{2}$ with $U$ symmetric in the same model \eqref{eq:bivar-linear-model}.
        Then
        \[
          \psi^{\loco}_{X_1}=\psi^{\cpi}_{X_1}=0,\qquad
          \psi^{\loco}_{X_2}=\psi^{\cpi}_{X_2}=\beta_{2}^{2}\VV(U).
        \]
        Thus, LOCO/CPI down-weights $X_{1}$ entirely: when $\beta_1\neq 0$ but $\beta_2=0$, neither of the features is viewed as important by LOCO/CPI.
        Furthermore, although $X_2$ is assigned nonzero importance, the assigned value only captures the linear contribution $\beta_2^2\VV(U)$, which underestimates the total predictive signal $\VV(\EE[Y\mid U]) = \VV(\beta_1U^2+\beta_2U)$. 
        Thus, the full predictive variability is still not accurately reflected.
    \end{example}

    These examples do not show that LOCO and CPI are universally inadequate. 
    Rather, they show that under strong feature dependence, LOCO/CPI can be poorly aligned with the goal of attributing marginal relevance to a fixed regression function. Because these methods quantify conditional incremental value, they appropriately treat shared signal as redundancy for feature selection, but in attribution settings this can assign zero or attenuated importance to features that still carry substantial predictive signal through dependence. As shown by \cite{verdinelli2024decorrelated}, Shapley-based incremental measures can inherit the same issue in high dimensions. This motivates a complementary notion of FI that targets predictive relevance under dependence.

\section{Disentangled feature importance}\label{sec:DFI}

    \subsection{Disentangled representations}\label{subsec:def-phiZ}

    To target attribution under dependence, we propose to transform the features into a disentangled space before computing FI.
    We seek to construct a new representation \( Z = (Z_1, \dots, Z_d) \) of the original features \( X \), with regular conditional distribution \(K(\cdot\mid X)\) such that the coordinates of \( Z \) are independent.
    Specifically, let $(\Omega_X,\cF_X)$ and $(\Omega_Z,\cF_Z)$ be Borel spaces, and let $K:\Omega_X\times \cF_Z\to[0,1]$ be a Markov kernel.
    Let $\PP_{XY}\in\cP(\Omega_X\times\RR)$ govern $(X,Y)$ with marginal $\PP_X$, and define $Z\mid X=x\sim K(\cdot\mid x)$; equivalently,
    $\PP_Z(B)=\int K(B\mid x)\,\PP_X(\rd x)$ for all $B\in\cF_Z$.
    Without loss of generality, we assume $Z_j$ has zero mean and unit variance for all $j=1,\ldots,d$.
    We denote by $\PP$ the induced joint law of $(X,Z,Y)$; under $\PP$, $Z\sim \PP_Z=\bigotimes_{j=1}^d \PP_{Z_j}$ and thus has independent coordinates. We call such a $Z$ a disentangled representation of $X$, or simply latent features.
    In particular, we use standard Gaussian latent features $\PP_Z=\cN_d(0,I_d)$ for our numerical implementation.

    Define the regression functions in the two feature spaces:
    \begin{align*}
        \mu(x):=\EE[Y\mid X=x],
        \qquad
        \eta(z):=\EE[Y\mid Z=z] =
        \int \mu(x)\,\PP(X\in \rd x \mid Z=z),
    \end{align*}
    where the conditional law $\PP(X\in\cdot\mid Z=z)$ exists for $\PP_Z$-a.e.\ $z$.
    Equivalently, $\eta$ is characterized as the Radon--Nikodym derivative \(\int_B \eta(z)\,\PP_Z(\rd z) = \int \mu(x)\,K(B\mid x)\,\PP_X(\rd x)\), \(\forall\,B\in\cF_Z\).
    If, in addition, $K(\cdot\mid x)\ll \PP_Z$ for $\PP_X$-a.e.\ $x$, we may write
    \[
    K(\rd z\mid x)=k(z\mid x)\,\PP_Z(\rd z)
    \]
    for the Radon--Nikodym derivative $k(\cdot\mid x)=\rd K(\cdot\mid x)/\rd\PP_Z$. In this case, $\eta$ admits the pointwise representation
    \begin{align}
        \eta(z)=\int \mu(x)\,k(z\mid x)\,\PP_X(\rd x),\qquad \PP_Z\text{-a.e.\ }z.    \label{eq:eta}
    \end{align}
    Once the disentangled representation \( Z \) is obtained, we define the latent FI measure under $\ell_2$ loss for each feature \( j \in [d] \):
    \begin{align}
        \phi_{Z_j}(\PP) := \EE\left[ \VV(\eta(Z) \mid Z_{-j})\right], \label{eq:phi-Z}
    \end{align}
    where the expectation is taken under the joint law $\PP$ of $(X,Z,Y)$ and the chosen transition kernel $K$.
    Thus, although suppressed in the notation, the latent importance measure depends on the disentangled representation determined by $K$ and $\PP_Z$; when needed, we write this explicitly as $\phi_{Z_j}^{(K,\PP_Z)}(\PP)$.
    
    From \Cref{lem:equiv}, this measure coincides with both LOCO and CPI defined for the latent feature $Z$ instead of the raw feature $X$.
    Because the coordinates of \( Z \) are independent, the FI measure \(\phi_{Z_j}\) defined in \eqref{eq:phi-Z} admits a clear interpretation under independent latent coordinates, as formalized in the following proposition.

    \begin{proposition}[Free of correlation distortion]\label[proposition]{prop:free-cor}
        For $\ell_2$ loss, $\phi_{Z_j}= 0 $ if and only if $\eta(Z) = \eta_{-j}(Z_{-j})$ almost surely.
    \end{proposition}

    \subsection{Importance attribution for original features}
    After computing FI for the disentangled features \( Z \), our goal is to attribute importance back to the original features \( X \).
    To define a meaningful importance score \(\phi_{X_l}\) for each original feature \(X_l\), we seek to satisfy the following properties:
    \begin{enumerate}[(i)]
        \item Null calibration under conditional mean irrelevance and independence: 
        If the outcome \( Y \) does not depend on \( X_l \), and \( X_l \) is independent of the other covariates, that is, $\EE[Y \mid X] = \EE[Y \mid X_{-l}]$ and $X_l \indep X_{-l}$, then the importance score satisfies \(\phi_{X_l} = 0\).
        
        \item Sensitivity to predictive relevance under dependence:
        If \( X_l \) carries predictive signal, then \(\phi_{X_l}\) should be positive, even when \(X_l\) has no incremental value given \(X_{-l}\).
    \end{enumerate}
    In other words, these desiderata motivate an attribution score that vanishes for conditionally mean-irrelevant independent covariates and remains responsive to predictive signal carried through dependence.

    Towards this goal, we define the FI score \(\phi_{X_l}\) for each original feature \(X_l\) by transferring importance from the disentangled features \(Z\) back to \(X\) through the sensitivity of \(Z_j\) with respect to \(X_l\).    
    Formally, we define the disentangled feature importance (DFI) measure as
    \begin{align}
        \phi_{X_l}(\PP)
        ~:=~
        \sum_{j=1}^d
        \EE\!\left[
        \VV\!\left(\eta(Z)\mid Z_{-j}\right)\,
        \Big(\partial_{z_j} S_{l}(Z)\Big)^2
        \right],
        \label{eq:def-phi-X}
    \end{align}
    where \(S_{l}(z) = \EE[X_l \mid Z=z]\) denotes the barycentric projection and  \(\partial_{z_j} S_{l}(\cdot)\) denotes the partial derivative of \(S_{l}(z)\) with respect to \(z_j\).
    Thus, the attributed importance in the original feature space is also defined relative to the same disentangled representation; when needed, we write it as $\phi_{X_\ell}^{(K,\PP_Z)}(\PP)$.
    
    The inner term $\VV(\EE[Y \mid Z] \mid Z_{-j})$ measures the predictive variation remaining when the disentangled coordinate $Z_j$ is varied while $Z_{-j}$ is fixed.
    Multiplying by \((\partial_{z_j}S_\ell(Z))^2\) gauges how strongly fluctuations in \(Z_j\) are expressed through \(X_\ell\) on average; integrating over the data distribution averages these local sensitivities into a global importance score.
    Thus, \(\phi_{X_l}\) quantifies how much of the irreducible signal carried by all latent directions is channelled through $X_l$.

    DFI is designed to satisfy the two desiderata above. 
    First, it assigns zero importance to null features. 
    Specifically, for any \(\cS\subseteq[d]\), suppose that
    \(
    \EE[Y\mid X]=\EE[Y\mid X_{\cS}]
    \qquad\text{and}\qquad
    X_{\cS}\indep X_{\cS^c},
    \)
    so that the features in \(X_{\cS^c}\) are null. 
    Under the EOT coupling with a product reference measure, the coupling kernel respects this block structure. Consequently, \(X_{\cS^c}\) is null if and only if \(\phi_{Z_j}=0\) for all \(j\in\cS^c\), which in turn is equivalent to \(\phi_{X_j}=0\) for all \(j\in\cS^c\). 
    Second, DFI remains sensitive to relevant features even when their predictive information is shared through dependence. 
    We illustrate this next using the perfectly correlated examples in \Cref{ex:perfect-dependent,ex:perfect-dependent-inj}.

    \setcounter{excont}{4}
    \begin{excont}[{\hyperref[ex:perfect-dependent]{Continued}}]\label{ex:perfect-dependent-cont}
        For $Z_1=X_1=X_2=U$ and $Z_2\sim\cN(0,1)$ independent of $Z_1$, we have 
        \[\phi_{X_1} = \phi_{X_2} = \phi_{Z_1} = (\beta_1+\beta_2)^2\VV(U),\quad \phi_{Z_2} =0 .\]
        For bijective disentanglement, both $X_1$ and $X_2$ have importance measures equal to the total variation of the regression function $\VV(\EE[Y \mid X_1, X_2]) = \VV(\EE[Y \mid Z])$ using $Z$ alone.
        This reflects that both observed covariates carry the same predictive signal; it should not be read as saying that both are needed for prediction after one of them is already observed.
    \end{excont}

    \begin{excont}[{\hyperref[ex:perfect-dependent-inj]{Continued}}]\label{ex:perfect-dependent-inj-cont}
        To make the scale of the attributed scores explicit, suppose
        \(\mathbb E[U]=\mathbb E[U^3]=0\), \(\mathbb E[U^2]=1\), and \(m_4:=\mathbb E[U^4]>1\). 
        Consider the standardized features
        \(
        X_1=\frac{U^2-1}{\sqrt{m_4-1}}
        \), \(
        X_2=U,
        \)
        and the regression model
        \(
        Y=\beta_1X_1+\beta_2X_2+\epsilon,
        \qquad
        \mathbb E[\epsilon\mid U]=0.
        \)
        Let \(Z_1=U\) and let \(Z_2\sim N(0,1)\) be independent of \(Z_1\). Then
        \(Z=(Z_1,Z_2)\) has independent coordinates, and the barycentric decoder is
        \[
        S_1(z)=\frac{z_1^2-1}{\sqrt{m_4-1}},
        \qquad
        S_2(z)=z_1.
        \]
        Hence
        \[
        \partial_{z_1}S_1(z)=\frac{2z_1}{\sqrt{m_4-1}},
        \qquad
        \partial_{z_1}S_2(z)=1,
        \qquad
        \partial_{z_2}S_1(z)=\partial_{z_2}S_2(z)=0.
        \]
        The latent regression function is
        \(
        \eta(z)
        =
        \beta_1\frac{z_1^2-1}{\sqrt{m_4-1}}
        +
        \beta_2 z_1 .
        \)
        Therefore,
        \[
        \phi_{Z_1}
        =
        \mathbb V\{\eta(Z)\}
        =
        \beta_1^2+\beta_2^2,
        \qquad
        \phi_{Z_2}=0.
        \]
        The attributed DFI scores are
        \[
        \phi_{X_1}
        =
        \frac{4}{m_4-1}\,(\beta_1^2+\beta_2^2),
        \qquad
        \phi_{X_2}
        =
        \beta_1^2+\beta_2^2.
        \]
        Thus, although the predictive signal is entirely represented by the single latent direction $Z_1$, both observed coordinates receive nonzero attribution because both vary along this latent direction. 
        This contrasts with LOCO and CPI for $X_1$, which are zero because $X_1$ is deterministic given $X_2$. 
        The DFI score should therefore be interpreted as attribution of latent predictive variation through the observed measurement channels, not as evidence that $X_1$ provides additional predictive utility once $X_2$ is observed.
    \end{excont}

    In the above examples, DFI targets attribution of predictive relevance under dependence, rather than additional predictive utility relative to a baseline set. 
    Thus, a large DFI score should not be interpreted as evidence that a feature is indispensable, nor as evidence that the feature should be retained in a minimal selected subset. 
    When several observed covariates measure the same predictive signal, DFI may assign positive, or even large, importance to more than one of them. 
    This behavior is intended for model interpretation: it indicates that the predictive signal is expressed through these observed covariates, not that each covariate provides unique predictive information beyond the others.

    \subsection{Disentangled transformation}\label{subsec:disentangled-transformation}
    The key component of DFI is the transition kernel $k$.
    One natural way to specify $k$ is via optimal transport (OT) between $\PP_X$ and an independent reference law $\PP_Z:=\otimes_{j=1}^d \PP_{Z_j}$ \citep{kantorovich1942translocation}.
    In the classical (unregularized) Monge regime \citep{brenier1991polar}, the transport is induced by a deterministic map $T$, yielding the deterministic kernel $k(\rd z\mid x)=\delta_{T(x)}(\rd z)$, which is analyzed in \Cref{app:sec:ot-maps} for absolute continuous covariate distribution.
    Entropic optimal transport (EOT) is a compelling alternative to unregularized OT in settings where the Monge map may fail to exist or be non-unique (e.g., mixed or discrete covariates), while also enabling
    efficient computation \citep{cuturi2013sinkhorn,eckstein2022quantitative}.
    
    We adopt the KL-regularized formulation:
    \begin{equation}
        \gamma
        \in
        \argmin_{\pi\in\Gamma(\PP_X,\PP_Z)}
        \int c(x,z)\,\rd\pi(x,z)
        +\varepsilon\,\KL \!\left(\pi\,\middle\|\,\PP_X\otimes\PP_Z\right),
        \label{eq:EOT}
    \end{equation}
    where $\Gamma(\PP_X,\PP_Z)$ denotes the set of couplings with marginals $\PP_X$ and $\PP_Z$, and $\varepsilon>0$ is fixed throughout.
    The KL term makes the objective strictly convex in $\pi$, yielding a unique EOT coupling $\gamma$.
    Moreover, $\KL(\gamma\|\PP_X\otimes\PP_Z)<\infty$ implies $\gamma\ll \PP_X\otimes\PP_Z$, so that $\gamma$ admits a density with respect to the product measure $\PP_X\otimes\PP_Z$:
    \[
        \gamma(\rd x,\rd z)= r(x,z)\,\PP_X(\rd x)\PP_Z(\rd z),\qquad r(x,z):=\frac{\rd\gamma}{\rd(\PP_X\otimes\PP_Z)}.
    \]
    Define the forward and backward kernels as regular conditional distributions
    \[
    K_f(B\mid x):=\PP(Z\in B\mid X=x),\qquad
    K_b(A\mid z):=\PP(X\in A\mid Z=z),
    \]
    for $B\in\cF_Z$ and $A\in\cF_X$ (defined $\PP_X$-a.e.\ in $x$ and $\PP_Z$-a.e.\ in $z$).
    Then, for any $A\in\cF_X$ and $B\in\cF_Z$,
    \(    \gamma(A\times B)
    =\PP(X\in A, Z\in B)
    =\int_A K_f(B\mid x)\,\PP_X(\rd x)
    =\int_B K_b(A\mid z)\,\PP_Z(\rd z).
    \)
    Equivalently, the disintegrations satisfy
    \(
    \gamma(\rd x,\rd z)=\PP_X(\rd x)\,K_f(\rd z\mid x)=\PP_Z(\rd z)\,K_b(\rd x\mid z).
    \)

    Since $\gamma\ll \PP_X\otimes\PP_Z$, both kernels are dominated by their base measures, and in fact admit the explicit forms
    \[
    K_f(\rd z\mid x)=k_f(z\mid x)\,\PP_Z(\rd z),\qquad
    K_b(\rd x\mid z)=k_b(x\mid z)\,\PP_X(\rd x),
    \]
    with $k_f(z\mid x)=r(x,z)$ $\PP_Z$-a.e.\ and $k_b(x\mid z)=r(x,z)$ $\PP_X$-a.e.
    This representation enables randomized disentanglement: drawing $Z\sim K_f(\cdot\mid X)$ guarantees that the marginal of $Z$ is exactly $\PP_Z$ with independent coordinates.
    For decoding, we use the barycentric projection
    \[
    S(z) ~:=~ \EE_\gamma[X\mid Z=z]
    ~=~
    \int x\,K_b(\rd x\mid z)
    ~=~
    \int x\,k_b(x\mid z)\,\PP_X(\rd x),
    \]
    which provides a deterministic proxy for sampling $X$ from $Z$ under coupling $\gamma$.
    The next section studies estimation and inference of the target estimands defined in \eqref{eq:phi-Z} and \eqref{eq:def-phi-X} with the EOT specification $(c,\PP_Z,\varepsilon)$ treated as fixed throughout.

    \begin{remark}[Estimand specification and identifiability]\label{rem:estimand}
        DFI should be interpreted as a family of population estimands, not as a single map-free quantity. 
        Under the EOT construction, the target is indexed by the transport cost $c$, the independent reference law $\PP_Z$, and the entropic regularization level $\varepsilon$:
        \(\phi^{(c,\PP_Z,\varepsilon)}_{X_\ell}(\PP)\) and \(\phi^{(c,\PP_Z,\varepsilon)}_{Z_j}(\PP)\).
        For any fixed $(c,\PP_Z,\varepsilon)$ with $\varepsilon>0$, the EOT problem has a unique coupling, and therefore a unique transition kernel. 
        Consequently, $\phi^{(c,\PP_Z,\varepsilon)}_{X_\ell}(\PP)$ is a well-defined functional of the joint law $\PP$. 
        Changing $c$, $\PP_Z$, or $\varepsilon$ changes the target estimand, just as changing the loss function, baseline distribution, or perturbation scheme changes the target of other FI methods.

        For a default EOT specification, we use the quadratic cost $c(x,z)=\|x-z\|_2^2$ and an independent reference law $\PP_Z$. 
        In our numerical implementation, this reference law is taken to be $\PP_Z=\cN_d(0,I_d)$ after standardizing the observed features. 
        The entropic level $\varepsilon$ is treated as part of the estimand specification: it is reported in all experiments, and sensitivity to this choice is assessed over a pre-specified grid of values.
    \end{remark}

\section{Statistical estimation and inference}\label{sec:stat-est-inf}
    
    \subsection{Estimation of disentangled feature importance}\label{subsec:ot-latent}

    From this point on, $k(x\mid z)$ denotes the backward kernel unless otherwise noted.
    Suppose we observe i.i.d.\ data $\{(X_i,Y_i)\}_{i=1}^n$. Let $\hat\mu$ be an estimator of the regression function $\mu(x):=\EE[Y\mid X=x]$, and let $\hat k(x\mid z)$ be an estimator of the EOT kernel $k(x\mid z)$.
    When $\hat k$ is obtained from an empirical Sinkhorn coupling between the empirical training law $\PP_{\mathrm{tr}}$ and the reference law $\PP_Z$, its normalization is naturally with respect to $\PP_{\mathrm{tr}}$:
    \[
    \int \hat k(x\mid z)\,\PP_{\mathrm{tr}}(\rd x)=1, \qquad \PP_Z\text{-a.e.\ } z .
    \]
    Using the same estimated coupling, generate latent features by drawing
    $Z_i\sim \hat K(\cdot\mid X_i)$, and let $\PP_n$ denote the empirical measure of
    $O_i=(X_i,Z_i,Y_i)$. We estimate $\eta$ defined in \eqref{eq:eta} by
    \[
        \hat\eta(z):=\PP_{\mathrm{tr}}\{\hat\mu(X)\,\hat k(X\mid z)\}.
    \]    
    We then define the CPI-type estimator of $\phi_{Z_j}$ in \eqref{eq:phi-Z} as
    \begin{align}
    \hat{\phi}_{Z_j}(\PP)
    := \frac{1}{2}\,\PP_n\!\left[(Y-\hat{\eta}(Z^{(j)}))^2-(Y-\hat{\eta}(Z))^2\right],
    \qquad j=1,\ldots,d, \label{def:hat-phi-Z}
    \end{align}
    where $Z^{(j)}=(Z_{-j},Z_j')$ with $Z_j'\sim \PP_{Z_j}$ independent of $(X,Z,Y)$.

    By \Cref{lem:equiv}, both LOCO and CPI target the same population functional, so the choice between them is primarily computational.
    We therefore work with the CPI-style estimator in the latent space.
    First, it avoids repeated refitting: perturbing the $j$th latent coordinate does not require re-estimating a reduced regression $\eta_{-j}(Z_{-j})=\EE[Y\mid Z_{-j}]$ for each $j$.
    Instead, we estimate the full regression $\eta$ once, via the plug-in construction based on $\hat\mu$ and the estimated EOT kernel.
    Second, the disentangling construction renders the coordinates of $Z$ independent, so sampling from $\PP(Z_j\mid Z_{-j})$ reduces to drawing an independent copy $Z_j'\sim \PP_{Z_j}$ for estimator \eqref{def:hat-phi-Z}.
    This independence eliminates conditional-density estimation or Markov sampling steps that are otherwise needed in CPI under dependent covariates.

    To present our main results, we introduce assumptions on the data-generating process and the nuisance function estimators.

    \begin{assumption}[Data model and estimator]\label[assumption]{asm:data}
    There exist a constant $C>0$ such that:
    \begin{enumerate}[(i)]

        \item\label{asm:data:cov} Covariate $X\in\RR^d$ has zero mean $\EE[X]=0$ and bounded second moment $\|X\|_{L_2}\le C$;

        \item\label{asm:data:response} Response $Y\in\RR$ has zero mean $\EE[Y]=0$ and bounded moment $\|Y\|_{L_\infty}\le~C$;
        
        \item\label{asm:data:k} Transform density $k$ is bounded in $L_2(\PP_X\otimes \PP_Z)$, $\|k\|_{L_2(\PP_X\otimes \PP_Z)}\leq C$;

        \item\label{asm:data:estimator} Nuisance function estimators $\hat\mu$ and $\hat k$ are estimated on auxiliary samples independent of the $n$ observations used in $\PP_n$ such that \( \|\hat\mu\|_{L_{\infty}(\PP_X)}\leq C\) and \(\|\hat\mu \hat k\|_{L_{\infty}(\PP_X\otimes\PP_Z)}\leq C\) with probability tending to one.
    \end{enumerate}
    \end{assumption}

    The moment conditions in \Cref{asm:data} are primarily technical and serve to guarantee that the transformed regression function \(\eta(z)=\int \mu(x)k(x\mid z)\rd \PP_X(x)\) and the associated score terms are well-defined in \(L_2\), and that the remainder terms admit simple norm bounds.
    Specifically, \Cref{asm:data}~\ref{asm:data:cov} and \ref{asm:data:k} ensure basic integrability of products involving \(X\) and \(k\) under \(\PP_X\otimes\PP_Z\), while \Cref{asm:data}~\ref{asm:data:response} and the \(L_\infty\) envelope conditions in \ref{asm:data:estimator} are used only to streamline the control of second-order remainder terms; these \(L_\infty\) conditions can be replaced by suitable finite-moment assumptions (e.g., uniform \(L_4/L_8\) bounds ensuring the relevant products remain in \(L_2\)).
    Finally, \Cref{asm:data}~\ref{asm:data:estimator} can be enforced in practice by cross-fitting \citep{du2025causal,kennedy2022semiparametric} and clipping \(\hat\mu\) and \(\hat\mu\hat k\) to a fixed envelope.

Under \Cref{asm:data}, we derive a simultaneous first-order approximation for the latent DFI estimators \(\hat\phi_{Z_j}\), and separate the impact of nuisance estimation into errors in estimating the regression function \(\mu\) and the EOT kernel \(k\).

\begin{theorem}[Simultaneous linear expansion of latent DFI estimators]\label{thm:error-decom-eot}
    Consider the statistical model on $O=(Z,X,Y)$ restricted by \Cref{asm:data}(i)-(ii).
    Denote $\Delta_j(z):=\eta(z)-\eta_{-j}(z)$ and the induced law of $O$ under the EOT coupling by $\PP$.
    Under \Cref{asm:data}, for $d$ fixed, the latent DFI estimators \eqref{def:hat-phi-Z} satisfy the simultaneous linear expansion
    \[
    \hat\phi_{Z_j}(\PP)-\phi_{Z_j}(\PP)
    =
    (\PP_n-\PP)\{\varphi_{Z_j}(O;\PP)\}
    +\Op\!\left(n^{-1/2}\cE_Z+\cE_Z^2\right),
    \]
    uniformly over $j=1,\ldots,d$, where the efficient influence function (EIF) is
    \[
    \varphi_{Z_j}(O;\PP) = 2 (Y - \eta(Z))\Delta_j(Z) + \Delta_j(Z)^2 - \phi_{Z_j}(\PP),
    \]
    and the remainder depends on
    \(
    \cE_Z := \|\hat\mu-\mu\|_{L_2(\PP_X)} + \| \hat k-k\|_{L_2(\PP_X\otimes\PP_Z)}.
    \)
\end{theorem}

From \Cref{thm:error-decom-eot}, the remainder term is $\Op(n^{-1/2}\cE_Z+\cE_Z^2)$, and hence is $\op(n^{-1/2})$ whenever $\cE_Z=\op(n^{-1/4})$.
For many nonparametric estimators, this holds under standard smoothness assumptions \citep{kennedy2022semiparametric}.
For instance, if $\mu$ belongs to an $\alpha$-H\"older class in $\RR^d$, minimax-optimal methods (e.g., local polynomials or splines) yield
$\|\hat\mu-\mu\|_{L_2(\PP_X)}=\Op(n^{-\alpha/(2\alpha+d)})$, which is faster than $n^{-1/4}$ when $2\alpha>d$.
For the coupling estimation, the regularized OT problem is strongly convex for any fixed entropic regularization parameter and its empirical Sinkhorn solution enjoys parametric $n^{-1/2}$ statistical rates for the regularized OT value and (under mild boundedness conditions) for the associated EOT potentials; see, e.g., \citet{genevay2019sample,mena2019statistical}.
Recent work on empirical EOT couplings establishes parametric density-estimation guarantees in this same metric: for the plug-in coupling computed from empirical marginals (via Sinkhorn), one has
\(\|\hat k-k \|_{L_2(\PP_X\otimes \PP_Z)} \lesssim n^{-1/2}\) with high probability under bounded-cost conditions (e.g., compact support); see, for instance, \citet[Theorem~6]{rigollet2025sample}.
In particular, this does not require $\PP_X$ to admit a Lebesgue density bounded away from zero.
Consequently, in such regimes $\|\hat k-k\|_{L_2(\PP_X\otimes \PP_Z)}=\Op(n^{-1/2})$ (up to logarithmic factors), the empirical-process term in the decomposition from \Cref{thm:error-decom-eot} dominates, yielding the following asymptotic normality result used for inference.

\begin{corollary}[Statistical inference]\label{cor:inference}
    Under the same setup as in \Cref{thm:error-decom-eot} with $\|\hat\mu-\mu\|_{L_2(\PP_X)} = \op(n^{-1/4})$ and $\| \hat k-k\|_{L_2(\PP_X\otimes\PP_Z)} = \op(n^{-1/4})$. 
    For $j=1,\ldots,d$, if $\phi_{Z_j}>0$, then 
    \begin{align}
        \sqrt{n}(\hat\phi_{Z_j}(\PP)-\phi_{Z_j}(\PP)) \dto \cN(0, \VV[\varphi_{Z_j}(O;\PP)]) . \label{eq:asym-normality-phi-Z-eot}
    \end{align}    
\end{corollary}

The influence function can be estimated via
\begin{align}
    \hat{\varphi}_{Z_j}(O;\hat{\PP})
= 2 (Y - \hat\eta(Z))\Big(\hat\eta(Z) - \EE_{Z^{(j)}_j}[\hat\eta(Z^{(j)})]\Big)
+ \Big(\hat\eta(Z) - \EE_{Z^{(j)}_j}[\hat\eta(Z^{(j)})]\Big)^2
- \hat\phi_{Z_j}.\label{eq:eif-phi-Z-eot-est}
\end{align}
Consequently, the asymptotic variance $\VV[\varphi_{Z_j}(O;\PP)]/n$ can be estimated by
$\hat{\sigma}^2 =\VV_n[\hat{\varphi}_{Z_j}(O;\hat{\PP})] / n$.

\begin{remark}[Confidence intervals near the null]\label[remark]{rmk:null}
When $\VV[\varphi_{Z_j}(O;\PP)]>0$, \Cref{eq:asym-normality-phi-Z-eot} yields the Wald interval $\cC_{n,\alpha}= (\hat{\phi}_{Z_j}-z_{\alpha/2}\hat{\sigma}, \ \hat{\phi}_{Z_j}+ z_{\alpha/2}\hat{\sigma})$ with asymptotic coverage $1-\alpha$.
When $\phi_{Z_j}=0$, the influence function vanishes and asymptotic normality does not hold anymore.
        In more general cases when $\phi_{Z_j}$ is a function of $n$ and converges to zero, finding a uniformly covered confidence interval is unsolved, due to the behavior of quadratic functionals.
        Though one way to mitigate such an issue is to expand the confidence interval by $\cO(n^{-1/2})$ to maintain validity at the expense of efficiency at the null; see \citet[Section 6]{verdinelli2024feature}.
\end{remark}

\subsection{Estimation of attributed feature importance}\label{subsec:ot-attributed}
    We consider a class of attributed importance scores that generalizes \eqref{eq:def-phi-X}.
    By the tower property, \eqref{eq:def-phi-X} can be written in a ``weighted'' form in which the sensitivity factor is averaged over the perturbed coordinate.
    Specifically, let \(w_{jl}(Z_{-j})\) be an arbitrary weight that is measurable with respect to \(Z_{-j}\).
    For a given collection \(\{w_{jl}\}_{j,l\in[d]}\), define
    \begin{align}
        \phi_{X_l}(\PP)
        =\sum_{j=1}^d \EE\!\left[v_j(Z_{-j})\,w_{jl}(Z_{-j})\right],
        \label{eq:def-phi-X-weighted}
    \end{align}
    where \(v_j(z_{-j}) := \EE\!\left[\Delta_j(Z)^2\mid Z_{-j}=z_{-j}\right]\) and \(\Delta_j(z):=\eta(z)-\eta_{-j}(z)\).
    The DFI functional in \eqref{eq:def-phi-X} is recovered by taking
    \(w_{jl}(z_{-j})=\EE\!\left[s_{jl}(Z)^2\mid Z_{-j}=z_{-j}\right]\), where \(s_{jl}(z):=\partial_{z_j}S_l(z)\).

    A key technical difference from the latent analysis in \Cref{subsec:ot-latent} is that the weights \(w_{jl}\) may themselves be functionals of the data law \(\PP\) (as in the DFI choice above, where \(w_{jl}\) depends on the coupling through \(S\)).
    As a result, the influence function of \(\phi_{X_l}(\PP)\) includes an additional contribution capturing the first-order impact of estimating \(w_{jl}\).
    We encode this requirement via a high-level EIF condition for the auxiliary functionals \(\theta_{jl}(\PP):=\EE[w_{jl}(Z_{-j})]\).

    \begin{assumption}[EIF for attributed weights]\label[assumption]{asm:eif-theta}
        For all $j,l\in[d]$, the functional \(\theta_{jl}(\PP)= \EE\!\big[w_{jl}(Z_{-j})\big]\) with weight \( w_{jl} \in L_4(\PP_{Z_{-j}})\) is pathwise differentiable at $\PP$ relative to the nonparametric model on $O$ and admits an EIF $\varphi_{\theta_{jl}}(O;\PP)$;
        and the estimator $\hat\theta_{jl}$ of $\theta_{jl}(\PP)$ constructed from a nuisance estimate $\hat w_{jl}$ satisfies the linear expansion uniformly in $(j,l)$:
            \[
            \hat\theta_{jl}-\theta_{jl}(\PP)=(\PP_n-\PP)\varphi_{\theta_{jl}}(O;\PP)+ R_{\theta_{jl}},
            \]
        such that the following conditions hold:
        \begin{enumerate}[(i)]
            \item\label{asm:eif-theta-i} \(
            \sup_{j,l\in[d]}\;\|\varphi_{\theta_{jl}}(O;\PP)\|_{L_4(\PP)}<\infty
            \).
            \item\label{asm:eif-theta-ii} $R_{\theta_{jl}} = \Op\left(\max_{j,l\in[d]}\|\hat w_{jl}-w_{jl}\|_{L_2(\PP_{Z_{-j}})}^2\right)$.
        \end{enumerate}
    \end{assumption}
    
    \Cref{asm:eif-theta} formalizes the regularity required to treat the weight functionals \(\theta_{jl}(\PP)=\EE[w_{jl}(Z_{-j})]\) as pathwise differentiable and to control the remainder term with estimation error of \(w_{jl}\).
    With this, \Cref{prop:eif-eot-weighted} derives an influence-function representation for \(\phi_{X_l}(\PP)\).

    \begin{proposition}[Efficient influence function for attributed DFI]\label[proposition]{prop:eif-eot-weighted}
        Under Assumptions \ref{asm:data} and \ref{asm:eif-theta}~\ref{asm:eif-theta-i}, the EIF of $\phi_{X_l}(\PP)$ in \eqref{eq:def-phi-X-weighted} is pathwise differentiable and the EIF is given by
        \begin{align}
            \varphi_{X_l}(O;\PP)
            &:=
            \sum_{j=1}^d
            \Big[
            w_{jl}(Z_{-j}) ( 2(Y-\eta(Z))\Delta_j(Z)+\Delta_j(Z)^2 ) \notag\\
            &\qquad
            +
            v_j(Z_{-j}) ( \varphi_{\theta_{jl}}(O;\PP)- w_{jl}(Z_{-j}) + \theta_{jl}(\PP) )
            \Big]
            -\phi_{X_l}(\PP).    \label{eq:eif-phi-X-eot}
        \end{align}
    \end{proposition}

    The influence function in \Cref{prop:eif-eot-weighted} decomposes into (i) a weighted component that treats \(w_{jl}\) as known and propagates uncertainty only through the LOCO contrast \(v_j\), and (ii) an adjustment that accounts for the fact that \(w_{jl}\) is estimated from \(\PP\).

    \begin{remark}[EIF for attributed weights in \Cref{asm:eif-theta}]\label[remark]{rmk:asm:eif}
    On finite spaces (e.g., discrete OT with finitely supported marginals), the entropic optimizer and Sinkhorn map are smooth, and CLTs for entropy-regularized OT costs/divergences yield pathwise differentiability for many smooth OT-derived functionals; see \citet{bigot2019central}.
    For continuous marginals, Hadamard differentiability and limit theory for EOT potentials/maps support functional delta-method CLTs and efficiency statements, while stability results justify the required continuity/differentiability under perturbations of the marginals and cost; see \citet{goldfeld2024limit,ghosal2022stability,eckstein2022quantitative}.

    In several canonical regimes, \Cref{asm:eif-theta} can be substantially weakened in practice.
    When the coupling (or induced map) is available in closed form (e.g., Gaussian/Bures--Wasserstein transport and Gaussian Schr\"odinger bridges), the weights \(w_{jl}\) and hence \(\theta_{jl}(\PP)\) are explicit; see \citet{bunne2023schrodinger}.
    Conversely, when \(w_{jl}\) is computed via Sinkhorn iterations, one may approximate \(\phi_{\theta_{jl}}\) by differentiating the Sinkhorn fixed point or via approximations (e.g., infinitesimal jackknife / leave-one-out), avoiding symbolic EIF derivations in implementations; see \citet{cuturi2013sinkhorn,genevay2018learning}.
    \end{remark}

    To specialize \Cref{prop:eif-eot-weighted} to DFI score \eqref{eq:def-phi-X} under EOT, we require that the barycentric decoder \(S_l(z)=\EE[X_l\mid Z=z]\) and its coordinatewise sensitivities \(s_{jl}(z)=\partial_{z_j}S_l(z)\) are well-defined and sufficiently integrable under the EOT coupling.
    \Cref{asm:bary-smooth} imposes smoothness and differentiability in the latent coordinate $z$, which is needed to ensure that the sensitivity weight \(w_{jl}(z_{-j})=\EE[s_{jl}(Z)^2\mid Z_{-j}=z_{-j}]\) is a well-defined target.
    
    \begin{assumption}[Differentiability and moments for barycentric sensitivities]\label[assumption]{asm:bary-smooth}
        There exists a constant $C>0$, such that for all $l,j\in[d]$:
        \begin{enumerate}[(i)]
        \item\label{asm:bary-smooth:i}
        For $\PP_X$-a.e.\ $x$, the map $z\mapsto k(x\mid z)$ is weakly differentiable in the coordinate $z_j$,
        with derivative $\partial_{z_j}k(x\mid z)$ satisfying
        \(
            \EE[X_l^2\big(\partial_{z_j}k(X\mid Z)\big)^2] < \infty
        \).
        
        \item\label{asm:bary-smooth:ii}
        The barycentric projection $S_l(z)$ is weakly differentiable in $z_j$ for $\PP_Z$-a.e.\ $z$,
        and its derivative admits the integral representation
        \(
        s_{jl}(z)
        =
        \int x_l\,\partial_{z_j}k(x\mid z)\,\PP_X(\rd x)\),
        $\PP_Z\text{-a.e. } z$.

        \item\label{asm:bary-smooth:iii}
        The coordinatewise sensitivities $s_{jl}$ have bounded moment
        \(
        \|s_{jl}(Z)\|_{L_{\infty}(\PP_Z)}< C
        \).
    
        \item\label{asm:bary-smooth:estimator} Nuisance functions \(\partial_z\hat k\) is estimated on an auxiliary sample independent of the $n$ observations used in $\PP_n$ and \(\|\partial_z\hat k\|_{L_\infty(\PP_X\otimes\PP_Z)}\leq C\) with probability tending to one.
        \end{enumerate}
    \end{assumption}

    \Cref{asm:bary-smooth} is a regularity condition for sensitivity-based attribution under EOT.
    Conditions~\ref{asm:bary-smooth:i}--\ref{asm:bary-smooth:ii} ensure that the barycentric decoder \(S_l\) is differentiable in \(z_j\) and that \(s_{jl}\) is square-integrable, which makes the weight \(w_{jl}(z_{-j})=\EE[s_{jl}(Z)^2\mid Z_{-j}=z_{-j}]\) well-defined and estimable from \(\partial_{z_j}\hat k\).
    In entropic OT with regularization \(\varepsilon>0\), the optimal coupling admits the Sinkhorn form
    \(
    \gamma(\rd x,\rd z)=u(x)v(z)\exp\{-c(x,z)/\varepsilon\}\PP_X(\rd x)\PP_Z(\rd z),
    \) and \(k(x\mid z)=\frac{\gamma(\rd x,\rd z)}{\PP_Z(\rd z)}\).
    While EOT yields a strictly positive conditional density $k(x\mid z)$, it does not automatically provide differentiability of $z\mapsto k(x\mid z)$ (and hence of $z\mapsto S_l(z)$); so \(k(x\mid z)\) inherits differentiability from \(z\mapsto c(x,z)\) under standard dominated-convergence conditions; see, e.g., stability/differentiability results for EOT potentials and barycentric maps in \Cref{rmk:asm:eif}.
    Condition~\ref{asm:bary-smooth:iii} supplies moment control needed only for bounding second-order remainder terms in the EIF expansion.
    Condition~\ref{asm:bary-smooth:estimator} is a restriction on the nuisance function similar to \Cref{asm:data}~\ref{asm:data:estimator}.
    
    Under \Cref{asm:bary-smooth}, we estimate the barycentric decoder and its coordinatewise sensitivities via the estimated backward kernel $\hat k$ and its partial derivative.
    Define, for $l\in[d]$,
    \[
        \hat S_l(z):=\int x_l\,\hat k(x\mid z)\,\PP_X(\rd x),
        \qquad
        \hat s_{jl}(z):=\int x_l\,\partial_{z_j}\hat k(x\mid z)\,\PP_X(\rd x).
    \]
    We then form the sensitivity weight
    \[
        \hat w_{jl}(z_{-j})
        :=\EE\!\left[\hat s_{jl}(Z)^2 \mid Z_{-j}=z_{-j}\right] =\int \hat s_{jl}(z_{-j},z_j)^2\,\PP_{Z_j}(\rd z_j),
    \]
    which simplifies under independent coordinates to the one-dimensional average.
    Otherwise, $\hat w_{jl}$ can be learned as a conditional mean of $\hat s_{jl}(Z)^2$ given $Z_{-j}$ (with cross-fitting).
    Finally, we estimate the attributed DFI by
    \begin{align}
        \hat\phi_{X_l}
        :=
        \frac12\sum_{j=1}^d
        \PP_n\!\left[
        \hat w_{jl}(Z_{-j})
        \Big\{(Y-\hat\eta(Z^{(j)}))^2-(Y-\hat\eta(Z))^2\Big\}
        \right]. \label{def:hat-phi-X-eot}
    \end{align}
    \Cref{thm:error-decom-X-eot} establishes the simultaneous linear expansion for estimator $\hat\phi_{X_l}$, which yields Wald-type intervals from the empirical variance of the estimated EIF.

    \begin{theorem}[Simultaneous linear expansion of attributed DFI under EOT]\label{thm:error-decom-X-eot}
        Consider the statistical model on $O=(Z,X,Y)$ where $(X,Z)$ arise from the EOT coupling $\gamma$ and the moment conditions in Assumptions \ref{asm:data}~\ref{asm:data:cov}-\ref{asm:data:response} and \ref{asm:bary-smooth} hold.
        Under Assumptions \ref{asm:data}, \ref{asm:eif-theta}, and \ref{asm:bary-smooth}, for $d$ fixed, the estimator $\hat\phi_{X_l}$ in \eqref{def:hat-phi-X-eot} admits the simultaneous linear expansion
        \[
        \hat\phi_{X_l}-\phi_{X_l}(\PP)
        =
        (\PP_n-\PP)\,\varphi_{X_l}(O;\PP)
        +
        \Op\!\Big(n^{-1/2}\cE_X+\cE_X^2\Big),
        \]
        uniformly over $l=1,\ldots,d$, with EIF $\varphi_{X_l}(O;\PP)$ defined in \eqref{eq:eif-phi-X-eot} and the remainder
        \[
        \cE_X
        :=
        \|\hat\mu-\mu\|_{L_2(\PP_X)}
        +
        \|\hat k-k\|_{L_2(\PP_X\otimes\PP_Z)}
        +
        \|\partial_z\hat k-\partial_z k\|_{L_2(\PP_X\otimes\PP_Z)}.
        \]
        In particular, if $\cE_X=\op(n^{-1/4})$, then for each fixed $l$, \(\sqrt{n}\,(\hat\phi_{X_l}-\phi_{X_l}(\PP)) \dto \cN(0, \VV(\varphi_{X_l}(O;\PP)))\).
    \end{theorem}

    For any fixed specification \((c,\PP_Z,\epsilon)\) with \(\epsilon>0\), the EOT construction defines a unique coupling and hence a well-defined population DFI estimand, including for discrete or mixed covariate laws. 
    This should be distinguished from the stronger regularity conditions used for inference. 
    \Cref{thm:error-decom-eot} requires \(L_2\)-consistent estimation of the backward kernel for latent FI, while \Cref{thm:error-decom-X-eot} additionally requires differentiability of the barycentric decoder and consistent estimation of its coordinatewise sensitivities for attributed FI. Thus, EOT removes map non-uniqueness at the estimand level, but root-\(n\) inference for attributed DFI is established under smoothness and nuisance-rate conditions.
    
    \subsection{Properties}\label{subsec:R2-extension}

    In \Cref{subsec:ot-latent,subsec:ot-attributed}, the estimands $\phi_{Z_j}$ and $\phi_{X_\ell}$ are shown to admit EIF-based linear expansions and standard Wald-type inference away from the null.
    By construction, $\PP_Z$ is a product measure; consequently, any $\eta\in L_2(\PP_Z)$ admits the unique functional ANOVA (Hoeffding--Sobol) decomposition \citep{sobol1990sensitivity,owen2017shapley}:
    \[
    \eta(Z)=\sum_{\cS\subseteq[d]}\eta_{\cS}(Z_{\cS}),
    \]
    where $\EE[\eta_{\cS}(Z_{\cS})]=0$ for all $\cS\neq\emptyset$, and the summands are orthogonal in $L_2(\PP_Z)$: $\EE[\eta_{\cS}(Z_{\cS})\eta_{\cS'}(Z_{\cS'})]=0$ for $\cS\neq\cS'$.
    Thus $\VV(\eta(Z))=\sum_{\emptyset\neq\cS\subseteq[d]}\VV\!\big(\eta_{\cS}(Z_{\cS})\big)$, with each $\eta_{\cS}$ capturing the contribution of the interaction among coordinates in $\cS$.

    This decomposition yields an immediate interpretation of the latent importance target $\phi_{Z_j}(\PP)=\EE[\VV(\eta(Z)\mid Z_{-j})]$.
    Since conditioning on $Z_{-j}$ removes precisely the ANOVA components that do not involve $j$, $\phi_{Z_j}(\PP)$ aggregates the variance contributions of all interaction terms that include $j$:
    \[
    \phi_{Z_j}(\PP)=\sum_{\cS:\,j\in\cS}\VV\!\big(\eta_{\cS}(Z_{\cS})\big),
    \]
    which coincides with the unnormalized total Sobol effect for coordinate $Z_j$ \citep{sobol1990sensitivity}.
    This connection yields the following decomposition lemma.

    \begin{lemma}[Decomposition of latent feature importance]\label[lemma]{lem:latent-importance-decomp}
    The following properties hold:
    \begin{enumerate}[(i)]
        \item (Weighted variance decomposition) The sum of the latent feature importances equals the sum of variances of the ANOVA terms, weighted by their interaction order: $ \sum_{j=1}^d \phi_{Z_j} (\PP) = \sum_{ \cS \subseteq [d]} |\cS| \cdot \VV[\eta_{\cS}(Z_{\cS})]$.
        
        \item (Total variance under additivity) The sum of the latent importances equals the total predictive variability if and only if the regression function $\eta(Z)$ is purely additive, i.e., $ \sum_{j=1}^d \phi_{Z_j}(\PP) = \VV[\EE[Y \mid Z]] $ if and only if $\eta(Z) = c + \sum_{j=1}^d g_j(Z_j), $ for some constant $c$ and deterministic functions $g_1,\ldots,g_d$.
    \end{enumerate}
    \end{lemma}

    \Cref{lem:latent-importance-decomp} highlights a fundamental distinction between DFI and methods that simply decompose variance. 
    While the total predictive variance sums the contributions from all interaction terms equally, the sum of DFI scores is a weighted sum where higher-order interactions contribute proportionally to their size. This implies that the total importance captured by DFI exceeds the predictive variance in the presence of interactions. The equality only holds for purely additive models, where all predictive signal comes from main effects.

    Next, we relate the attributed importance in the original feature space to the latent importance in the disentangled space.
    Because $Z$ is generated from $X$ through the EOT coupling, the regression signal in $Z$ cannot exceed that in $X$ by the law of total variance:
    \[
    \VV(\EE[Y\mid Z])\le \VV(\EE[Y\mid X]) ,
    \]
    with equality when $Z$ is almost surely an invertible function of $X$.
    In particular, if $\gamma$ is concentrated on the graph of an invertible differentiable map $T:X\mapsto Z$, then $L(\cdot\mid z)=\delta_{T^{-1}(z)}$, $\eta(z)=\mu(T^{-1}(z))$, and $S(z)=T^{-1}(z)$.
    Hence \eqref{eq:phi-Z} is unchanged, while \eqref{eq:def-phi-X} reduces to
    \[
    \phi_{X_\ell}(\PP)
    =
    \sum_{j=1}^d
    \EE\!\left[
    \VV\!\big(\EE[Y\mid Z]\mid Z_{-j}\big)
    \left(\frac{\partial X_\ell}{\partial Z_j}\right)^2
    \right].
    \]
    When $T$ is linear, this relationship further simplifies, yielding an explicit link between importances in the $X$- and $Z$-spaces.

    \begin{proposition}[Gaussian transport decomposition]\label[proposition]{prop:decomp-phi}
        Assume \(X\sim \cN(0,\Sigma)\) with \(\Sigma \succeq 0\).
        Suppose the unregularized OT coupling \(\gamma\) between \(\PP_X\) and \(\PP_Z=\cN_d(0,I)\) is induced by
        \[
            Z=\Sigma^{-\frac{1}{2}}X,
        \]
        then the attributed feature importance \(\phi_{X_l}(\PP)\) in \eqref{eq:def-phi-X} satisfies that for all \(l\in[d]\),
        \begin{align}
            \phi_{X_l}(\PP) = \sum_{j=1}^d (\Sigma_{jl}^{\frac{1}{2}})^2 \phi_{Z_j} (\PP)  
        \quad\text{and}\quad
            \sum_{j=1}^d
            \Sigma_{jj}\phi_{Z_j}(\PP)
            \;=\;
            \sum_{l=1}^d\phi_{X_l}(\PP).\label{eq:phi-X-R2}
        \end{align}
    \end{proposition}

    A key consequence of \Cref{prop:decomp-phi} is that for standardized features, where $\VV[X_l]=\Sigma_{ll}=1$ for all $l$, the total importance is conserved across spaces: \(\sum_{l=1}^d\phi_{X_l}(\PP) = \sum_{j=1}^d\phi_{Z_j}(\PP)\).
    This identity, combined with \Cref{lem:latent-importance-decomp}, provides a crucial link to classical regression analysis. For a multiple linear regression model, the function $\eta(Z)$ is purely additive, meaning the total importance equals the total predictive variability. 
    Our DFI measure for each feature, $\phi_{X_l}$, precisely recovers the corresponding term in the $R^2$ decomposition discussed in \Cref{ex:R2}. 
    This result establishes our framework as a principled nonparametric generalization of the classical $R^2$ decomposition \citep{genizi1993decomposition}, extending its logic to models with arbitrary nonlinearities and feature dependencies.    
    
    \setcounter{excont}{0}
    \begin{excont}[{\hyperref[ex:R2]{Continued}}]
        Consider the multiple linear model \eqref{model:linear}, for the disentangled feature $Z = \Sigma^{-\frac{1}{2}}X \sim \cN_d(0,I)$, one can show that $\phi_{Z_j} = (e_j^{\top}\Sigma^{\frac{1}{2}}\beta)^2$.
        Then, we attribute importance to original features \(\phi_{X_l} = \sum_{j=1}^d(\Sigma^{\frac{1}{2}})_{jl}^2 \phi_{Z_j}\) based on \eqref{eq:phi-X-R2}, which coincides with the decomposition of the coefficient of determination $R^2$ for linear regression with correlated regressors \citep{genizi1993decomposition}.
        
        More generally, for any positive definite covariance matrix $\Sigma$, one has $\phi_{Z_j} = \Sigma_{jj}^{-1/2}(e_j^{\top}\Sigma^{\frac{1}{2}}\beta)^2$, while the attribution relationship \eqref{eq:phi-X-R2} remains the same.
    \end{excont}

\section{Simulation}\label{sec:simu}
    \subsection{Comparative analysis of importance attribution}
    In the first simulation study, we present the following four simulated examples. In each case, we plot the estimated importance of all features (averaged over 100 simulations) with different values of correlation $\rho$ between the features. In each case, the sample size is $n = 2000$.

    \begin{enumerate}[(M1)]

    \item\label{M1} Linear Gaussian model:
    \(
    Y = 5X_1+\epsilon,
    \) where $\epsilon\sim\cN(0,1)$ and \(X\sim\cN_{10}(0,\Sigma)\) with \(\diag(\Sigma)=1, \Sigma_{12}=\Sigma_{21}=\rho\) and zero otherwise.
    
    \item\label{M2} Nonlinear model:
    \(
    Y = 5\cos(X_1)+5\cos(X_{2})+\epsilon,
    \) with the same covariate and noise as \hyperref[M1]{(M1)}.

    \item\label{M3} Piecewise interaction with indicators:    
    \(
    Y = 1.5X_1X_{2}\,\ind\{X_{3}>0\} + X_{4}X_{5}\,\ind\{X_{3}<0\} + \epsilon ,
    \)
    where $\epsilon\sim\cN(0,0.4)$ and \(
    X\sim\cN_{5}(0,\Sigma),\;
    \Sigma_{jj}=1,\;
    \Sigma_{12}=\Sigma_{21}=\Sigma_{45}=\Sigma_{54}=\rho ,
    \) and $0$ otherwise.
    
    \item\label{M4} Low-density model:
    \(
    Y = 5X_1+\epsilon,
    \) where
    \(
    X_{-2}\sim\cN_9(0,1),
    X_{2}=3X_1^{2}+\delta
    \) and \((\delta,\epsilon)\sim\cN_2(0,I)\).

    \end{enumerate}

    \begin{figure}[!t]
        \centering
        \includegraphics[width=0.494\linewidth]{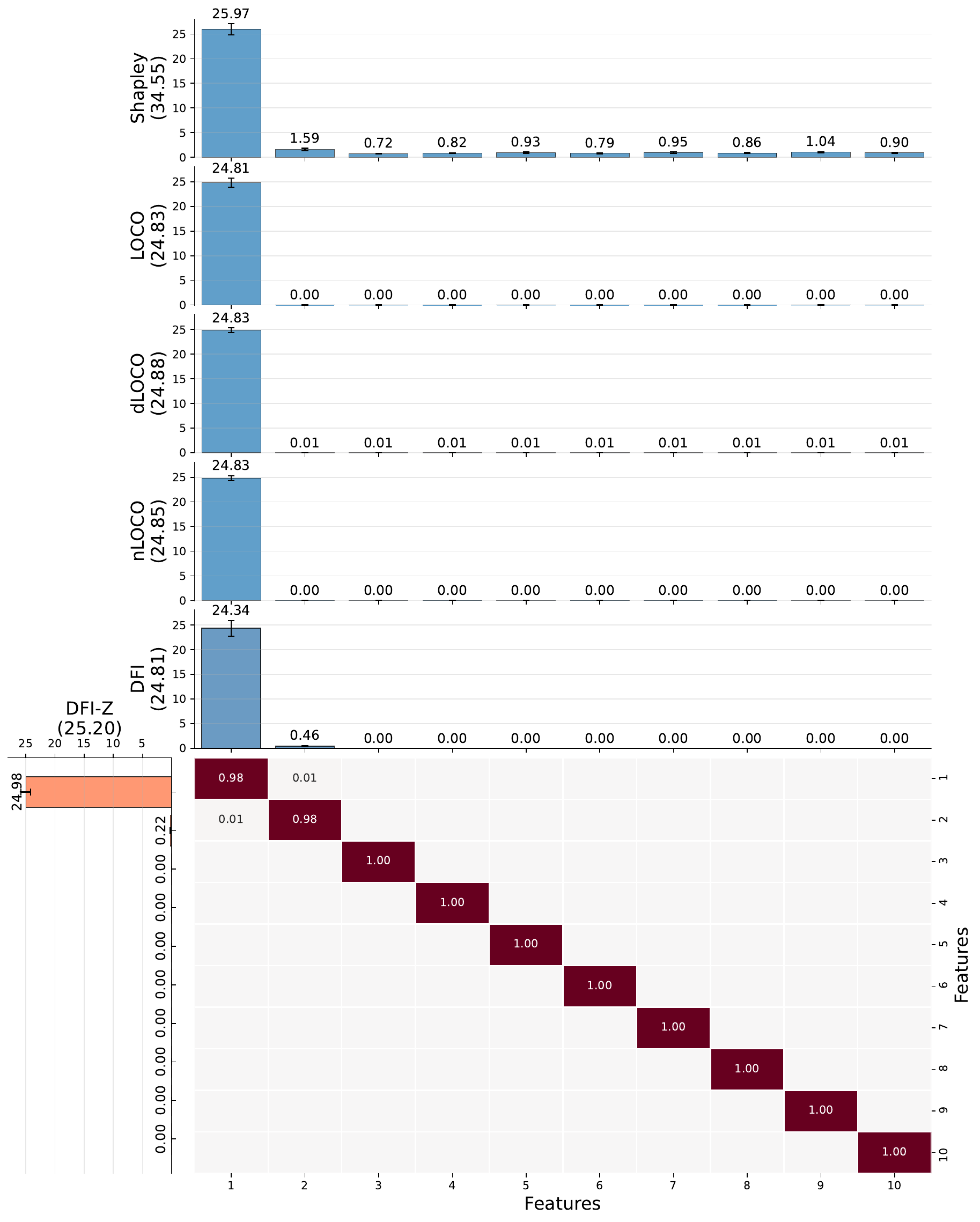}
        \includegraphics[width=0.494\linewidth]{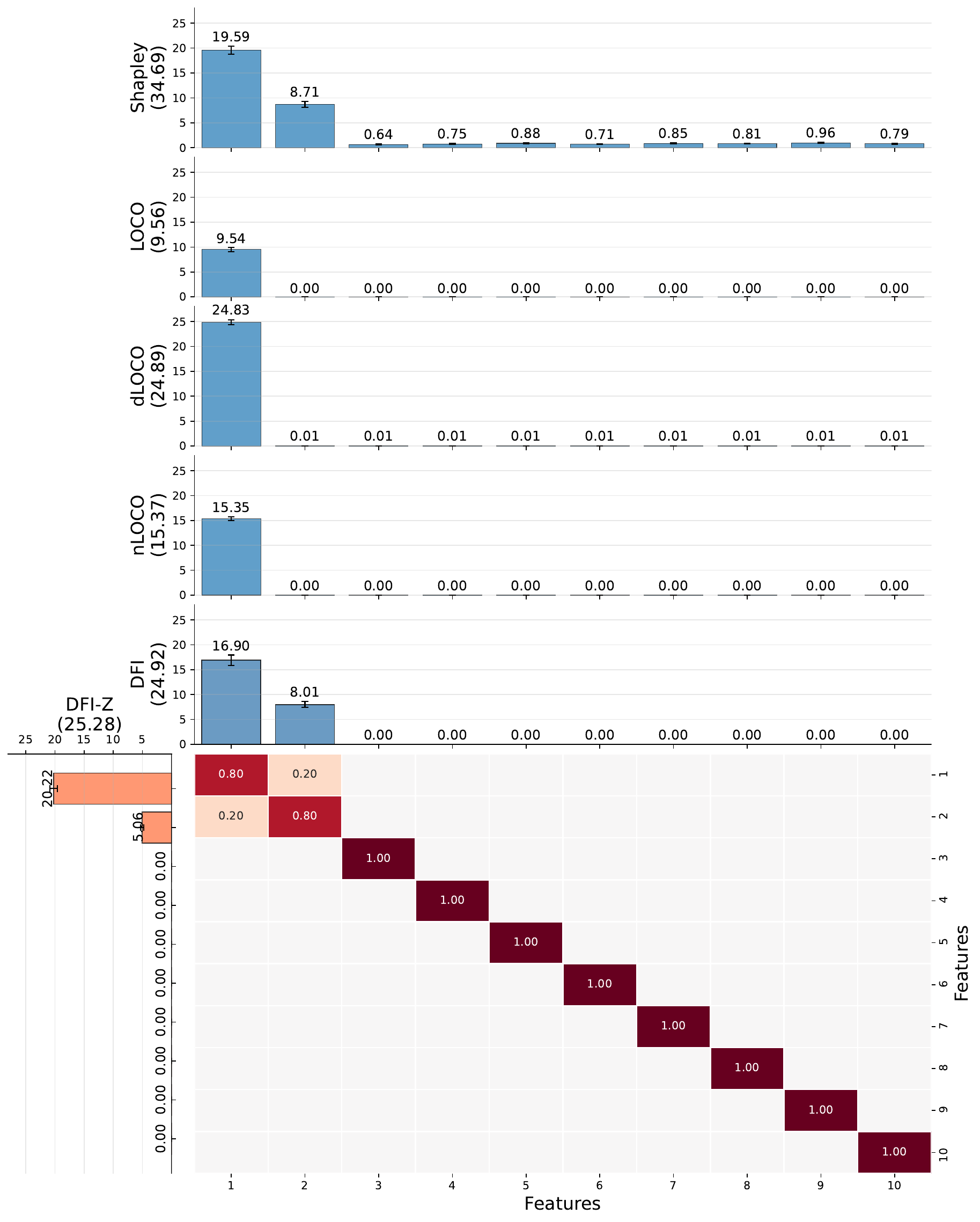}
        \caption{Simulation results under \hyperref[M1]{(M1)}.
        (a) weak correlation $(\rho = 0.2)$; 
        (b) strong correlation $(\rho = 0.8)$. 
        For every method, the number shown in parentheses is the total importance, which sums up to the signal variance $\VV\!\bigl[\EE[Y \mid X]\bigr]=25$ when the covariates are independent $(\rho = 0)$.
        Bars give the mean over $100$ random seeds, and the error bars indicate the corresponding standard deviations. 
        The heatmap on the right visualizes the weights $(\Sigma^{\frac{1}{2}})_{jl}^{\,2}$ that transfer importance from the latent coordinates to the observed features.}
        \label{fig:simu-1}
    \end{figure}

   \begin{figure}[!t]
        \centering
        \includegraphics[width=0.494\linewidth]{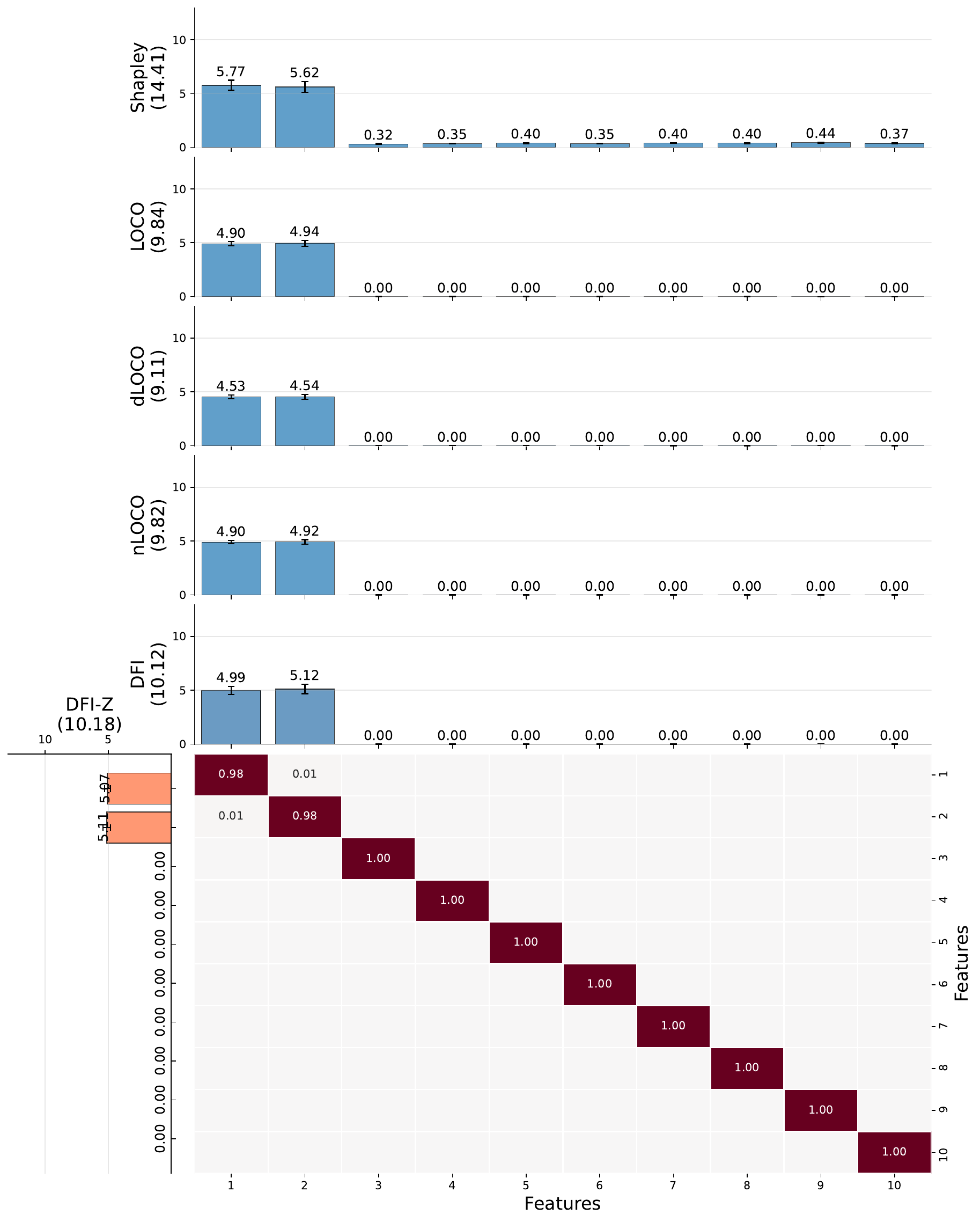}
        \includegraphics[width=0.494\linewidth]{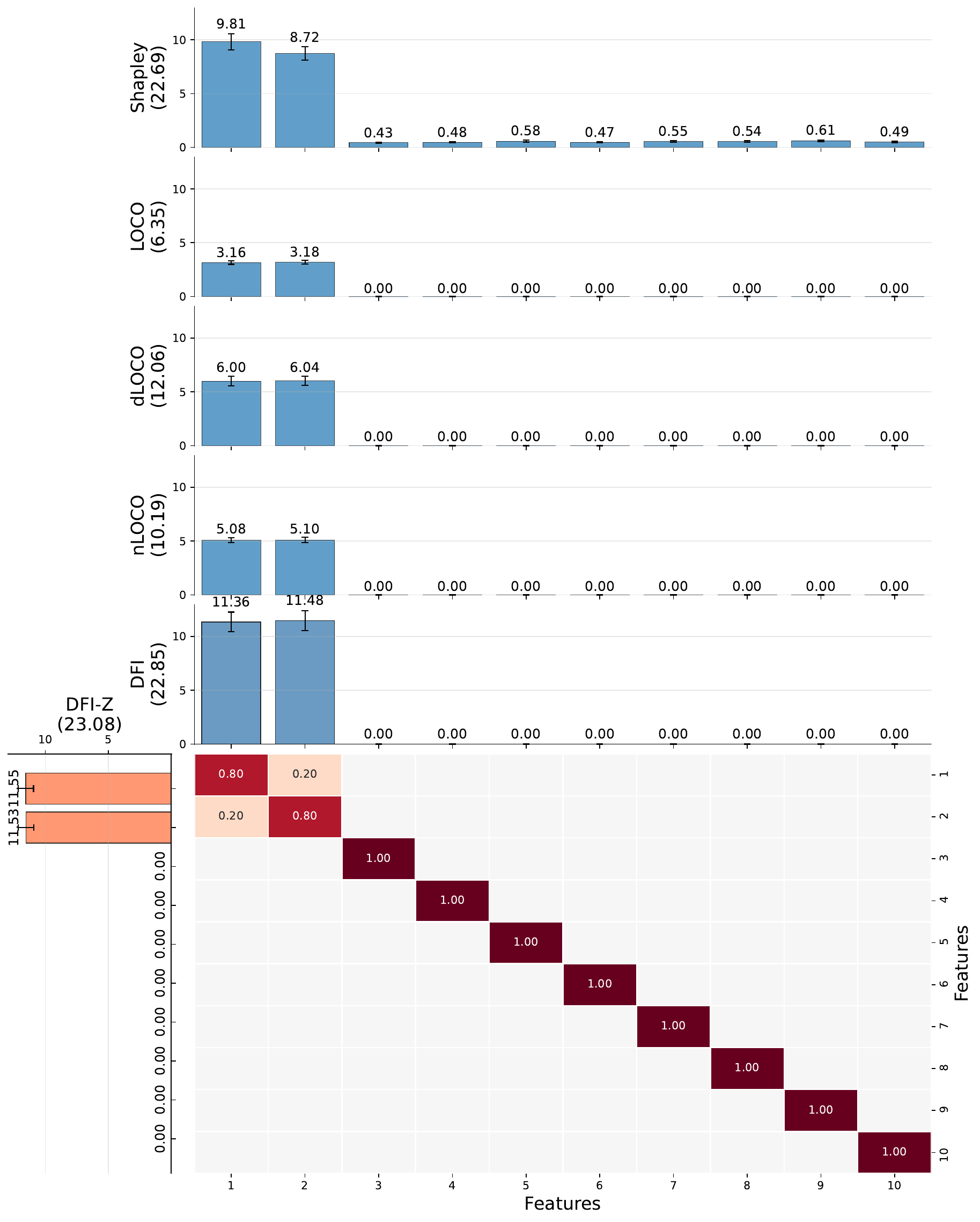}
        \caption{Simulation results for model \hyperref[M2]{(M2)}.  
        (a) weak correlation ($\rho = 0.2$);  
        (b) strong correlation ($\rho = 0.8$).  
        The interpretation of the bars, error bars, and heat-map is identical to \Cref{fig:simu-1}.
        The number in parentheses is the total estimated importance, which should be close to the signal variance  
        \(
        \VV\!\bigl[\EE[Y\mid X]\bigr]
          = 25 + 25\mathrm{e}^{-2} - 100\mathrm{e}^{-1} + 50\mathrm{e}^{-1}\cosh(\rho)\) (which is $\approx 10$ when $\rho = 0$, and $\approx 20$ when $\rho = 1$).
        }
        \label{fig:simu-2}
    \end{figure}

    \begin{figure}[!t]
        \centering
        \includegraphics[width=0.494\linewidth]{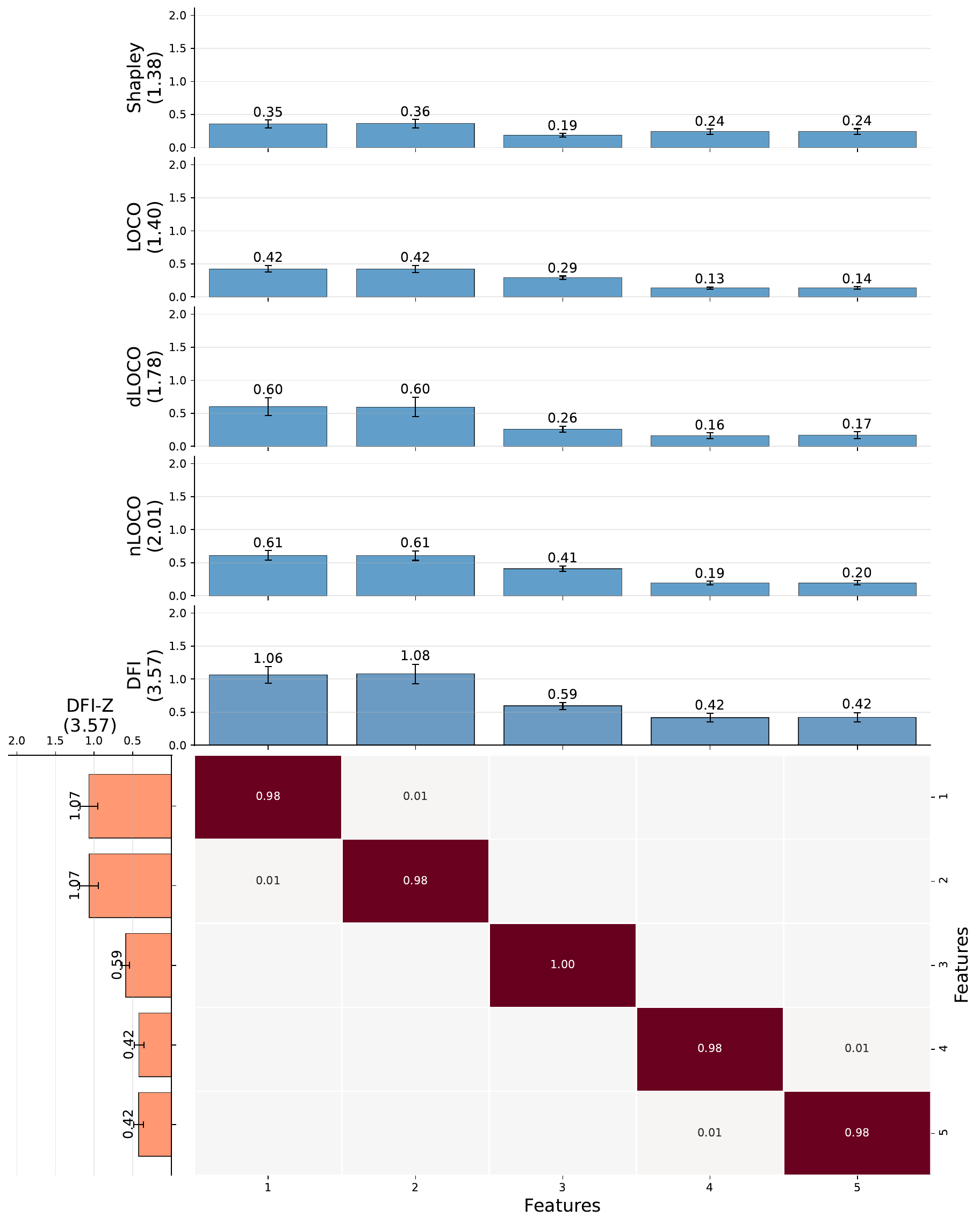}
        \includegraphics[width=0.494\linewidth]{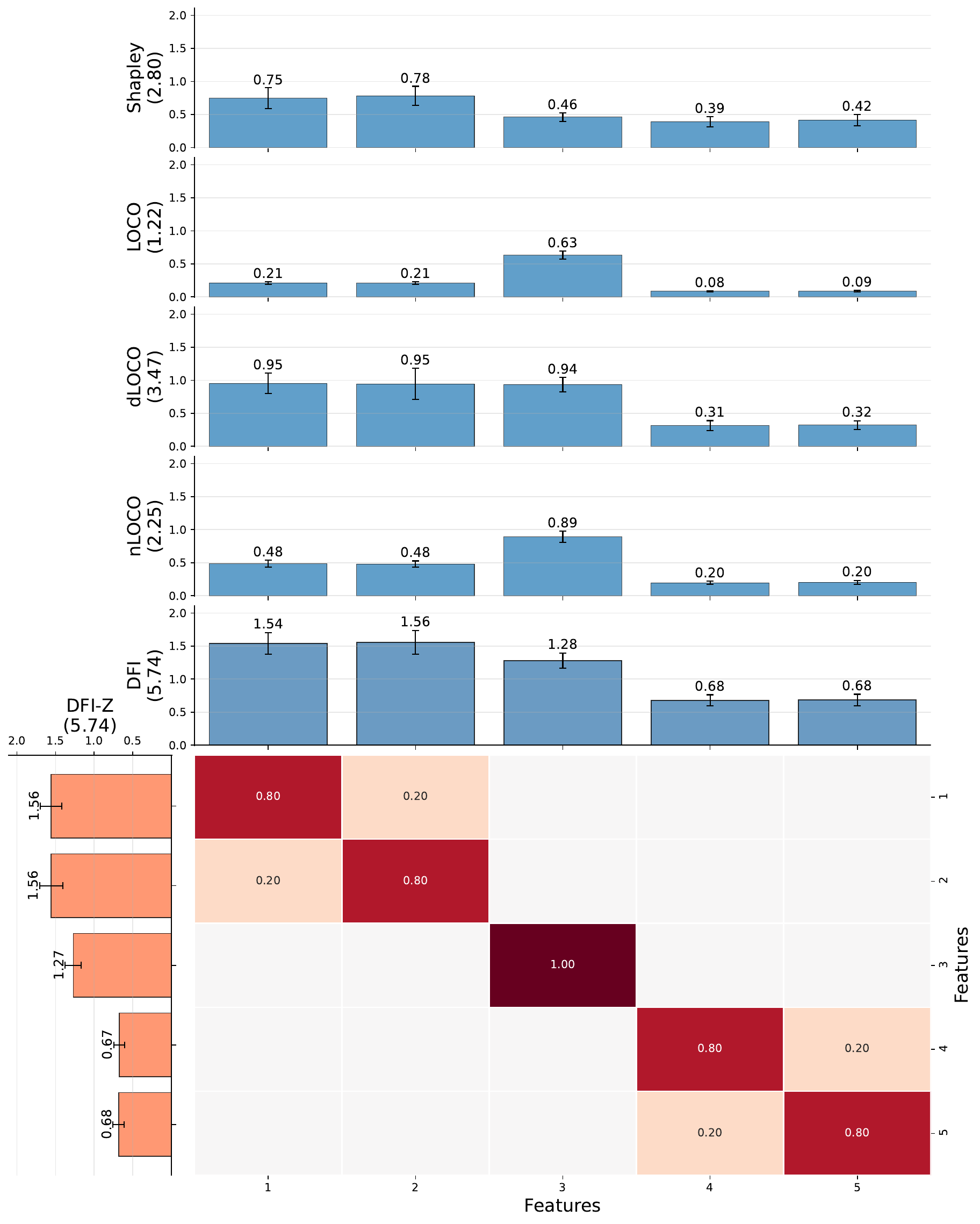}
        \caption{Simulation results for model \hyperref[M3]{(M3)}.  
        (a) weak correlation ($\rho = 0.2$);  
        (b) strong correlation ($\rho = 0.8$).  
        The interpretation of the bars, error bars, and heat-map is identical to \Cref{fig:simu-1}.
        }
        \label{fig:simu-3}
    \end{figure}

    In our analysis, we denote the FI measures derived from our framework as DFI ($\phi_{X_l}$) and DFI-Z ($\phi_{Z_j}$) by EOT with regularization parameter $\varepsilon=10^{-3}$ and 50 samples of $Z^{(j)}$ to improve the estimation, as described in \Cref{app:sec:simu}.
    We compare their performance with several established methods: LOCO, nLOCO (normalized LOCO; \citealp{verdinelli2024feature}), dLOCO (decorrelated LOCO; \citealp{verdinelli2024decorrelated}), and the Shapley value \citep{shapley1953value}. 
    For dLOCO and Shapley values, we use sampling approximation as described in \citet{verdinelli2024feature}, with 1000 and 100 samples, respectively, to reduce computational burden.
    All methods employ two-fold cross-fitting, where one fold estimates the nuisance functions (regression models and transport maps) and the other fold computes the importance measures.
    For nuisance function estimation, we use a random forest with 500 trees and a minimum of five samples per leaf.
    The comparative results for the four simulation models are presented in \Crefrange{fig:simu-1}{fig:simu-3} and \ref{fig:simu-4}.

    In the linear model \hyperref[M1]{(M1)}, as shown in \Cref{fig:simu-1}, all methods perform well under weak correlation ($\rho = 0.2$), correctly identifying $X_1$ as the important feature. 
    However, under strong correlation ($\rho = 0.8$), the importance value of LOCO and nLOCO diminishes for $X_1$ in the presence of a correlated feature $X_2$, illustrating the correlation distortion issue noted by \citet{verdinelli2024decorrelated}. 
    In contrast, dLOCO, Shapley, and DFI correctly assign a larger importance value ($\geq 16$) to $X_1$, while only the latter two measures acknowledge $X_2$ to be informative, though less important than $X_1$. 
    In addition, DFI is among the methods whose total importance remains close to the total predictive variability $\VV(\EE[Y\mid X])$, while also preserving the intended attribution interpretation.
    Similarly, for the nonlinear model \hyperref[M2]{(M2)}, presented in \Cref{fig:simu-2}, all methods correctly identify the two equally important features, $X_1$ and $X_2$. However, the LOCO variants and the Shapley value substantially underestimate their importance. 
    In this additive latent setting, DFI is the only method among those compared whose total score remains close to the corresponding predictive variability while distributing importance symmetrically between \(X_1\) and \(X_2\).
    
    For the piecewise interaction model \hyperref[M3]{(M3)}, which involves complex dependencies and interaction terms \citep{benard2022mean}, the results are presented in \Cref{fig:simu-3}. 
    The findings highlight that most LOCO-based methods, with the exception of dLOCO, underestimate the importance of the correlated features $(X_1,X_2)$. In contrast, DFI correctly identifies the relative importance of the features. 
    A theoretical calculation detailed in \Cref{app:subsec:M3} establishes the population DFI ranking as $\phi_{X_1}=\phi_{X_2}>\phi_{X_3}>\phi_{X_4}=\phi_{X_5}$. 
    Our empirical results successfully recover this correct ordering for different correlation levels. Nevertheless, the overall magnitude of the DFI estimates is slightly lower than the theoretical values, an outcome attributable to the challenge of accurately estimating the interaction-heavy regression function using random forests at the given sample size. 
    In an oracle setting where the regression function is known, the DFI estimates are consistent with the theoretical values, as demonstrated in \Cref{fig:simu-3-supp}.

    Finally, we consider model \hyperref[M4]{(M4)}, where $X_1$ and $X_2$ exhibit a non-linear relationship but are linearly uncorrelated. 
    This represents a misspecified setting for our method, as DFI seeks the best linear projection to explain the data generating process. 
    As shown in \Cref{fig:simu-4}, the LOCO variants assign almost all importance to $X_1$, while both Shapley and DFI recognize $X_2$'s relevance.
    Although the true model depends solely on $X_1$, knowing $X_2$ reduces uncertainty about the magnitude of $X_1$, which may be valuable for prediction.
    Under EOT, DFI's partial attribution to $X_2$ represents the best linear approximation to the true disentangled representations under model misspecification, yielding that $\sum_j \phi_{Z_j} = \VV[\EE[Y\mid Z]]=25$ but $\sum_l \phi_{X_l} = \VV[\EE[Y\mid X]] = 25$.

    Across these scenarios, DFI is better aligned with the post-hoc attribution objective studied in this paper, particularly when predictive signal is shared across correlated covariates.

    \subsection{Inferential validity and computational performance}

    Beyond providing accurate point estimates, a practical FI measure must allow for valid statistical inference and be computationally tractable. 
    In our second simulation study, we evaluate these quantitative aspects of DFI and its competitors. 
    We assess inferential validity by examining the empirical coverage rates of the confidence intervals constructed for the importance measures. 
    Our inferential comparison excludes the Shapley value and dLOCO due to known challenges in tractable uncertainty quantification.
    Estimating the influence function of Shapley values is computationally intensive \citep{williamson2020efficient}, while the standard error of dLOCO is analytically intractable because of its formulation as a U-statistic \citep{verdinelli2024decorrelated}. 
    As shown in \Cref{tab:simu}, DFI, along with the other methods, achieves average coverage at or above the nominal 90\% level, confirming the validity of our asymptotic theory.

    However, a key practical advantage of DFI lies in its computational efficiency. As illustrated in \Cref{fig:time}, DFI is significantly faster than competing methods that also account for feature correlations, such as dLOCO and sampling-based approximations of the Shapley value. This efficiency stems from DFI's design, which avoids both the need to refit multiple submodels (a drawback of LOCO-based approaches) and the direct estimation of complex conditional distributions (a challenge for standard CPI). This combination of statistical robustness, inferential validity, and computational efficiency makes DFI a practical tool for researchers analyzing complex, high-dimensional datasets.
    
    \begin{table}[!t]\tiny
        \centering
        \begin{tabular}{lcccccccccccc}
        \toprule
          $\rho$ & \multicolumn{4}{c}{0.0} & \multicolumn{4}{c}{0.4} & \multicolumn{4}{c}{0.8} \\
        \cmidrule(lr){2-5} \cmidrule(lr){6-9} \cmidrule(lr){10-13}
         $n$ & 100 & 250 & 500 & 750 & 100 & 250 & 500 & 750 & 100 & 250 & 500 & 750 \\
        \midrule
        DFI ($\varepsilon=0$) & 0.942 & 0.968 & 0.987 & 0.990 & 0.942 & 0.981 & 0.987 & 1.000 & 0.962 & 0.974 & 0.994 & 1.000 \\
        DFI ($\varepsilon=0.001$) & 0.907 & 0.946 & 0.978 & 0.978 & 0.917 & 0.952 & 0.958 & 0.987 & 0.929 & 0.939 & 0.978 & 0.990 \\
        DFI ($\varepsilon=0.01$) & 0.907 & 0.939 & 0.974 & 0.971 & 0.920 & 0.952 & 0.958 & 0.984 & 0.946 & 0.942 & 0.987 & 0.984 \\
        DFI ($\varepsilon=0.1$) & 0.904 & 0.939 & 0.971 & 0.968 & 0.913 & 0.952 & 0.965 & 0.987 & 0.946 & 0.942 & 0.987 & 0.987 \\
        DFI ($\varepsilon=1$) & 0.910 & 0.939 & 0.974 & 0.971 & 0.913 & 0.955 & 0.965 & 0.984 & 0.946 & 0.946 & 0.987 & 0.987 \\
        \cmidrule(lr){2-13}
        LOCO & 0.987 & 0.955 & 0.942 & 0.958 & 0.987 & 0.968 & 0.952 & 0.942 & 0.939 & 0.971 & 0.968 & 0.955 \\
        \cmidrule(lr){2-13}
        \cmidrule(lr){2-13}
        nLOCO & 0.984 & 0.955 & 0.942 & 0.955 & 0.990 & 0.968 & 0.952 & 0.939 & 0.942 & 0.974 & 0.968 & 0.952 \\
        \bottomrule
        \end{tabular}
        \caption{
            Coverage of different importance measures for the null features $(X_{3},\ldots,X_{10})$ under model \hyperref[M1]{(M1)} with a nominal level of $\alpha=0.1$.}
        \label{tab:simu}
    \end{table}

    \begin{figure}[!ht]
        \centering
        \includegraphics[width=\linewidth]{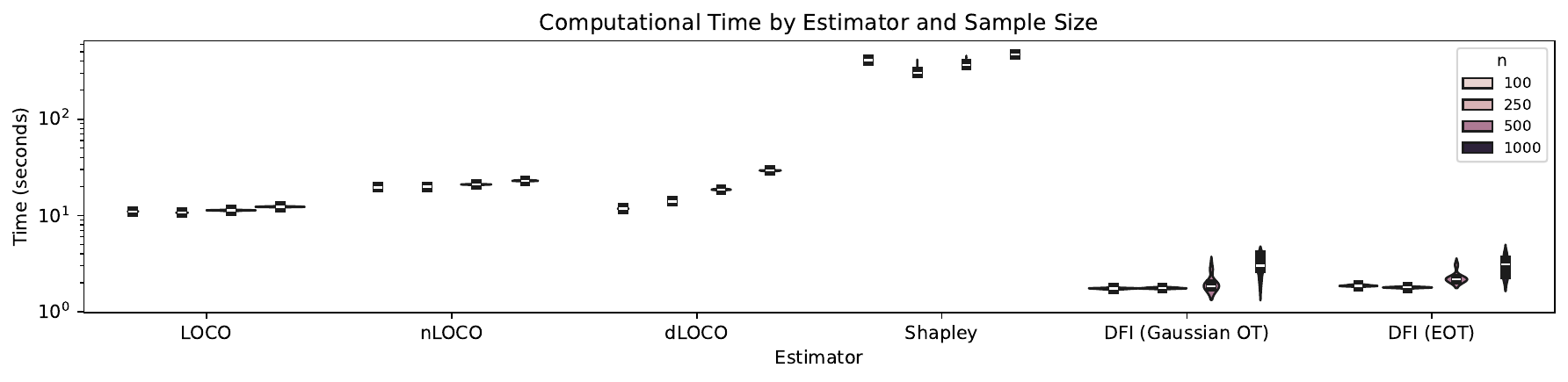}
        \caption{
                Computational time of different FI measures in the logarithmic scale.
                For DFI, the computational time includes the estimation of transport plans and importance scores.}
        \label{fig:time}
    \end{figure}

\section{Real data}\label{sec:real-data}
    In this section, we use DFI to attribute predictive signal in a real-world HIV-1 neutralization-resistance dataset \cite{montefiori2009measuring}, which is a critical step toward designing effective HIV-1 vaccines and antibody-based therapies.  
    \citet{williamson2023general} analyzed the same dataset using an algorithm-agnostic variable-importance framework based on a different estimand. 
    Here, we use DFI to target a complementary question: how shared predictive signal is attributed across correlated biological feature groups.
    
    Following the preprocessing steps of \cite{williamson2023general} and additional standardization of features and the outcome, the data comprises 611 samples and 832 features organized into 14 groups. 
    The outcome is a binary indicator of whether the 50\% inhibitory concentration ($\mathrm{IC}_{50}$) was right-censored, where high $\mathrm{IC}_{50}$ values signify viral resistance to neutralization \citep{montefiori2009measuring}. 
    The 14 feature groups represent functionally distinct aspects of viral envelope structure and biology, including the VRC01 binding footprint, CD4 binding sites, etc (\Cref{fig:real-data}). 
    These groups exhibit varying degrees of correlation due to biological relationships; for instance, features related to binding sites and glycosylation patterns are naturally correlated through their shared impact on antibody-virus interactions.

\begin{figure}[!t]
    \centering
    \includegraphics[width=\linewidth]{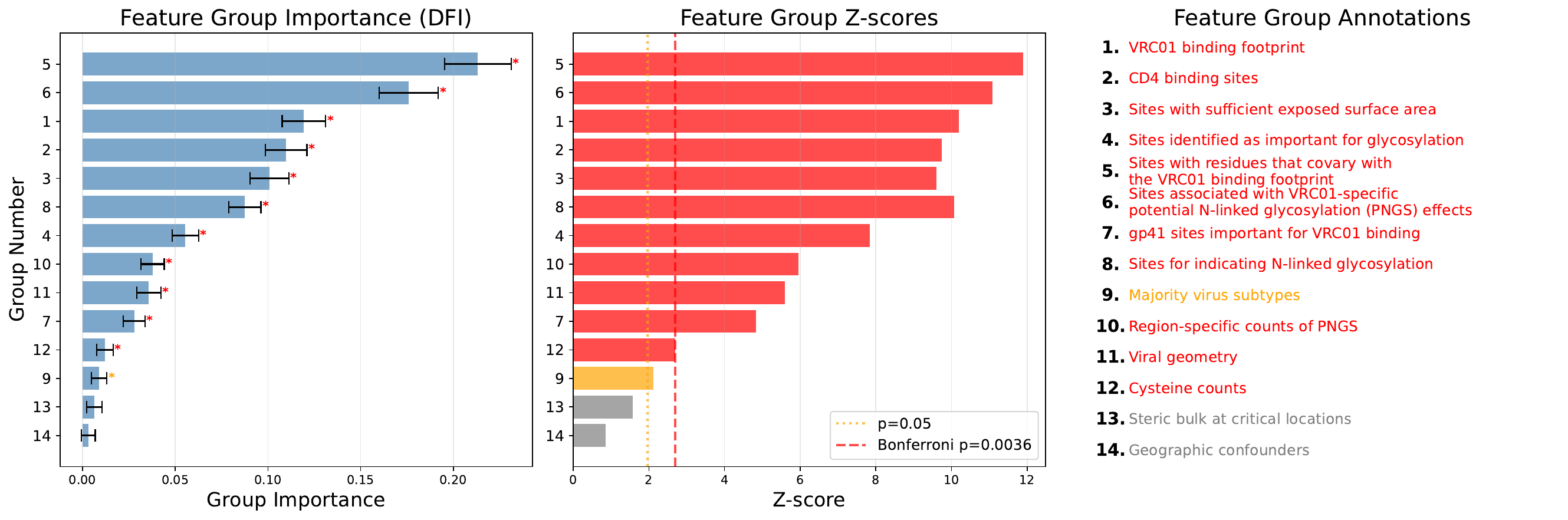}
    \includegraphics[width=\linewidth]{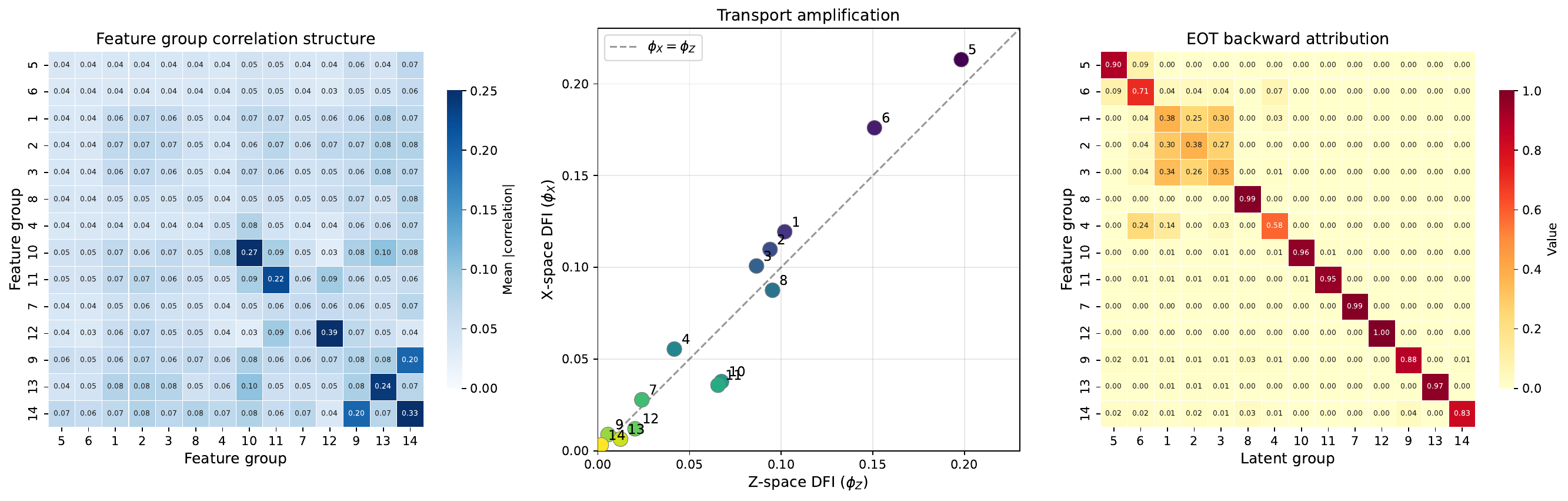}
    \caption{Feature importance of antibody against HIV-1 infection.
    DFI with EOT ($\varepsilon=10^{-3}$) is used to compute the group FI scores.
    (a) The bars show the point estimate, while the error bars indicate the values within one estimated standard deviation.
    (b) The corresponding Z-score for one-sided tests of $\cH_{0l}:\phi_{X_l}(\PP)\leq 0$.
    (c) Feature group description. 
    Stars ${\color{orange}\ast}$ and ${\color{red}\ast}$ denote importance deemed statistically significantly different from zero at the 0.05 and 0.0036 (0.05/14) levels, respectively.
    (d) Heatmap of absolute correlation between feature groups.
    (e) Scatter plot of transport amplification levels.
    (f) Heatmap of EOT backward attribution weights.
    }
    \label{fig:real-data}
\end{figure}

When individual features are strongly correlated or belong to a common scientific unit, individual importance scores may be less stable or less meaningful than the importance of the whole group.
The additive structure of the DFI measure \eqref{eq:def-phi-X} facilitates the assessment of group importance for a feature set $X_{\cG}$ with $\cG\subseteq[d]$, defined as 
\begin{align*}
\phi_{X_{\cG}}(\PP) := \sum_{l\in\cG}\sum_{j=1}^d \EE\left[ \VV\left( \EE[Y \mid Z] \, \Big| \, Z_{-j} \right) \left( \frac{\partial X_l}{\partial Z_j} \right)^2 \right].
\end{align*}
These group-level scores should be interpreted as attribution of predictive signal in the fitted prediction problem, not as evidence that the corresponding groups are causally necessary or uniquely predictive after conditioning on all other groups.
The statistical estimation and inference framework from \Cref{sec:stat-est-inf} extends directly to this group setting. 
To ensure stable estimates, we estimate the EOT coupling on the full dataset and use 2-fold cross-fitting for the regression function and importance score estimation. To address the issue of inference near the null noted in \Cref{rmk:null}, we inflate the estimated variance by $n^{-1/2}/z_{1-\alpha}$ when computing p-values at a significance level of $\alpha$.

For $\alpha=0.05$ (with Bonferroni correction for one-sided tests), DFI identifies 19 of 832 features as significant in the disentangled space and 30 of 832 in the original feature space.
The group-level importance results are summarized in \Cref{fig:real-data,fig:sensitivity-epsilon-sens50} for sensitivities in EOT regularization parameter $\varepsilon$.
The analysis reveals that a few feature groups are significantly associated with VRC01 neutralization resistance. The three most influential groups are Group 5 (Sites with residues covarying with the VRC01 binding footprint), Group 6 (Sites linked to VRC01-specific potential N-linked glycosylation effects), and Group 1 (VRC01 binding footprint). These groups exhibit the highest importance and Z-scores, indicating strong statistical significance that far exceeds the Bonferroni-corrected threshold. This finding suggests that the interplay between the VRC01 binding footprint and glycosylation patterns is a primary determinant of neutralization sensitivity.

Groups 2 (CD4 binding sites), 3 (sites with sufficient exposed surface area), 8 (sites indicating N-linked glycosylation), and 4 (sites identified as important for glycosylation) also show significant, though smaller, attributed importance. The remaining groups, including those related to viral geometry and geographic confounders, have comparatively minor importance, with several failing to reach statistical significance. 
These findings are scientifically concordant with the analysis of \citet{williamson2023general}, in the sense that both analyses identify VRC01-footprint and glycosylation-related features as central to neutralization resistance. 
The distinction is interpretive: \citet{williamson2023general} analyze this dataset through an algorithm-agnostic variable-importance framework based on a different estimand, whereas DFI targets the attribution of shared predictive signal across correlated observed feature groups. 
Thus, the agreement in leading biological signals should be viewed as external validation of the relevance of these groups, while the DFI analysis provides an additional decomposition of how this signal is distributed under feature dependence.

More specifically, \Cref{fig:real-data}~(d) shows that the leading groups form a correlated biological block: Group~5 (sites with residues covarying with the VRC01 binding footprint), Group~6 (VRC01-specific PNGS effects), Group~1 (VRC01 binding footprint), Group~2 (CD4 binding sites), and Group~3 (sites with sufficient exposed surface area) are mutually dependent measurements of VRC01-footprint and glycosylation biology.
Standard conditional-incremental measures may treat such shared signal as redundancy, and only provide scalar summaries of importance scores.
By contrast, \Cref{fig:real-data}~(e) compares latent-space importance $\phi_Z$ with observed-space attribution $\phi_X$, revealing which groups are amplified or attenuated by the transport from independent latent factors to the original correlated measurements. 
\Cref{fig:real-data}~(f) further decomposes each observed group score by its latent source: Groups~5 and~6 are largely self-attributed, whereas Groups~1-3 share importance across a tightly interlocked VRC01-footprint/CD4-binding/exposed-surface module. DFI therefore does not merely report that these groups are important; it assigns shared predictive signal to biologically interpretable correlated modules, while retaining formal uncertainty quantification.

\section{Discussion}\label{sec:discussion}

We introduced DFI as a population-level framework for feature-importance attribution under dependent covariates. 
The central premise is that feature importance is not a single universal target: when predictors are correlated, feature selection, compression, acquisition, and post-hoc attribution can require different estimands. 
DFI targets the latter problem. 
Rather than treating shared predictive information only as redundancy to be conditioned away, DFI maps the observed covariates to an independent latent representation, computes importance in that latent space, and attributes the resulting predictive signal back to the observed covariates through barycentric sensitivities.

This interpretation is essential. 
DFI is not intended as a criterion for selecting a minimal nonredundant subset of variables. 
A large DFI score indicates that an observed covariate carries or reflects predictive signal under the chosen disentangling geometry; it does not imply that the covariate is indispensable once correlated covariates are available. 
Conversely, conditional-incremental measures remain natural for feature selection or acquisition, where the operational question is whether a feature adds predictive value beyond the others. 
Thus, DFI should be viewed as complementary to conditional-incremental FI methods, not as a replacement for them.

The proposed framework is also estimand-indexed. 
Under the EOT construction, DFI depends on the transport cost \(c\), reference law \(\PP_Z\), and regularization level \(\varepsilon\). 
For any fixed specification, the target is a well-defined population functional, but changing the specification changes the attribution target. 
This dependence is not a defect; it is analogous to the dependence of other FI methods on the loss function, perturbation scheme, conditioning set, or baseline distribution. 
In practice, the chosen specification should be reported, and sensitivity to \(\varepsilon\) or other design choices should be assessed when conclusions depend on fine-grained rankings.

Our theoretical results support this interpretation by establishing efficient influence-function representations and asymptotic normality for both latent and attributed DFI estimators, with second-order remainders that are negligible when the nuisance estimators converge at \(\op(n^{-1/4})\) rates. 
The functional ANOVA characterization clarifies how latent DFI aggregates main effects and interactions, while the Gaussian linear case connects attributed DFI to the classical \(R^2\) decomposition for correlated regressors. 
Computationally, DFI avoids repeated reduced-model fitting and direct conditional covariate distribution estimation, replacing these with estimation of a disentangling coupling and sampling from a product reference law.

Several limitations remain. 
First, estimating the EOT coupling and barycentric sensitivities can be challenging in high dimensions, especially when the sample size is modest relative to the number of covariates. 
Second, the present theory focuses on squared-error loss and global population-level attribution; extensions to classification, other losses, and local or individual-level attribution are natural next steps. 
Third, while DFI can attribute shared predictive signal under dependence, it does not by itself provide causal importance, variable-selection guarantees, or a downstream decision rule. 
Developing scalable implementations, incorporating structural information such as feature groups or networks, and constructing reliable monitoring or ranking procedures under distribution shift are important directions for future work.

\clearpage
\appendix
\crefalias{section}{appendix}
\crefalias{subsection}{appendix}

\setcounter{section}{0} 
\setcounter{theorem}{0} 
\counterwithin{theorem}{section}

\renewcommand{\thesection}{\Alph{section}}
\setcounter{table}{0}
\renewcommand{\thetable}{\thesection\arabic{table}}
\setcounter{figure}{0} 
\renewcommand\thefigure{\thesection\arabic{figure}}
\renewcommand{\thealgorithm}{\thesection.\arabic{algorithm}}

\begin{center}
    {\LARGE\bf Appendix}
\end{center}

\bigskip

The appendix includes the proof for all the theorems, computational details, and extra experimental results.

\bigskip

\noindent\textbf{Outline.} The structure of the appendix is listed below:
\begin{table}[!ht]
\centering \small
\begin{tabularx}{0.95\textwidth}{l l D}
    \toprule    
    \multicolumn{2}{c}{\textbf{Appendix}} & \textbf{Content} \\
    \midrule \addlinespace[0.5ex]
    \multirow{3}{*}{\Cref{app:proof-properties}} & \Cref{app:sub:lem:null-feature} & Proof of \Cref{lem:null-feature}. \\
     & \Cref{app:sub:lem:equiv} & Proof of \Cref{lem:equiv}. \\
     & \Cref{app:sub:prop:free-cor} & Proof of \Cref{prop:free-cor}. \\\addlinespace[0.5ex] \cmidrule(l){1-3}\addlinespace[0.5ex] 
    \multirow{5}{*}{\Cref{app:proof}} & \Cref{app:sub:them:error-decom} & Proof of \Cref{thm:error-decom-eot}. \\
        & \Cref{app:subsec:proof-eif-eot-weighted} & Proof of \Cref{prop:eif-eot-weighted}. \\
        & \Cref{app:subsec:proof-error-decom-X-eot} & Proof of \Cref{thm:error-decom-X-eot}. \\
        & \Cref{app:subsec:decomp} & Proof of \Cref{lem:latent-importance-decomp}. \\
        & \Cref{app:sub-prop:decomp-phi} & Proof of \Cref{prop:decomp-phi}. \\
        \addlinespace[0.5ex] \cmidrule(l){1-3}\addlinespace[0.5ex] 
    \multirow{4}{*}{\Cref{app:lemmas}} & \Cref{app:sub:basic} & Basic properties of estimators: \Cref{lem:L2-boundedness,lem:HS-bound}. \\
        & \Cref{app:sub:loco} & Lemmas related to LOCO estimation: \Cref{lem:decomp-varphi-diff,lem:eta-j}. \\
        & \Cref{app:sub:DVI} & Lemmas related to the disentangled feature importance estimator: \Cref{lem:EIF-product,lem:eif-plug-in-decomp}. \\ \addlinespace[0.5ex] \cmidrule(l){1-3}\addlinespace[0.5ex] 
    \multirow{3}{*}{\Cref{app:sec:ot-maps}} & \Cref{app:subsec:ot-maps} & Optimal transport maps and disentangled representations. \\
        & \Cref{app:subsec:proof-ot} & Proof for results related to optimal transport. \\
        & \Cref{app:sub:KR} & Extension to Knothe-Rosenblatt transport map. \\
        \addlinespace[0.5ex] \cmidrule(l){1-3}\addlinespace[0.5ex]
    \multirow{3}{*}{\Cref{app:sec:simu}} & \Cref{app:subsec:implement} & Practical considerations. \\
        & \Cref{app:subsec:M3} & Detailed calculation under Model \hyperref[M3]{(M3)} \\
        & \Cref{app:subsec:extra-simu} & Extra simulation results. \\
    \bottomrule
\end{tabularx}
\end{table}

\clearpage

\section{Proof in \Cref{sec:FI,sec:DFI}}\label{app:proof-properties}

\subsection[Proof of Lemma \ref{lem:null-feature}]{Proof of \Cref{lem:null-feature}}\label{app:sub:lem:null-feature}
    \begin{proof}
    First suppose that $\EE[Y\mid X]=\EE[Y\mid X_{-j}]$ almost surely, i.e.\ $\mu(X)=\mu_{-j}(X_{-j})$ almost surely.
    Then $\psi^{\loco}_{X_j}= \EE[(Y-\mu_{-j}(X_{-j}))^2]-\EE[(Y-\mu(X))^2]=0$.
    Moreover, by construction $X^{(j)}_{-j}=X_{-j}$ and $\mu(X^{(j)})=\mu_{-j}(X_{-j})$ almost surely as well, so
    \[
    \psi^{\cpi}_{X_j}
    =\frac{1}{2}\EE\!\left[(Y-\mu(X^{(j)}))^2-(Y-\mu(X))^2\right]=0.
    \]

    Conversely, under $\ell_2$ loss, \Cref{lem:equiv} implies
    \[
    \psi^{\loco}_{X_j}=\psi^{\cpi}_{X_j}=\EE\!\left[\VV\!\left(\mu(X)\mid X_{-j}\right)\right]\ge 0.
    \]
    Hence $\psi^{\loco}_{X_j}=\psi^{\cpi}_{X_j}=0$ holds if and only if $\VV(\mu(X)\mid X_{-j})=0$ almost surely, which is equivalent to
    $\mu(X)=\EE[\mu(X)\mid X_{-j}]=\mu_{-j}(X_{-j})$ almost surely, i.e.\ $\EE[Y\mid X]=\EE[Y\mid X_{-j}]$ almost surely.
    \end{proof}

\subsection[Proof of Lemma \ref{lem:equiv}]{Proof of \Cref{lem:equiv}}\label{app:sub:lem:equiv}

\begin{proof}
        For any random variable $W$, its variance can be expressed as $\VV[W] = \EE[(W- \EE[W])^2] = \EE[(W- W')^2]/2$ where $W'$ is an independent copy of $W$.
        Conditioned on $X_{-j}$, it follows that
        \begin{align*}
            \VV[\mu(X) \mid X_{-j}] &= \EE[ (\mu(X) - \EE[\mu(X)\mid X_{-j}])^2\mid X_{-j}] = \EE[(\mu(X) - \mu_{-j}(X_{-j}))^2\mid X_{-j}]  \\
            \VV[\mu(X) \mid X_{-j}] &= \frac{1}{2} \EE[ (\mu(X) - \EE[\mu(X^{(j)}) \mid X_{-j}])^2\mid X_{-j}] = \frac{1}{2}\EE[(\mu(X) - \mu(X^{(j)}))^2\mid X_{-j}].
        \end{align*}
        On the other hand, for $\ell_2$ loss, we have
        \begin{align*}
            \psi^{\loco}_{X_j}&= \EE[(Y - \mu_{-j}(X_{-j}))^2] - \EE[(Y- \mu(X) )^2]\\
            &= 
            \EE[(\mu(X) - \mu_{-j}(X_{-j}))^2] - 2\EE[(\mu(X) - Y)(\mu(X) - \mu_{-j}(X_{-j}))]    \\
            &= \EE[(\mu(X) - \mu_{-j}(X_{-j}))^2] ,
        \end{align*}
        because $\EE[Y(\mu(X) - \mu_{-j}(X_{-j})) ] = \EE[\EE[Y(\mu(X) - \mu_{-j}(X_{-j})) \mid X]] = 
        \EE[\mu(X) (\mu(X) - \mu_{-j}(X_{-j}))]
        $ by iterative expectation.
        Analogously, one has
        \begin{align*}
            2 \psi^{\cpi}_{X_j} &= \EE[(Y - \mu(X^{(j)}))^2] - \EE[(Y- \mu(X) )^2]\\
            &= \EE[(\mu(X) - \mu(X^{(j)}))^2] - 2\EE[(\mu(X) - Y)(\mu(X) - \mu(X^{(j)}))]    \\
            &= \EE[(\mu(X) - \mu(X^{(j)}))^2] .
        \end{align*}
        Combining the above results finishes the proof.
    \end{proof}

\subsection[Proof of Proposition \ref{prop:free-cor}]{Proof of \Cref{prop:free-cor}}\label{app:sub:prop:free-cor}
    \begin{proof}
        By definition,
        \(
        \phi_{Z_j}
        =
        \mathbb E\left[
        \mathbb V\{\eta(Z)\mid Z_{-j}\}
        \right].
        \)
        Since conditional variance is nonnegative, \(\phi_{Z_j}=0\) if and only if
        \(
        \mathbb V\{\eta(Z)\mid Z_{-j}\}=0
        \) almost surely.
        This is equivalent to
        \(
        \eta(Z)
        =
        \mathbb E\{\eta(Z)\mid Z_{-j}\}
        =
        \eta_{-j}(Z_{-j})
        \) almost surely.
        The converse follows immediately from the same display.
    \end{proof}

\clearpage
\section{Proof of main results}\label{app:proof}

\subsection[Proof of Theorem \ref{thm:error-decom-eot}]{Proof of \Cref{thm:error-decom-eot}}\label{app:sub:them:error-decom}
    \begin{proof}
    We split the proof into three steps.

    \paragraph{Step 1: Error decomposition for the LOCO functional.}
    From \citet[Lemma 1]{williamson2021nonparametric}, the efficient influence function for $\phi_{Z_j} = \EE[ (\eta(Z) - \eta_{-j}(Z_{-j}))^2 ]$ is $\varphi(O;\PP) = 2 (Y - \eta(Z))(\eta(Z) - \eta_{-j}(Z_{-j})) + (\eta(Z) - \eta_{-j}(Z_{-j}))^2 - \phi_{Z_j}$.
    Under \Cref{asm:data}, $\varphi(O;\PP)$ is well-defined and belongs to $L_2(\PP)$.
    Furthermore, for any given estimator $\hat{\PP}$ (with associated estimators $\hat{\eta}$ and $\hat{\eta}_{-j}$) of the nonparametric model $\PP$, the LOCO estimator
    \begin{align}
        \hat{\phi}_{Z_j}^{\loco}(O;\PP \mid \hat{\eta},\hat{\eta}_{-j}) := \PP_n \left[ (Y - \hat{\eta}_{-j}(Z_{-j}))^2 - (Y - \hat{\eta}(Z))^2 \right],\label{eq:phi-hat-loco}
    \end{align}
    admits the following error decomposition (see the proof of \citet[Theorem 1]{williamson2021nonparametric}):
    \begin{align*}
        \hat{\phi}_{Z_j}^{\loco}(O;\PP \mid\hat{\eta},\hat{\eta}_{-j}) - \phi_{Z_j} &= (\PP_n-\PP) \{\varphi(O;\PP)\} + R(\hat{\PP}, \PP) + H(\hat{\PP}, \PP),
    \end{align*}
    where the first term scaled by $\sqrt{n}$ is asymptotically normal by the Central Limit Theorem (since $\varphi$ is mean-zero and square-integrable), and the remaining terms read as
    \begin{align}
        R(\hat{\PP}, \PP) &= \|\hat{\eta}_{-j} - \eta_{-j}\|_{L_2}^2 - \|\hat{\eta} - \eta\|_{L_2}^2 \label{eq:R-loco}\\
        H(\hat{\PP}, \PP) &= (\PP_n-\PP)\{\varphi(O;\hat{\PP}) - \varphi(O;\PP)\}.
    \end{align}
    Define the estimation errors:
    \[
    \delta(Z):=\hat{\eta}(Z)-\eta(Z),
    \qquad
    \delta_{-j}(Z_{-j}):=\hat{\eta}_{-j}(Z_{-j})-\eta_{-j}(Z_{-j}).
    \]
    Then we have 
    \[ | R(\hat{\PP}, \PP)| \leq \|\delta_{-j}\|_{L_2}^2 + \|\delta\|_{L_2}^2 . \]
    By Cauchy-Schwarz inequality and the fact that
    \(\PP_n-\PP\) is an average of \(n\) i.i.d. mean-zero variables,
    \begin{align}
        |H(\hat{\PP},\PP)|
        &\le
        \|\varphi(O;{\hat{\PP}})-\varphi(O;{\PP})\|_{L_2}\,
        \sqrt{\frac{\mathrm{Var}[\varphi(O;{\hat{\PP}})-\varphi(O;{\PP})]}{n}} \notag\\
        &\;=\;
        \Op\left(
        \frac{\|\varphi(O;{\hat{\PP}})-\varphi(O;{\PP})\|_{L_2}}{\sqrt{n}}
        \right). \label{A.4}
    \end{align}
    Further by \Cref{lem:decomp-varphi-diff}, we have
    \[
    \;|H(\hat{\PP},\PP)|
     = \Op\left(
     \frac{1}{\sqrt{n}}\left(
     \|\delta\|_{L_2}+\|\delta_{-j}\|_{L_2}
     \right)
     \;+\;
     \frac{1}{\sqrt{n}}\left(
     \|\delta\|_{L_2}+\|\delta_{-j}\|_{L_2}
     \right)^{2} \right).
    \]
    Consequently, we have
    \begin{align}
        \hat{\phi}_{Z_j}^{\loco}(O;\PP \mid\hat{\eta},\hat{\eta}_{-j}) - \phi_{Z_j}(O;\PP) &= (\PP_n-\PP) \{\varphi(O;\PP)\}         
        + \Op(\|\delta\|_{L_2}^2 + \|\delta_{-j}\|_{L_2}^2) \notag\\
        &\qquad  + \frac{1}{\sqrt{n}}\Op(\|\delta\|_{L_2} + \|\delta_{-j}\|_{L_2}) 
        .\label{eq:error-loco}
    \end{align}

    \paragraph{Step 2: Error decomposition for the CPI functional.}
    Next, we relate the CPI functional to the LOCO functional.
    Recall that our estimator is defined as
    \[
    \hat{\phi}_{Z_j} := \frac{1}{2} \PP_n \left[ (Y - \hat{\eta}(Z^{(j)}))^2 - (Y - \hat{\eta}(Z))^2 \right].
    \]
    We decompose the estimation error through an intermediate oracle estimator using the true regression funciton $\hat\eta_{-j}=\eta_{-j}$ (so that $\delta_{-j} = 0$):
    \begin{align}
        \hat\phi_{Z_j}-\phi_{Z_j} & \;=\; [\hat\phi_{Z_j}-\hat{\phi}_{Z_j}^{\loco}(\PP;\hat{\eta},\eta_{-j})] + [\hat{\phi}_{Z_j}^{\loco}(\PP;\hat{\eta},\eta_{-j}) -\phi_{Z_j}    ] \notag\\
        &= \frac{1}{2}\PP_n\left[ (Y - \hat{\eta}(Z^{(j)}))^2 - (Y - \eta_{-j}(Z_{-j}))^2 \right] \notag\\
        &\quad+\frac{1}{2}\PP_n\left[ (Y - \hat{\eta}(Z))^2 - (Y - \eta_{-j}(Z_{-j}))^2 \right] \notag\\
        &\quad + (\PP_n-\PP) \{\varphi(O;\PP)\} + \Op(n^{-1/2}\|\delta\|_{L_2} + \|\delta\|_{L_2}^2) \notag\\
        &= \frac{1}{2}R_e + \frac{1}{2} R_b  + (\PP_n-\PP) \{\varphi(O;\PP)\} + \Op(n^{-1/2}\|\delta\|_{L_2} + \|\delta\|_{L_2}^2), \label{eq:decomp-cpi}
    \end{align}
    where the emprical process terms $R_e$ and remaining bias terms $R_b$ are defined as:
    \begin{align*}
        R_e &= (\PP_n-\PP) [ \hat\phi_{Z_j}-\hat{\phi}_{Z_j}^{\loco}(\PP;\hat{\eta},\eta_{-j})  ]\\
        R_b & = \PP [ \hat\phi_{Z_j}-\hat{\phi}_{Z_j}^{\loco}(\PP;\hat{\eta},\eta_{-j})  ].
    \end{align*}
    Before analyzing each of these terms, we introduce
    \[
      \zeta:=\eta(Z)-\eta_{-j}(Z_{-j}),
      \quad
      \zeta_j:=\eta(Z^{(j)})-\eta_{-j}(Z_{-j}),
      \quad
      \delta_j:=\hat\eta(Z^{(j)})-\eta(Z^{(j)}),
    \]
    and write $\varepsilon:=Y-\eta$, $U:=\varepsilon+\zeta$, $U_j:=\varepsilon+\zeta_j$.
    Because
    \begin{align*}
         (Y-\hat\eta(Z))^2-(Y-\eta_{-j}(Z_{-j}))^2
          &=(U-\delta)^2-U^2
          =\delta^2-2U\delta, \\
        (Y-\hat\eta(Z^{(j)}))^2-(Y-\eta_{-j}(Z_{-j}))^2
          &=(U_j-\delta_j)^2-U^2
          =\delta_j^2-2U_j\delta_j,
    \end{align*}
    we have
    \begin{align*}
        R_e &= \frac12\,(\PP_n-\PP) [(Y-\hat\eta(Z^{(j)}))^2-(Y-\eta_{-j}(Z_{-j}))^2 ]
         +\frac12\,(\PP_n-\PP) [(Y-\hat\eta(Z))^2-(Y-\eta_{-j}(Z_{-j}))^2 ]\\
         &= \frac12\,(\PP_n-\PP) [\delta^2-2U\delta+\delta_j^2-2U_j\delta_j ].
    \end{align*}
    Because $\|\delta\|_{L_\infty}\le 2C$ and $\|U\|_{L_\infty}\le 3C$ from \Cref{lem:L2-boundedness} with \Cref{asm:data}~\ref{asm:data:response}, we have
    \begin{align*}
        \|\delta^2\|_{L_2} &\le \|\delta\|_{L_\infty}\,\|\delta\|_{L_2} \le2C\,\|\delta\|_{L_2} \\
        \|U\delta\|_{L_2} & \le \|U\|_{L_\infty}\,\|\delta\|_{L_2} \leq 3C\|\delta\|_{L_2} .    
    \end{align*}
    Additionally,
    \[
    \|\delta_j\|_{L_2}^2
          =\EE[\EE [(\hat\eta(Z^{(j)})-\eta(Z^{(j)}))^{2}\mid Z_{-j} ]]
          =\EE [(\hat\eta(Z)-\eta(Z))^{2} ]
          =\|\delta\|_{L_2}^{2},
    \]
    Hence, it follows that
    \begin{align}
        R_e \;=\;\Op(n^{-1/2}\,\|\delta^2-2U\delta+\delta_j^2-2U_j\delta_j\|_{L_2} )
            =  \Op(n^{-1/2}\,\|\delta\|_{L_2}).\label{eq:E}
    \end{align}

    For the bias term, note that
    \[
        R_b
       =\frac12\,\EE [(Y-\hat\eta(Z^{(j)}))^2-(Y-\eta_{-j}(Z_{-j}))^2 ]
         +\frac12\,\EE [(Y-\hat\eta(Z))^2-(Y-\eta_{-j}(Z_{-j}))^2 ] =: A_1 + A_2.
    \]    
    Using $\hat\eta(Z^{(j)})-\eta_{-j}(Z_{-j})=\zeta_j+\delta_j$ and
    $\EE[Y-\eta_{-j}(Z_{-j})\mid Z^{(j)},Z_{-j}]=0$, the two terms of $R$ can be written as:
    \begin{align*}
        A_1
        &= \frac{1}{2}\EE[(Y-\eta(Z^{(j)})+ \eta(Z^{(j)})-\hat\eta(Z^{(j)}))^2-(Y-\eta_{-j}(Z_{-j}))^2]    \\
        &= \frac{1}{2}\EE[(Y-\eta(Z^{(j)})^2 + (\eta(Z^{(j)})-\hat\eta(Z^{(j)}))^2-(Y-\eta_{-j}(Z_{-j}))^2]\\
        &= -\frac{1}{2}\phi_{Z_j} +\frac{1}{2}\EE[(\eta(Z^{(j)})-\hat\eta(Z^{(j)}))^2] \\
        &= -\frac{1}{2}\phi_{Z_j} + \frac{1}{2}\|\delta\|_{L_2}^2\\
        A_2
        &= \frac12\,\EE [(Y- \eta(Z) +\eta(Z)-\hat\eta(Z))^2-(Y-\eta_{-j}(Z_{-j}))^2 ]\\
        &= \frac12\,\EE [(Y- \eta(Z))^2 + (\eta(Z)-\hat\eta(Z))^2-(Y-\eta_{-j}(Z_{-j}))^2 ]\\
        &= \frac{1}{2}\phi_{Z_j} + \frac{1}{2}\|\delta\|_{L_2}^2,
    \end{align*}
    Hence, one has
    \begin{align}
        R_b = \|\delta\|_{L_2}^2    \label{eq:R}
    \end{align}

    Combining \eqref{eq:decomp-cpi}, \eqref{eq:E} and \eqref{eq:R} yields that
    \begin{align}
        \hat\phi_{Z_j}-\phi_{Z_j} & \;=\;  (\PP_n-\PP) \{\varphi(O;\PP)\} + \Op(n^{-1/2}\|\delta\|_{L_2} + \|\delta\|_{L_2}^2).
       \label{eq:decomp-cpi-2}
    \end{align}

    \paragraph{Step 3: bounding $\|\delta\|_{L_2(\PP_Z)}=\|\hat\eta-\eta\|_{L_2(\PP_Z)}$.}
    Recall that under the EOT coupling $\gamma$ with marginals $(\PP_X,\PP_Z)$,
    \[
    \eta(z)=\int \mu(x)\,k(x\mid z)\,\PP_X(\rd x).
    \]
    Since $\PP_X$ is not observed, the feasible estimator is defined as
    \[
    \hat\eta(z):=\PP_n\{\hat\mu(X)\hat k(X\mid z)\}
    = \frac1n\sum_{i=1}^n \hat\mu(X_i)\hat k(X_i\mid z).
    \]
    Introduce the intermediate (infeasible) quantity
    \[
    \tilde\eta(z):=\int \hat\mu(x)\hat k(x\mid z)\,\PP_X(\rd x),
    \]
    and decompose
    \begin{align}
        \delta(z)=\underbrace{\{\hat\eta(z)-\tilde\eta(z)\}}_{=:A(z)}
        +\underbrace{\{\tilde\eta(z)-\eta(z)\}}_{=:B(z)}.\label{eq:delta}
    \end{align}

    For the first term, note that
    \(
    A(z)=(\PP_n-\PP)\{g_z(X)\}
    \)
    with $g_z(x):=\hat\mu(x)\hat k(x\mid z)$.
    Conditioning on $(\hat\mu,\hat k)$, for each fixed $z$ the summands
    $g_z(X_i)-\PP g_z$ are i.i.d.\ mean-zero. By Fubini theorem and the usual variance calculation,
    \begin{align*}
    \EE\!\left[\|A\|_{L_2(\PP_Z)}^2 \,\middle|\, \hat\mu,\hat k\right]
    &= \int \VV\!\left((\PP_n-\PP)g_z(X)\,\middle|\,\hat\mu,\hat k\right)\,\PP_Z(\rd z) \\
    &= \frac1n\int \VV_{\PP}\!\big(g_z(X)\big)\,\PP_Z(\rd z)
    \\
    &\le\frac1n\int \EE_{\PP}\!\big[g_z(X)^2\big]\,\PP_Z(\rd z) \\
    &= \frac1n \|\hat\mu(X)\hat k(X\mid Z)\|_{L_2(\PP_X\otimes \PP_Z)}^2.
    \end{align*}
    Hence,
    \begin{align}
        \|A\|_{L_2(\PP_Z)} = \Op\!\left(n^{-1/2}\,\|\hat\mu\,\hat k\|_{L_2(\PP_X\otimes \PP_Z)}\right). \label{eq:delta-A}
    \end{align}
    Under the construction that $\|\hat\mu(X)\hat k(X\mid Z)\|_{L_2(\PP_X\otimes\PP_Z)}=\Op(1)$, this simplifies to $\|A\|_{L_2(\PP_Z)}=\Op(n^{-1/2})$.

    For the second term, write
    \[
    B(z)=\int \{\hat\mu-\mu\}(x) \hat k(x\mid z)\PP_X(\rd x)
    \;+\;\int \mu(x) \{ \hat k - k\}(x\mid z)\PP_X(\rd x)
    =:B_1(z)+B_2(z),
    \]
    and let $\varepsilon_\mu:=\|\hat\mu-\mu\|_{L_2(\PP_X)}$, $\varepsilon_k:=\|\hat k-k\|_{L_2(\PP_X\otimes \PP_Z)}$.
    Then, from \Cref{lem:HS-bound} and \Cref{asm:data}~\ref{asm:data:response}-\ref{asm:data:k}, we have
    \begin{align}
        \|B\|_{L_2(\PP_Z)} &\le \|B_1\|_{L_2(\PP_Z)} + \|B_2\|_{L_2(\PP_Z)} \notag \\
        &\le \|\hat k\|_{L_2(\PP_X\otimes \PP_Z)}\varepsilon_\mu + \|\mu\|_{L_2(\PP_X)}\,\varepsilon_k \notag \\
        &\lesssim \varepsilon_\mu + \varepsilon_k.    \label{eq:delta-B}
    \end{align}

    Finally, combining \eqref{eq:delta}, \eqref{eq:delta-A} and \eqref{eq:delta-B} through the triangle inequality yields
    \begin{align}
        \|\delta\|_{L_2(\PP_Z)}
        \le \|A\|_{L_2(\PP_Z)}+\|B\|_{L_2(\PP_Z)}
        = \Op(n^{-1/2} + \varepsilon_\mu + \varepsilon_k).    \label{eq:delta-final}
    \end{align}
    Combining this with \eqref{eq:decomp-cpi-2} completes the proof.

    \end{proof}

\subsection[Proof of Proposition \ref{prop:eif-eot-weighted}]{Proof of \Cref{prop:eif-eot-weighted}}\label{app:subsec:proof-eif-eot-weighted}
\begin{proof}
    We split the proof into two steps.

    \paragraph{Step 1: well-definedness of the EIF.}

    From \Cref{lem:L2-boundedness}, under \Cref{asm:data}, we have $\eta_{-j},\eta(Z)\in L_\infty(P)$, $\Delta_j(Z)=\eta(Z)-\eta_{-j}(Z_{-j})\in L_\infty(\PP_{Z})$, $v_j(Z_{-j})=\VV(\eta(Z)\mid Z_{-j})\in L_\infty(\PP_{Z_{-j}})$
    Similarly, we have $\|w_{j,l}(Z_{-j})\|_{L_\infty(\PP_{Z_{-j}})}<\infty$ as assumed.
    In particular, $\psi_{j,l}(O;\PP):=\EE[v_j(Z_{-j})w_{j,l}(Z_{-j})] \in L_2(\PP)$ is well-defined by Cauchy--Schwarz, and so is $\phi_{X_l}(O;\PP)=\sum_{j=1}^d\psi_{j,l}(P)$.

    It remains to check $\varphi_{X_l}(O;P)\in L_2(\PP)$.
    For the first bracket in \eqref{eq:eif-phi-X-eot}, apply H\"older twice:
    \begin{align*}
        \|w_{j,l}(Z_{-j})(Y-\eta(Z))\Delta_j(Z)\|_{L_2}
        &\le \|w_{j,l}(Z_{-j})\|_{L_4}\,\|(Y-\eta(Z))\Delta_j(Z)\|_{L_4}\\
        &\le \|w_{j,l}\|_{L_4}\,\|Y-\eta\|_{L_8}\,\|\Delta_j\|_{L_8}\\
        &<\infty,
    \end{align*}
    and similarly
    \[
    \|w_{j,l}(Z_{-j})\Delta_j(Z)^2\|_{L_2}
    \le \|w_{j,l}\|_{L_4}\,\|\Delta_j^2\|_{L_4}
    = \|w_{j,l}\|_{L_4}\,\|\Delta_j\|_{L_8}^2<\infty.
    \]
    For the second bracket in \eqref{eq:eif-phi-X-eot},
    \[
    \|v_j(Z_{-j})\{\varphi_{\theta_{j,l}}(O;P)-w_{j,l}(Z_{-j})+\theta_{j,l}(P)\}\|_{L_2}
    \le \|v_j\|_{L_4}\,\big(\|\varphi_{\theta_{j,l}}\|_{L_4}+\|w_{j,l}\|_{L_4}+|\theta_{j,l}(P)|\big),
    \]
    where $\theta_{j,l}(P)<\infty$ follows from $w_{j,l}\in L_4$.
    Thus every summand in \eqref{eq:eif-phi-X-eot} is square-integrable, hence $\varphi_{X_l}(O;P)\in L_2(P)$.

    \paragraph{Step 2: EIF derivation via product rule.}
    For fixed $(j,l)$, write
    \(
    \psi_{jl}(\PP):=\EE[v_j(Z_{-j})w_{jl}(Z_{-j})]
    \)
    so that $\phi_{X_l}(\PP)=\sum_{j=1}^d\psi_{jl}(\PP)$.
    Apply \Cref{lem:EIF-product} with \(A(O):=v_j(Z_{-j})\) and \(B(O):=w_{jl}(Z_{-j})\).

    Note that $\psi_A(\PP)=\EE[v_j(Z_{-j})]=\EE[\VV(\eta(Z)\mid Z_{-j})]=\phi_{Z_j}(\PP)$.
    From \Cref{thm:error-decom-eot}, the EIF for $\psi_A(\PP):=\EE[A(O)]$ is given by
    \[
    \varphi_A(O;\PP)=2(Y-\eta(Z))\Delta_j(Z)+\Delta_j(Z)^2-\psi_A(\PP).
    \]
    Therefore the remainder term in \Cref{lem:EIF-product} is
    \[
    \alpha_A(O;\PP)
    =
    \varphi_A(O;\PP)-\{A(O)-\psi_A(\PP)\}
    =
    2(Y-\eta(Z))\Delta_j(Z)+\Delta_j(Z)^2 - v_j(Z_{-j}).
    \]

    By assumption (ii), $\psi_B(\PP)=\theta_{jl}(\PP)$ is pathwise differentiable with EIF $\varphi_{\theta_{jl}}(O;\PP)$, hence
    \[
    \alpha_B(O;\PP)
    =
    \varphi_{\theta_{jl}}(O;\PP)-\{B(O)-\psi_B(\PP)\}
    =
    \varphi_{\theta_{jl}}(O;\PP)-\big(w_{jl}(Z_{-j})-\theta_{jl}(\PP)\big).
    \]

    Combining via \Cref{lem:EIF-product}, the EIF for $\psi_{jl}(\PP)=\EE[A(O)B(O)]$ is
    \[
    \varphi_{jl}(O;\PP)
    =
    A(O)B(O)-\psi_{jl}(\PP)+B(O)\alpha_A(O;\PP)+A(O)\alpha_B(O;\PP).
    \]
    Substituting $A(O)=v_j(Z_{-j})$ and $B(O)=w_{jl}(Z_{-j})$ and simplifying cancels the $A(O)B(O)$ term, yielding
    \begin{align*}
    \varphi_{jl}(O;\PP)
    &=
    w_{jl}(Z_{-j})\{2(Y-\eta(Z))\Delta_j(Z)+\Delta_j(Z)^2\}
    \\
    &\quad
    +
    v_j(Z_{-j})\Big\{\varphi_{\theta_{jl}}(O;\PP)-\big(w_{jl}(Z_{-j})-\theta_{jl}(\PP)\big)\Big\}
    -\psi_{jl}(\PP).
    \end{align*}
    Summing over $j=1,\dots,d$ gives the stated EIF for $\phi_{X_l}(\PP)$, which finishes the proof.
\end{proof}

\subsection[Proof of Theorem \ref{thm:error-decom-X-eot}]{Proof of \Cref{thm:error-decom-X-eot}}\label{app:subsec:proof-error-decom-X-eot}
\begin{proof}
    Throughout, condition on the auxiliary sample used to construct the nuisance objects
    $(\hat\mu,\hat k,\partial_z\hat k)$ (equivalently $(\hat\eta,\hat w_{jl})$). Under the stated sample-splitting/cross-fitting, these nuisance estimators are independent of the $n$ observations used in $\PP_n$, so they can be treated as fixed when applying empirical-process bounds to $(\PP_n-\PP)$.
    We split the proof into 4 steps.

    \paragraph{Step 1: square-integrability of EIF.}
    We first verify the square-integrability conditions needed to ensure $\varphi_{X_l}(\cdot;\PP)\in L_2(\PP)$.
    \Cref{asm:bary-smooth}~\ref{asm:bary-smooth:iii} gives $s_{jl}\in L_8(\PP_Z)$, and therefore \(w_{jl}(Z_{-j})=\EE[s_{jl}(Z)^2\mid Z_{-j}] \in L_4(\PP_{Z_{-j}})\), uniformly over $(j,l)$. 
    Similar to Step 1 of the proof of \Cref{prop:eif-eot-weighted}, we have that $\varphi_{X_l}(O;\PP)$ in \eqref{eq:eif-phi-X-eot} is well-defined and square-integrable.

    \paragraph{Step 2: von Mises expansion at the EIF.}
    Let $\hat\varphi_{X_l}$ denote the plug-in version of \eqref{eq:eif-phi-X-eot}, obtained by replacing
    $(\eta,\Delta_j,v_j,w_{jl},\theta_{jl})$ by their nuisance estimators
    $(\hat\eta,\hat\Delta_j,\hat v_j,\hat w_{jl},\hat\theta_{jl})$.
    We denote
    \begin{align}
        \hat\phi_{jl}(\PP)
        &:=
        \frac12
        \PP_n\!\left[
        \hat w_{jl}(Z_{-j})
        \Big\{(Y-\hat\eta(Z^{(j)}))^2-(Y-\hat\eta(Z))^2\Big\}
        \right] \label{eq:phi-jl}\\
        \varphi_{jl}(O;\PP)
        &=
        w_{jl}(Z_{-j})\Big\{2\big(Y-\eta(Z)\big)\Delta_j(Z)+\Delta_j(Z)^2\Big\}\notag \\
        &\quad
        +
        v_j(Z_{-j})\Big\{\varphi_{\theta_{jl}}(O;\PP)-\big(w_{jl}(Z_{-j})-\theta_{jl}(\PP)\big)\Big\}
        -\psi_{jl}(\PP), \label{eq:eif-phi-jl}
    \end{align}
    so that $\phi_{X_l}= \sum_{j=1}^d \phi_{jl}(\PP)$ and $\varphi_{X_l}(O;\PP)=\sum_{j=1}^d \varphi_{jl}(O;\PP)$.
    A standard von Mises expansion (conditional on the auxiliary sample) yields
    \begin{equation}\label{eq:vm-psi}
    \hat\psi_{jl}-\psi_{jl}(\PP)
    =
    (\PP_n-\PP)\,\varphi_{jl}(O;\PP)
    +
    H_{n,jl}
    +
    R_{n,jl},
    \end{equation}
    where the empirical-process term is
    $
    H_{n,jl}:=(\PP_n-\PP)\{\hat\varphi_{jl}-\varphi_{jl}\}
    $
    and the plug-in bias is
    $
    R_{n,jl}:=\PP\{\hat\varphi_{jl}-\varphi_{jl}\}.
    $
    Here $\hat\varphi_{jl}$ denotes $\varphi_{jl}$ with $(\eta,\Delta_j,v_j,w_{jl},\theta_{jl})$ replaced by
    $(\hat\eta,\hat\Delta_j,\hat v_j,\hat w_{jl},\hat\theta_{jl})$.

    \paragraph{Step 3: empirical process term $H_{n,jl}$.}
    Conditioning on the auxiliary sample makes $\hat\varphi_{jl}$ fixed, hence
    \[
    |H_{n,jl}|
    =\Op\!\Big(n^{-1/2}\,\|\hat\varphi_{jl}-\varphi_{jl}\|_{L_2(\PP)}\Big).
    \]
    From \Cref{lem:eif-plug-in-decomp}, one has, with probability tending to 1,
    \[
    \|\hat\varphi_{jl}-\varphi_{jl}\|_{L_2(\PP)}
    \lesssim \cE_X,
    \]
    and therefore, we have
    \[
    H_{n,jl}=\Op\!\big(n^{-1/2}\cE_X\big),
    \qquad
    \sum_{j=1}^d H_{n,jl}=\Op\!\big(n^{-1/2}\cE_X\big).
    \]

    \paragraph{Step 4: bias term $R_{n,jl}$ via the product-rule structure in \Cref{prop:eif-eot-weighted}.}
    Recall that $\varphi_{jl}$ is constructed through \Cref{prop:eif-eot-weighted}, by applying \Cref{lem:EIF-product} to $A(O)=v_j(Z_{-j})$ and $B(O)=w_{jl}(Z_{-j})$.
    Concretely, \Cref{lem:EIF-product} shows that the EIF for $\psi_{jl}(\PP)=\EE[AB]$ augments the naive product $AB-\psi_{jl}$ by the two correction terms $B\,\alpha_A$ and $A\,\alpha_B$, where $\alpha_A$ and $\alpha_B$ are exactly the pathwise-derivative contributions from the distribution-dependence of $A$ and $B$. With this construction, the first-order perturbations in $A$ and $B$ cancel at the level of the mean when we plug in estimators, leaving only second-order terms:
    \begin{align*}
        R_{n,jl}
        =\PP\{\hat\varphi_{jl}-\varphi_{jl}\}
        &=
        \PP\big[(\hat v_j-v_j)(\hat w_{jl}-w_{jl})\big]
        \;+\;
        r^{(A)}_{n,j}
        \;+\;
        r^{(B)}_{n,jl},
    \end{align*}
    where $r^{(A)}_{n,j}$ is the second-order remainder coming from estimating the $A$-part
    ($\psi_A(\PP)=\phi_{Z_j}(\PP)$), and $r^{(B)}_{n,jl}$ is the second-order remainder coming from estimating the $B$-part ($\psi_B(\PP)=\theta_{jl}(\PP)$). The term $\PP[(\hat v_j-v_j)(\hat w_{jl}-w_{jl})]$ is the bilinear remainder from expanding $AB$.

    By Cauchy--Schwarz,
    \[
    \Big|\PP\big[(\hat v_j-v_j)(\hat w_{jl}-w_{jl})\big]\Big|
    \le
    \|\hat v_j-v_j\|_{L_2(\PP_{Z_{-j}})}\,\|\hat w_{jl}-w_{jl}\|_{L_2(\PP_{Z_{-j}})}
    =\Op(\cE_X^2),
    \]
    using the bounds from Step 2. Moreover, the $A$-remainder satisfies $r^{(A)}_{n,j}=\Op(\|\hat\eta-\eta\|_{L_2(\PP_Z)}^2)=\Op(\cE_X^2)$ by the explicit second-order remainder analysis in \Cref{thm:error-decom-eot}, and the $B$-remainder satisfies $r^{(B)}_{n,jl}=\Op(\|\hat w_{jl}-w_{jl}\|_{L_2(\PP_{Z_{-j}})}^2)=\Op(\cE_X^2)$ by \Cref{asm:eif-theta}~\ref{asm:eif-theta-ii}.
    Therefore $R_{n,jl}=\Op(\cE_X^2)$ and $\sum_{j=1}^d R_{n,jl}=\Op(\cE_X^2)$.

    \paragraph{Step 5: conclusion.}
    Finally, summing \eqref{eq:vm-psi} over $j=1,\dots,d$ yields
    \[
    \hat\phi_{X_l}-\phi_{X_l}(\PP)
    =
    (\PP_n-\PP)\,\varphi_{X_l}(O;\PP)
    +\Op\!\big(n^{-1/2}\cE_X+\cE_X^2\big),
    \]
    uniformly over $l$ by the uniform moment bounds in Step 1 and \Cref{asm:bary-smooth}. The asymptotic normality under $\cE_X=o(n^{-1/4})$ follows by Slutsky and CLT for $(\PP_n-\PP)\varphi_{X_l}$.
\end{proof}

\subsection[Proof of Lemma \ref{lem:latent-importance-decomp}]{Proof of \Cref{lem:latent-importance-decomp}}\label{app:subsec:decomp}
\begin{proof}
We split the proof into two parts.

For Part (1), we begin with the definition of the latent feature importance, $\phi_{Z_j} = \EE[\VV(\eta(Z) \mid Z_{-j})]$. Substituting the functional ANOVA decomposition of $\eta(Z)$ yields:
$$ \phi_{Z_j} = \EE\left[\VV\left(\sum_{\cS \subseteq [d]} \eta_\cS(Z_\cS) \;\bigg|\; Z_{-j}\right)\right] $$
Due to the independence of the latent features $Z_k$ and the orthogonality of the ANOVA terms ($\EE[\eta_\cS\eta_{\cS'}] = 0$ for $\cS \neq \cS'$), any term $\eta_\cS(Z_\cS)$ where $j \notin \cS$ is a constant when conditioning on $Z_{-j}$. All terms for which $j \in \cS$ remain as random variables that are functions of $Z_j$, which is independent of the conditioning set $Z_{-j}$. The variance of the sum is therefore the sum of the variances of terms involving $Z_j$:
$$ \VV\left(\sum_{\cS \subseteq [d]} \eta_\cS(Z_\cS) \;\bigg|\; Z_{-j}\right) = \sum_{\cS: j \in \cS} \VV[\eta_\cS(Z_\cS)] $$
Since the right-hand side is a constant, taking the outer expectation is trivial, so $\phi_{Z_j} = \sum_{\cS: j \in \cS} \VV[\eta_\cS(Z_\cS)]$. Summing over all $j$:
$$ \sum_{j=1}^d \phi_{Z_j} = \sum_{j=1}^d \sum_{\cS: j \in \cS} \VV[\eta_\cS(Z_\cS)] $$
By reversing the order of summation, we count how many times each term $\VV[\eta_\cS(Z_\cS)]$ appears. A term $\eta_\cS$ is included in the sum for $\phi_{Z_j}$ for every $j \in \cS$. Thus, the term $\VV[\eta_\cS(Z_\cS)]$ appears exactly $|\cS|$ times in the total sum, which proves the first statement.

For Part (2), the law of total variance for orthogonal ANOVA terms states that the total predictive variability is the sum of the variances of each term:
$$ \VV[\EE[Y \mid X]] = \VV[\eta(Z)] = \sum_{\emptyset \neq \cS \subseteq [d]} \VV[\eta_\cS(Z_\cS)] $$
Comparing this with the result from Part (1), the equality $\sum_j \phi_{Z_j} = \VV[\eta(Z)]$ holds if and only if:
$$ \sum_{\emptyset \neq \cS \subseteq [d]} |\cS| \cdot \VV[\eta_\cS(Z_\cS)] = \sum_{\emptyset \neq \cS \subseteq [d]} \VV[\eta_\cS(Z_\cS)] $$
This equality is satisfied if and only if $\VV[\eta_\cS(Z_\cS)] = 0$ for all sets $\cS$ where $|\cS| > 1$. This is the definition of a purely additive function, $\eta(Z) = \eta_\emptyset + \sum_{j=1}^d \eta_{\{j\}}(Z_j)$, which we write as $c + \sum_j g_j(Z_j)$.
By orthogonality of the ANOVA components in $L_2(\PP_Z)$, the predictive variability in the disentangled space satisfies
\(
\VV\!\big(\eta(Z)\big)
=
\sum_{\emptyset \neq \cS \subseteq [d]} \VV\!\big(\eta_\cS(Z_\cS)\big).
\)
\end{proof}

\subsection[Proof of Proposition \ref{prop:decomp-phi}]{Proof of \Cref{prop:decomp-phi}}\label{app:sub-prop:decomp-phi}
\begin{proof}
    Let \(B=\Sigma^{1/2}\), so that \(X=BZ\). Then
    \[
    \partial_{z_j}S_l(z)=B_{lj}.
    \]
    Hence
    \[
    \phi_{X_l}(\mathbb P)
    =
    \sum_{j=1}^d B_{lj}^2 \phi_{Z_j}(\mathbb P).
    \]
    Summing over \(l\) gives
    \[
    \sum_{l=1}^d\phi_{X_l}(\mathbb P)
    =
    \sum_{j=1}^d
    \left(\sum_{l=1}^d B_{lj}^2\right)\phi_{Z_j}(\mathbb P)
    =
    \sum_{j=1}^d
    (B^\top B)_{jj}\phi_{Z_j}(\mathbb P)
    =
    \sum_{j=1}^d
    \Sigma_{jj}\phi_{Z_j}(\mathbb P),
    \]
    which completes the proof.
\end{proof}

\clearpage
\section{Supporting lemmas}\label{app:lemmas}

\subsection{Basic properties of regression estimators}\label{app:sub:basic}

\begin{lemma}[$L_p$ boundedness]\label[lemma]{lem:L2-boundedness}
Let $(Y, Z)$ be a pair of random variables with $Z = (Z_1, \dots, Z_d)$ and $Z_{-j}$ denoting all components except $Z_j$. Define the conditional expectations
\[
\eta(Z) := \EE[Y \mid Z], \qquad \eta_{-j}(Z_{-j}) := \EE[Y \mid Z_{-j}].
\]
If $\|Y\|_{L_p} \leq C$ for some $p \geq 2$, then it holds that
\begin{enumerate}[(1)]
    \item $\|\eta_{-j}(Z_{-j})\|_{L_p} \leq \|\eta(Z)\|_{L_p} \leq C$;
    \item $\|Y - \eta(Z)\|_{L_p} \leq 2C $ and $ \|\eta(Z) - \eta_{-j}(Z_{-j})\|_{L_p} \leq 2C$;
    \item $\|\VV(\eta(Z)\mid Z_{-j}) \|_{L_{p/2}} \leq C^2$.
\end{enumerate}
\end{lemma}
\begin{proof}
    For (1), by Jensen's inequality,
    \[
    \|\eta(Z)\|_{L_p} \leq \|\eta_{-j}(Z_{-j})\|_{L_p} \leq \|Y\|_{L_p} \leq C.
    \]

    For (2), by Minkowski's inequality,
    \[
    \|Y-\eta(Z)\|_{L_p}
    \le \|Y\|_{L_p} + \|\eta(Z)\|_{L_p}
    \le C + C
    = 2C.
    \]
    Similarly,
    \[
    \|\eta(Z)-\eta_{-j}(Z_{-j})\|_{L_p}
    \le \|\eta(Z)\|_{L_p} + \|\eta_{-j}(Z_{-j})\|_{L_p}
    \le C + C
    = 2C.
    \]

    For (3), because $\VV(\eta(Z)\mid Z_{-j})\leq\EE[\eta(Z)^2\mid Z_{-j}]$,
    by the definition of conditional variance and Jensen's inequality,
    Jensen,
    \[
    \|\VV(\eta(Z)\mid Z_{-j}) \|_{L_{p/2}}^{p/2}
    \le \EE[\|\EE[\eta(Z)^2\mid Z_{-j}]\|_{L_{p/2}}^{p/2}]
    \le \EE[\EE[\eta(Z)^p\mid Z_{-j}]]
    = \|\eta(Z)\|_{L_p}^p
    \le C^p.
    \]
    This completes the proof.
\end{proof}

\begin{lemma}[Hilbert--Schmidt bound for an integral operator]\label[lemma]{lem:HS-bound}
Let $\PP_X$ and $\PP_Z$ be probability measures, and let
$k \in L_2(\PP_X\otimes \PP_Z)$ and $f\in L_2(\PP_X)$.
Define
\[
(T_k f)(z) \;:=\; \int f(x)\,k(z\mid x)\rd \PP_X(x),
\qquad z\in \cZ.
\]
Then $T_k f\in L_2(\PP_Z)$ and
\[
\|T_k f\|_{L_2(\PP_Z)}
\le
\|k\|_{L_2(\PP_X\otimes \PP_Z)}\,\|f\|_{L_2(\PP_X)}.
\]
The map $T_k:L_2(\PP_X)\to L_2(\PP_Z)$ is a Hilbert--Schmidt integral operator.
The inequality above is exactly the Hilbert--Schmidt norm bound $\|T_k f\|_2\le \|T_k\|_{\mathrm{HS}}\|f\|_2$, with $\|T_k\|_{\mathrm{HS}}=\|k\|_{L_2(\PP_X\otimes \PP_Z)}$.
\end{lemma}

\begin{proof}
    Fix $z$. By Cauchy--Schwarz in $L_2(\PP_X)$,
    \[
    \big|(T_k f)(z)\big|
    =
    \left|\int f(x)\,k(z\mid x)\rd \PP_X(x)\right|
    \le
    \left(\int f(x)^2\rd \PP_X(x)\right)^{1/2}
    \left(\int k(z\mid x)^2\rd \PP_X(x)\right)^{1/2}.
    \]
    Squaring and integrating over $z\sim \PP_Z$ yields
    \begin{align*}
    \int (T_k f)(z)^2\rd \PP_Z(z)
    &\le
    \left(\int f(x)^2\rd \PP_X(x)\right)
    \left(\int \left[\int k(z\mid x)^2\rd \PP_X(x)\right] d\PP_Z(z)\right)\\
    &=
    \|f\|_{L_2(\PP_X)}^2\;
    \int\!\!\int k(z\mid x)^2\rd \PP_X(x)\rd \PP_Z(z)\\
    &=
    \|f\|_{L_2(\PP_X)}^2\;\|k\|_{L_2(\PP_X\otimes \PP_Z)}^2,
    \end{align*}
    where the middle equality uses Tonelli/Fubini (the integrand is nonnegative and
    $k\in L_2(\PP_X\otimes \PP_Z)$).
    Taking square-roots proves the claim.
\end{proof}

\subsection{LOCO estimation}\label{app:sub:loco}

\begin{lemma}[Estimated influence function for LOCO]\label[lemma]{lem:decomp-varphi-diff}
    Assume that $\hat{\eta}$ and $\eta$ are bounded in $L_{\infty}$ with probability tending to one.
    Under \Cref{asm:data}, the influence function estimation error for the LOCO estimator \eqref{eq:phi-hat-loco} satisfies that
    \begin{align*}
         \|\varphi_{Z_j}(O;\hat{\PP})-\varphi_{Z_j}(O;\PP) \|_{L_2}
        &\;=\;
      \Op(\|\hat\eta-\eta\|_{L_2}
                 +\|\hat\eta_{-j}-\eta_{-j}\|_{L_2}
        + \|\hat\eta-\eta\|_{L_2}^2+\|\hat\eta_{-j}-\eta_{-j}\|_{L_2}^2),
    \end{align*}    
\end{lemma}
\begin{proof}
    For the proof below, we drop the subscript $Z_j$ and the dependency in $(Z,Z_{-j})$ for simplicity.
    Recall the efficient influence function
    \[
    \varphi(O;\PP)
    =\;2\,(Y-\eta)(\eta-\eta_{-j})
      +(\eta-\eta_{-j})^2-\phi,
    \qquad
    \eta:=\EE[Y\mid Z],\;
    \eta_{-j}:=\EE[Y\mid Z_{-j}].
    \]
    For the estimated distribution \(\hat{\PP}\) we plug in
    \(\hat\eta,\hat\eta_{-j}\):
    \[
    \varphi(O;\hat{\PP})
    =\;2\,(Y-\hat\eta)(\hat\eta-\hat\eta_{-j})
      +(\hat\eta-\hat\eta_{-j})^2-\hat{\phi}_{Z_j}.
    \]
    Denote $\delta:=\hat\eta-\eta$, $\delta_{-j}:=\hat\eta_{-j}-\eta_{-j}$ and $V:=\delta-\delta_{-j}$.
    Observe
    \(\hat\eta-\hat\eta_{-j}=(\eta-\eta_{-j})+V\).
    Then, we have
    \[
    \varphi(O;\hat{\PP})-\varphi(O;\PP)
    =2[(Y-\hat\eta)(\hat\eta-\hat\eta_{-j})
            -(Y-\eta)(\eta-\eta_{-j}) ]
     +[(\hat\eta-\hat\eta_{-j})^2-(\eta-\eta_{-j})^2 ] + (\phi -\hat{\phi}).
    \]
    Insert
    \(Y-\hat\eta=(Y-\eta)-\delta\)
    and
    \(\hat\eta-\hat\eta_{-j}=(\eta-\eta_{-j})+V\), it follows that
    \[
    \varphi(O;\hat{\PP})-\varphi(O;\PP)
       = 2\,(Y-\eta)\,V
         -2\delta(\eta-\eta_{-j})
         -2\delta V
         +2(\eta-\eta_{-j})V
         +V^{2} + (\phi -\hat{\phi}).
    \]
    Because
    \((\eta-\eta_{-j})V-\delta(\eta-\eta_{-j})
       =-(\eta-\eta_{-j})\delta_{-j}\), we further have
    \[
    \varphi(O;\hat{\PP})-\varphi(O;\PP)
      = 2\,(Y-\eta)\,V
        \;-\;2(\eta-\eta_{-j})\,\delta_{-j}
        \;-\;2\delta V
        \;+\;V^{2}+ (\phi -\hat{\phi}).
    \]
    
    Under the boundedness assumption $\EE[Y^2] \leq C$ with \Cref{lem:L2-boundedness} and using Cauchy-Schwarz:
    \begin{align*}
        \|\varphi(O;\hat{\PP})-\varphi(O;\PP) \|_{L_2}
      &\leq
      2\|Y-\eta\|_{L_2}\|V\|_{L_2}
      \;+\;
      2\|\eta-\eta_{-j}\|_{L_2}\|\delta_{-j}\|_{L_2} +
     2 \|\delta\|_{L_2}\|V\|_{L_2}
      \\
      &\qquad \;+\; \|V\|_{L_2}^2 + \|\phi -\hat{\phi}\|_{L_2}
      \\
      &\leq 4C\|V\|_{L_2} + 2C\|\delta_{-j}\|_{L_2} + \|V\|_{L_2}^2 + (\phi -\hat{\phi}).
    \end{align*}
    Because $\hat{\phi}$ is uniformly bounded in probability and $\hat{\phi}-\phi = \op(1)$, by dominate convergence theorem, we have that $\|\phi -\hat{\phi}\|_{L_2} =\op(1)$.
    Since \(\|V\|_{L_2}\le\|\delta\|_{L_2}+\|\delta_{-j}\|_{L_2}\), the right-hand side is
    \[
        \Op(\|\hat\eta-\eta\|_{L_2}
                 +\|\hat\eta_{-j}-\eta_{-j}\|_{L_2}
        + \|\hat\eta-\eta\|_{L_2}^2+\|\hat\eta_{-j}-\eta_{-j}\|_{L_2}^2),
    \]
    which completes the proof.
\end{proof}

\begin{lemma}[$L_2$-bound for LOO estimator]\label[lemma]{lem:eta-j}
    Let $\hat{\eta}$ be an estimator of $\eta$, estimated from independent samples.
    Suppose $Z_j\indep Z_{-j}$ and define the leave-$j$-out estimator
    \[
    \hat\eta_{-j}(z_{-j})
    \;=\;
    \EE[\hat\eta(Z)\mid \hat{\eta}, Z_{-j}=z_{-j}] = \EE[\hat\eta(Z^{(j)})\mid \hat{\eta},Z_{-j}=z_{-j}],
    \]
    for $\eta_{-j}(z_{-j})=\EE[\eta(Z)\mid Z_j=z_j]$.  
    Then
    \[
    \|\hat\eta_{-j}(Z_{-j})-\eta_{-j}(Z_{-j}) \|_{L_2(P_{Z_{-j}})}
    \le
    \|\hat\eta(Z)-\eta(Z) \|_{L_2(\PP_Z)}.
    \]
\end{lemma}
\begin{proof}
    Since $Z_j\indep Z_{-j}$, the joint law of $(Z_{-j},Z_j)$ (the same holds for $(Z_{-j},Z_j^{(j)})$ ) factors and by Jensen's inequality
    \[
    \begin{aligned}
        \EE_{Z_{-j}}[(\hat\eta_{-j}(Z_{-j})-\eta_{-j}(Z_{-j}))^2 ]
        &=\EE_{Z_{-j}}[(\int[\hat\eta(Z_{-j},z_j)-\eta(Z_{-j},z_j) ]\rd F_{Z_j}(z_j))^{2} ]\\
        &\le 
        \EE_{Z_{-j}}[\int[\hat\eta(Z_{-j},z_j)-\eta(Z_{-j},z_j) ]^{2}\rd F_{Z_j}(z_j) ]\\
        &=
        \EE_{Z}[(\hat\eta(Z)-\eta(Z))^{2} ].
        \end{aligned}
    \]
    Taking square roots yields the result.
\end{proof}

\subsection{Disentangled feature importance}\label{app:sub:DVI}

\begin{lemma}[EIF of the expected product]\label[lemma]{lem:EIF-product}
    Let the data be \(O\sim P \) and suppose that the (possibly distribution-dependent) functions \(A\) and \(B\) satisfy
    \[
    \psi_A( P ) \;=\; \EE_{ P }[A(O) ], 
    \qquad
    \psi_B( P ) \;=\; \EE_{ P }[B(O) ].
    \]
    Denote by \(\varphi_A(O; P )\) and \(\varphi_B(O; P )\) their efficient
    influence functions (EIFs), normalised so that
    \[
    \EE_{ P } [\varphi_A(O; P ) ]=
    \EE_{ P } [\varphi_B(O; P ) ]=0.
    \]
    Define the centered functions
    \(A_c(O)=A(O)-\psi_A( P )\) and \(B_c(O)=B(O)-\psi_B( P )\),
    and the remainder term
    \[
        \alpha_A(O) := \varphi_A(O; P ) - A_c(O) , 
        \qquad 
        \alpha_B(O) := \varphi_B(O; P ) - B_c(O) ,
    \]
    where the remainder terms \(\alpha_A,\alpha_B\) collect the pathwise-derivative contributions coming from the dependence of \(A\) and \(B\) on \( P \).
    Then, the EIF of $\psi_{AB}( P ):=\EE_{ P }[A(O)B(O)]$ is given by:
    \begin{align}
        \varphi_{AB}(O; P )
      = A(O)B(O)-\psi_{AB}( P )
        +B(O)\alpha_A(O)
        +A(O)\alpha_B(O). \label{eq:EIF-prod}
    \end{align}
\end{lemma}
\begin{proof}
    Let \( P _\varepsilon=(1+\varepsilon h) P \) be a regular, mean‑zero submodel with score \(h\).
    Differentiating \(\psi_{AB}( P _\varepsilon)=\int A_{ P _\varepsilon}B_{ P _\varepsilon}\rd  P _\varepsilon\) at \(\varepsilon=0\) gives
    \begin{align}
        \frac{\rd}{\rd\varepsilon}\psi_{AB}( P _\varepsilon)\Big|_{\varepsilon=0}
        = \EE_{ P } [h(O)\,A(O)B(O) ]
      + \EE_{ P } [B(O)\,\dot{A}_h(O) ]
      + \EE_{ P } [A(O)\,\dot{B}_h(O) ], \label{eq:EIF-prod-1}
    \end{align}
    where \(\dot{A}_h,\dot{B}_h\) are the pathwise derivatives of
    \(A_{ P _\varepsilon}(O)\) and \(B_{ P _\varepsilon}(O)\) along \(h\).

    The defining property of \(\varphi_A\) and \(\varphi_B\) is
    \begin{align}
        \EE_{ P } [h(O)\,\varphi_A(O; P ) ]
      \;=\;
      \EE_{ P } [h(O)\,A(O) ] 
      + \EE_{ P } [\dot{A}_h(O) ], \label{eq:EIF-prod-2}
    \end{align}
    and similarly for $B$.
    This implies
    \begin{align}
        \EE_{ P } [\dot{A}_h(O) ] 
      \;=\;
      \EE_{ P } [h(O)\,\alpha_A(O) ],
    \qquad
    \EE_{ P } [\dot{B}_h(O) ] 
      \;=\;
      \EE_{ P } [h(O)\,\alpha_B(O) ].    \label{eq:EIF-prod-3} 
    \end{align}

    Define
    \begin{align}
        \varphi_{AB}(O; P )
        \;:=\;
        A(O)B(O) - \psi_{AB}( P )
        \;+\;
        B(O)\,\alpha_A(O)
        \;+\;
        A(O)\,\alpha_B(O).\label{eq:EIF-prod-4}
    \end{align}
    We verify that \eqref{eq:EIF-prod-4} satisfies the EIF equation.  Using \eqref{eq:EIF-prod-1}-\eqref{eq:EIF-prod-3}, we have
    \[
    \EE_{ P } [h(O)\,\varphi_{AB}(O; P ) ]
    \;=\;
    \frac{\rd}{\rd\varepsilon}\psi_{AB}( P _\varepsilon)\Big|_{\varepsilon=0},
    \]
    and \(\EE_{ P }[\varphi_{AB}]=0\) by construction, so \eqref{eq:EIF-prod-4} is indeed the
    efficient influence function of \(\psi_{AB}\).
    This completes the proof.
\end{proof}

\begin{remark}[Special cases of \Cref{lem:EIF-product}]
    When $A$ is distribution-independent (i.e.\ $A_P\equiv A$), one has \(\alpha_A=0\) and \eqref{eq:EIF-prod} becomes \(\varphi_{AB}=A\,\varphi_B\).     
     When neither \(A\) nor \(B\) depends on \( P \),  one has \(\alpha_A=\alpha_B=0\), \(\varphi_A=A_c\), \(\varphi_B=B_c\), and \eqref{eq:EIF-prod} reduces to \(\varphi_{AB}=A B-\EE[AB]\).
     This can also happen if \(A\) and \(B\) are linear functions.
     For instance, for the covariance parameter $\psi_{AB}( P ) = \EE_{ P }[(A-\EE_{ P }[A])(B-\EE_{ P }[B])]$, one has $X(O) = X $ and $\varphi_X(O; P ) = X_c(O) = X - \EE_{ P }[X]$ for $X\in\{A,B\}$.
      Hence, \(\alpha_A=\alpha_B=0\) and the last two terms in \eqref{eq:EIF-prod} vanish and \eqref{eq:EIF-prod} reduces to $\varphi_{AB}(O; P ) = (A-\EE_{ P }[A])(B-\EE_{ P }[B]) - \psi_{AB}( P )$.
\end{remark}

\begin{lemma}[EIF plug-in error bound under EOT]\label[lemma]{lem:eif-plug-in-decomp}
    Fix $(j,l)\in[d]^2$ and assume \Cref{asm:data,asm:bary-smooth} hold.
    Assume further that the EIF in \Cref{prop:eif-eot-weighted} exists and satisfies
    \(
    \sup_{j,l\in[d]}\;\|\varphi_{\theta_{jl}}(O;\PP)\|_{L_4(\PP)}<\infty.
    \)
    Let $(\hat\mu,\hat k,\partial_z\hat k)$ be nuisance estimators trained on an auxiliary sample independent of the $n$ observations used in $\PP_n$, so that $\|\hat\mu\|_{L_\infty}\vee\|\hat\mu\hat k\|_{L_\infty}\vee\|\partial_z\hat k\|_{L_\infty}\le M$ for some fixed $M<\infty$.
    Then, uniformly over $j\in[d]$, the following bounds hold with probability tending to $1$:
    \begin{align*}
    \|\hat\eta-\eta\|_{L_2(\PP_Z)}
    &\lesssim
    \|\hat\mu-\mu\|_{L_2(\PP_X)}
    +
    \|\hat k-k\|_{L_2(\PP_X\otimes\PP_Z)}
    \lesssim\varepsilon_{jl},\\
    \|\hat v_j-v_j\|_{L_2(\PP_{Z_{-j}})}
    &\lesssim
    \|\hat\eta-\eta\|_{L_4(\PP_Z)}
    \lesssim\varepsilon_{jl},\\
    \|\hat s_{jl}-s_{jl}\|_{L_2(\PP_Z)}
    &\le
    \|X_l\|_{L_2(\PP_X)}\;\|\partial_{z_j}\hat k-\partial_{z_j}k\|_{L_2(\PP_X\otimes\PP_Z)}
    \lesssim\varepsilon_{jl},\\
    \|\hat w_{jl}-w_{jl}\|_{L_2(\PP_{Z_{-j}})}
    &\lesssim\|\hat s_{jl}-s_{jl}\|_{L_2(\PP_Z)}
    \lesssim\varepsilon_{jl},\\
    |\hat\theta_{jl}-\theta_{jl}|
    &\lesssim n^{-1/2}+\varepsilon_{jl}^2.
    \end{align*}
    where \[
    \varepsilon_{jl}
    := \|\hat\mu-\mu\|_{L_2(\PP_X)}
    + \|\hat k-k\|_{L_2(\PP_X\otimes\PP_Z)}
    + \|\partial_{z_j}\hat k-\partial_{z_j}k\|_{L_2(\PP_X\otimes\PP_Z)}.
    \]
	    Let $\hat\varphi_{jl}$ denote the plug-in version of $\varphi_{jl}$ obtained by replacing
	    \[
	    (\eta,\Delta_j,v_j,w_{jl},\theta_{jl})
	    \quad\text{with}\quad
	    (\hat\eta,\hat\Delta_j,\hat v_j,\hat w_{jl},\hat\theta_{jl}),
	    \]
	    where $\varphi_{\theta_{jl}}(O;\PP)$ is known.
	    Then
    \[
    \|\hat\varphi_{jl}-\varphi_{jl}\|_{L_2(\PP)}
    \lesssim
    \|\hat\eta-\eta\|_{L_2(\PP_Z)}
    +
    \|\hat v_j-v_j\|_{L_2(\PP_{Z_{-j}})}
    +
    \|\hat w_{jl}-w_{jl}\|_{L_2(\PP_{Z_{-j}})}
    +
    |\hat\theta_{jl}-\theta_{jl}|
    \lesssim n^{-1/2}+\varepsilon_{jl}.
    \]
\end{lemma}

\begin{proof}
    Throughout, condition on the auxiliary sample used to train $(\hat\mu,\hat k,\partial_z\hat k)$, so that the nuisance estimators may be treated as fixed.
    Work on the event (with probability tending to $ 1$) on which
    $\|\hat\mu\|_{L_\infty}\vee\|\hat\mu\hat k\|_{L_\infty}\vee\|\partial_z\hat k\|_{L_\infty}\le M$ and \Cref{asm:data,asm:bary-smooth} hold. All implicit constants below may depend on
    $M$ and the fixed moment/smoothness envelopes in \Cref{asm:data,asm:bary-smooth}.

    \paragraph{Step 1: Bound $\|\hat\eta-\eta\|_{L_2(\PP_Z)}$.}
    From \Cref{eq:delta-final} and \Cref{lem:HS-bound}, we have
    \[
    \|\hat\eta_{-j}-\eta_{-j}\|_{L_2(\PP_{Z_{-j}})}
    \le
    \|\hat\eta-\eta\|_{L_2(\PP_Z)}\lesssim
    \|\hat\mu-\mu\|_{L_2(\PP_X)}+\|\hat k-k\|_{L_2(\PP_X\otimes\PP_Z)}
    \lesssim\varepsilon_{jl}.
    \]

    \paragraph{Step 2: Bound $\|\hat v_j-v_j\|_{L_2(\PP_{Z_{-j}})}$.}
    Write $\Delta_j(Z)=\eta(Z)-\eta_{-j}(Z_{-j})$ and
    $\hat\Delta_j(Z)=\hat\eta(Z)-\hat\eta_{-j}(Z_{-j})$, so that
    \[
    v_j(z_{-j})=\EE[\Delta_j(Z)^2\mid Z_{-j}=z_{-j}],
    \qquad
    \hat v_j(z_{-j})=\EE[\hat\Delta_j(Z)^2\mid Z_{-j}=z_{-j}].
    \]
    By Jensen,
    \[
    \|\hat v_j-v_j\|_{L_2(\PP_{Z_{-j}})}
    \le
    \|\hat\Delta_j^2-\Delta_j^2\|_{L_2(\PP_Z)}.
    \]
    Using $\hat\Delta_j^2-\Delta_j^2=(\hat\Delta_j-\Delta_j)(\hat\Delta_j+\Delta_j)$ and H\"older,
    \[
    \|\hat\Delta_j^2-\Delta_j^2\|_{L_2(\PP_Z)}
    \le
    \|\hat\Delta_j-\Delta_j\|_{L_4(\PP_Z)}\cdot \|\hat\Delta_j+\Delta_j\|_{L_4(\PP_Z)}.
    \]
    Under \Cref{asm:data}, \Cref{lem:L2-boundedness} implies $\Delta_j\in L_4(\PP_Z)$ uniformly in $j$.
    Also, $\hat\eta$ is uniformly bounded as assumed, hence
    $\hat\Delta_j\in L_\infty(\PP_Z)\subset L_4(\PP_Z)$ uniformly in $j$.
    Therefore $\|\hat\Delta_j+\Delta_j\|_{L_4(\PP_Z)}\lesssim 1$ and
    \[
    \|\hat v_j-v_j\|_{L_2(\PP_{Z_{-j}})}
    \lesssim
    \|\hat\Delta_j-\Delta_j\|_{L_4(\PP_Z)}.
    \]
    Finally,
    \[
    \hat\Delta_j-\Delta_j
    =
    (\hat\eta-\eta)-(\hat\eta_{-j}-\eta_{-j}),
    \]
    and by Jensen's contraction of conditional expectation,
    $\|\hat\eta_{-j}-\eta_{-j}\|_{L_4(\PP_{Z_{-j}})}\le \|\hat\eta-\eta\|_{L_4(\PP_Z)}$.
    Thus $\|\hat\Delta_j-\Delta_j\|_{L_4(\PP_Z)}\lesssim \|\hat\eta-\eta\|_{L_4(\PP_Z)}$.
    Combining, and noting $\|\hat\eta-\eta\|_{L_2}\le \|\hat\eta-\eta\|_{L_4}$ on a probability space, \(\|\hat v_j-v_j\|_{L_2(\PP_{Z_{-j}})} \lesssim \|\hat\eta-\eta\|_{L_4(\PP_Z)}\) and in particular
    \[
    \|\hat v_j-v_j\|_{L_2(\PP_{Z_{-j}})}
    \lesssim
    \|\hat\eta-\eta\|_{L_2(\PP_Z)}
    \lesssim \varepsilon_{jl}.
    \]

    \paragraph{Step 3: Bound $\|\hat s_{jl}-s_{jl}\|_{L_2(\PP_Z)}$.}
    By \Cref{asm:bary-smooth}(ii), $s_{jl}$ admits the integral representation
    \[
    s_{jl}(z)=\int x_l\,\partial_{z_j}k(x\mid z)\,\PP_X(\rd x),
    \qquad
    \hat s_{jl}(z)=\int x_l\,\partial_{z_j}\hat k(x\mid z)\,\PP_X(\rd x),
    \quad \PP_Z\text{-a.e. }z.
    \]
    Therefore, by Cauchy--Schwarz and \Cref{asm:data}(i),
    \[
    \|\hat s_{jl}-s_{jl}\|_{L_2(\PP_Z)}
    \le
    \|X_l\|_{L_2(\PP_X)}\,
    \|\partial_{z_j}\hat k-\partial_{z_j}k\|_{L_2(\PP_X\otimes \PP_Z)}
    \lesssim \varepsilon_{jl}.
    \]

    \paragraph{Step 4: Bound $\|\hat w_{jl}-w_{jl}\|_{L_2(\PP_{Z_{-j}})}$ and $|\hat\theta_{jl}-\theta_{jl}|$.}
    Recall
    \[
    w_{jl}(z_{-j})=\EE[s_{jl}(Z)^2\mid Z_{-j}=z_{-j}],
    \qquad
    \hat w_{jl}(z_{-j})=\EE[\hat s_{jl}(Z)^2\mid Z_{-j}=z_{-j}].
    \]
    By Jensen,
    \[
    \|\hat w_{jl}-w_{jl}\|_{L_2(\PP_{Z_{-j}})}
    \le
    \|\hat s_{jl}^2-s_{jl}^2\|_{L_2(\PP_Z)}
    \le
    \|\hat s_{jl}-s_{jl}\|_{L_2(\PP_Z)}\cdot \|\hat s_{jl}+s_{jl}\|_{L_\infty(\PP_Z)}.
    \]
    By \Cref{asm:bary-smooth}~\ref{asm:bary-smooth:iii}-\ref{asm:bary-smooth:estimator}, $s_{jl}$ and
    $\hat s_{jl}$ are uniformly bounded in $L_\infty(\PP_Z)$; hence $\|\hat s_{jl}+s_{jl}\|_{L_\infty(\PP_Z)}\lesssim 1$.
    Thus
    \[
    \|\hat w_{jl}-w_{jl}\|_{L_2(\PP_{Z_{-j}})}
    \lesssim
    \|\hat s_{jl}-s_{jl}\|_{L_2(\PP_Z)}
    \lesssim \varepsilon_{jl}.
    \]
    For $\theta_{jl}(\PP)=\EE[w_{jl}(Z_{-j})]$, \Cref{asm:eif-theta} implies that the estimator $\hat\theta_{jl}$ admits the uniform linear expansion
    \[
    \hat\theta_{jl}-\theta_{jl}(\PP)
    =
    (\PP_n-\PP)\,\varphi_{\theta_{jl}}(O;\PP)
    +
    R_{\theta,jl},
    \]
    with $\sup_{j,l}\|\varphi_{\theta_{jl}}(O;\PP)\|_{L_4(\PP)}<\infty$ and
    \[
    R_{\theta,jl}
    =
    \Op\!\left(\|\hat w_{jl}-w_{jl}\|_{L_2(\PP_{Z_{-j}})}^2\right).
    \]
    Therefore,
    \[
    |\hat\theta_{jl}-\theta_{jl}(\PP)|
    \le
    |(\PP_n-\PP)\varphi_{\theta_{jl}}(O;\PP)|
    +
    |R_{\theta,jl}|
    =
    \Op\!\left(n^{-1/2}+\|\hat w_{jl}-w_{jl}\|_{L_2(\PP_{Z_{-j}})}^2\right).
    \]
    In particular, using Step~4's bound $\|\hat w_{jl}-w_{jl}\|_{L_2(\PP_{Z_{-j}})}\lesssim \varepsilon_{jl}$ yields
    \[
    |\hat\theta_{jl}-\theta_{jl}(\PP)|=\Op\!\left(n^{-1/2}+\varepsilon_{jl}^2\right).
    \]

    \paragraph{Step 5: Bound $\|\hat\varphi_{jl}-\varphi_{jl}\|_{L_2(\PP)}$.}
    We begin by obtaining an exact algebraic decomposition of $\varphi_{jl}(O;\PP)$ in \eqref{eq:eif-phi-jl} that will be used throughout:
    \[
    \varphi_{jl}(O;\PP)
    =
    w_{jl}(Z_{-j})\,R_j(O;\eta)
    +
    v_j(Z_{-j})\,U_{jl}(O;\PP)
    -\psi_{jl}(\PP),
    \]
    where
    $R_j(O;\eta):=2\{Y-\eta(Z)\}\Delta_j(Z)+\Delta_j(Z)^2$ and
    \[
    U_{jl}(O;\PP):=\varphi_{\theta_{jl}}(O;\PP)-\{w_{jl}(Z_{-j})-\theta_{jl}(\PP)\},
    \qquad
    \psi_{jl}(\PP):=\EE[v_j(Z_{-j})w_{jl}(Z_{-j})].
    \]
    Define $\hat\varphi_{jl}$ by plugging in
    $(\hat\eta,\hat\Delta_j,\hat v_j,\hat w_{jl},\hat\theta_{jl})$ while keeping
    $\varphi_{\theta_{jl}}(O;\PP)$ fixed (as assumed known).
    Then, by triangle inequality,
    \[
    \|\hat\varphi_{jl}-\varphi_{jl}\|_{L_2(\PP)}
    \le T_1+T_2+T_3,
    \]
    where
    \begin{align*}
    T_1&:=\|\hat w_{jl}(Z_{-j})R_j(O;\hat\eta)-w_{jl}(Z_{-j})R_j(O;\eta)\|_{L_2(\PP)},\\
    T_2&:=\|\hat v_j(Z_{-j})\hat U_{jl}(O;\PP)-v_j(Z_{-j})U_{jl}(O;\PP)\|_{L_2(\PP)},\\
    T_3&:=|\hat\psi_{jl}(\PP)-\psi_{jl}(\PP)|,
    \qquad \hat\psi_{jl}(\PP):=\EE[\hat v_j(Z_{-j})\hat w_{jl}(Z_{-j})].
    \end{align*}

    \emph{Control of $T_1$.}
    Write
    \[
    \hat w_{jl}R_j(O;\hat\eta)-w_{jl}R_j(O;\eta)
    =
    \hat w_{jl}\{R_j(O;\hat\eta)-R_j(O;\eta)\} + (\hat w_{jl} - w_{jl}) R_j(O;\eta).
    \]
    Since $w_{jl}(Z_{-j})=\EE[s_{jl}(Z)^2\mid Z_{-j}]$ and $s_{jl}$ is bounded in $L_8(\PP_Z)$ by \Cref{asm:bary-smooth}(iii), we have $\|w_{jl}\|_{L_4(\PP_{Z_{-j}})}\lesssim 1$.
    Similarly, $\|\hat w_{jl}\|_{L_\infty(\PP_{Z_{-j}})}\lesssim 1$ under the assumed $L_\infty$ bounds on the nuisance estimators.
    Moreover, by \Cref{asm:data} and \Cref{lem:L2-boundedness}, $R_j(O;\eta)\in L_\infty(\PP)$ uniformly in $j$, and the same holds for $R_j(O;\hat\eta)$ on the high-probability event where $\hat\eta$ is bounded.
    By H\"older inequality:
    \begin{align*}
        \| R(\hat\eta) - R(\eta) \|_{L_2(\PP)} & \lesssim \|\hat\eta-\eta\|_{L_2(\PP_Z)} + \|\hat\eta_{-j}-\eta_{-j}\|_{L_2(\PP_{Z_{-j}})}
        \lesssim \|\hat\eta-\eta\|_{L_2(\PP_Z)}.
    \end{align*}
    Therefore,
    \[
    T_1
    \lesssim
    \|\hat w_{jl}-w_{jl}\|_{L_2(\PP_{Z_{-j}})}
    +
    \|\hat\eta-\eta\|_{L_2(\PP_Z)}.
    \]

    \emph{Control of $T_2$.} 
    From \(\hat U_{jl}-U_{jl} = -(\hat w_{jl}-w_{jl})+(\hat\theta_{jl}-\theta_{jl})\), we have
    \[
    \|\hat U_{jl}-U_{jl}\|_{L_2(\PP)}
    \le
    \|\hat w_{jl}-w_{jl}\|_{L_2(\PP_{Z_{-j}})}+|\hat\theta_{jl}-\theta_{jl}|.
    \]
    By Step 4 and \Cref{asm:eif-theta},
    \[
    \|\hat U_{jl}-U_{jl}\|_{L_2(\PP)}
    =
    \Op\!\left(\varepsilon_{jl}+n^{-1/2}+\varepsilon_{jl}^2\right).
    \]

    Moreover, $U_{jl}(O;\PP)\in L_2(\PP)$ uniformly in $(j,l)$: indeed, since
    \(U_{jl}=\varphi_{\theta_{jl}}-(w_{jl}-\theta_{jl})\),
    \[
    \|U_{jl}\|_{L_2(\PP)}
    \le
    \|\varphi_{\theta_{jl}}\|_{L_2(\PP)}+\|w_{jl}\|_{L_2(\PP_{Z_{-j}})}+|\theta_{jl}|
    \lesssim 1,
    \]
    using \Cref{asm:eif-theta}(i) and $w_{jl}\in L_4(\PP_{Z_{-j}})$.
    Similarly, on any event where $\|\hat w_{jl}\|_{L_2(\PP_{Z_{-j}})}$ is bounded (e.g. the same high-probability
    event used in Step~4) and $|\hat\theta_{jl}|$ is bounded, we have $\hat U_{jl}(O;\PP)\in L_2(\PP)$ uniformly.
    Finally, $|\hat\theta_{jl}| \le |\theta_{jl}|+|\hat\theta_{jl}-\theta_{jl}|= \Op(1)$ by the above expansion,
    so such a boundedness event has probability tending to one.
    Decompose
    \[
    \hat v_j\,\hat U_{jl} - v_j U_{jl}
    =
    (\hat v_j - v_j)\hat U_{jl} + v_j(\hat U_{jl} - U_{jl}).
    \]
    Consequently,
    \[
    T_2
    \lesssim
    \|\hat v_j-v_j\|_{L_2(\PP_{Z_{-j}})}
    +
    \|\hat w_{jl}-w_{jl}\|_{L_2(\PP_{Z_{-j}})}
    +
    |\hat\theta_{jl}-\theta_{jl}|.
    \]

    \emph{Control of $T_3$.}
    Using $\hat\psi_{jl}-\psi_{jl}=\EE[\hat v_j\hat w_{jl}-v_j w_{jl}]$ and Cauchy--Schwarz, together with the boundedness of $\hat w_{jl}$ and the uniform $L_2$ control of $v_j$ and $\hat v_j$ from \Cref{lem:L2-boundedness} and Step~2,
    \[
    T_3
    \lesssim
    \|\hat v_j-v_j\|_{L_2(\PP_{Z_{-j}})}
    +
    \|\hat w_{jl}-w_{jl}\|_{L_2(\PP_{Z_{-j}})}.
    \]

    Combining the bounds on $T_1,T_2,T_3$ yields
    \[
    \|\hat\varphi_{jl}-\varphi_{jl}\|_{L_2(\PP)}
    \lesssim
    \|\hat\eta-\eta\|_{L_2(\PP_Z)}
    +
    \|\hat v_j-v_j\|_{L_2(\PP_{Z_{-j}})}
    +
    \|\hat w_{jl}-w_{jl}\|_{L_2(\PP_{Z_{-j}})}
    +
    |\hat\theta_{jl}-\theta_{jl}|.
    \]
    Finally, Steps~1--4 show each term on the right-hand side is $\lesssim \varepsilon_{jl}$,
    uniformly over $j\in[d]$, which completes the proof.
\end{proof}

\clearpage
\section{DFI under transport maps}\label{app:sec:ot-maps}
    \subsection{Optimal transport maps}\label{app:subsec:ot-maps}

    Depending on the distributions of features and/or latent variables, the randomized transformation (through kernel $k$) reduces to a transport map $T:\RR^d\to\RR^d$ such that $Z=T(X)$.
    Two examples of such maps are as follows:
    \begin{enumerate}[(i)]
    \item Bures-Wasserstein map \citep{givens1984class}:
        If $X\sim N(0,\Sigma)$ and $Z\sim\cN_d(0_d,I_{d})$, the optimal map is the whitening transform map $T(x)=\Sigma^{-\frac{1}{2}}x$.
    
    \item Knothe-Rosenblatt transport \citep{rosenblatt1952remarks}:
    For any $X$ with joint distribution function $F$ and $Z\sim \Unif[0,1]^{\otimes d}$, the triangular Knothe-Rosenblatt transport map is given by
    $
    T(x)=(
       F_{X_1}(x_{1}),
       F_{X_2\mid X_1}(x_{2}\mid x_{1}),
       $
       $\dots,
       F_{X_d\mid X_{[d-1]}}(x_{d}\mid x_{{[d-1]}})
    )
    $,  which pushes $X$ to $Z\sim\Unif[0,1]^{\otimes d}$.
    \end{enumerate}
    When the covariate distribution is absolutely continuous, the optimal transport map $T$ pushing $\PP_X$ to $\PP_Z$ is unique and can be characterized as the gradient of a convex function \citep{brenier1991polar}. 
    This property ensures that the disentangled representation $Z = T(X)$ is well-defined and unique up to sets of measure zero.
    The optimal transport map is applicable to continuous or discrete covariates that admit joint absolute continuity; although, in real-world applications, covariate distributions may exhibit singular components (e.g., a mixture of continuous and discrete variables), rendering the optimal transport map ill-defined or non-unique.

    For our theoretical analysis in this section, we focus on the optimal transport map, which provides a principled method for transforming one probability distribution into another by minimizing a specified transportation cost. Consider two probability measures $\PP_X$ and $\PP_Z$ on $\RR^d$. Given a cost function $c(x, z)$, the Monge problem of optimal transport seeks a measurable map $T: \RR^d \to \RR^d$ solving:
    \begin{align}
        \inf_{T_\sharp \PP_X = \PP_Z} \int_{\RR^d} c(x, T(x)) \rd \PP_X(x),
    \end{align}
    where $T_\sharp \PP_X = \PP_Z$ indicates that $T$ pushes forward $\PP_X$ to $\PP_Z$, i.e., $Z = T(X) \sim \PP_Z$ whenever $X \sim \PP_X$.
    In this work, we focus on the squared Euclidean distance:
    \[
        c(x, z) = \|x - z\|^2.
    \]
    Under this quadratic cost, optimal transport maps exhibit desirable mathematical properties, leading naturally to disentangled representations of the original random feature $X$.
    
    To guarantee the existence and uniqueness of the optimal transport map, we impose the following standard assumption:
    
    \begin{assumption}[Absolute continuity]\label[assumption]{asm:abs-continuity}
        The measure $\PP_X$ is absolutely continuous with respect to the Lebesgue measure on $\RR^d$.
    \end{assumption}
    
    Under Assumption \ref{asm:abs-continuity}, Brenier's theorem \citep{brenier1991polar} guarantees the existence of a unique (almost everywhere) optimal transport map, represented as the gradient of a convex function:
    
    \begin{theorem}[Brenier's theorem]\label{thm:brenier}
        Under Assumption \ref{asm:abs-continuity}, there exists a convex potential function $\Phi: \RR^d \to \RR$ such that the optimal transport map is uniquely determined $\PP_X$-almost surely by:
        \[
            T(x) = \nabla \Phi(x).
        \]
    \end{theorem}

    Consequently, the random vector $Z = T(X)$, obtained via an optimal transport map, provides a unique (up to a set of measure zero) disentangled representation of the original data $X$.

    \subsubsection{Latent feature importance via optimal transport}\label{app:subsec:ot-dfi}
    To present the main results on the statistical properties of the proposed estimator, we first introduce technical assumptions.
    Apart from \Cref{asm:abs-continuity}, we also require convexity of the Brenier potential.
    Let $\cC^2(\RR^d)$ denote the space of twice continuously differentiable functions on $\RR^d$.

    \begin{assumption}[Convexity]\label[assumption]{asm:strong-convexity}
        The Brenier potential $\Phi$ is a convex function such that $\Phi\in\cC^2(\RR^d)$ and $\lambda^{-1}I_d\preceq \nabla^2\Phi(x) \preceq \lambda I_d$ for all $x\in\RR^d$ and for some universal constant $\lambda>1$.
    \end{assumption}

    To present our main results, we introduce assumptions on the data-generating process and the nuisance functions, similar to \Cref{asm:data}.

    \begin{assumption}[Data model and estimator]\label{OT:asm:data}
    There exist a constant $C>0$ such that:
    \begin{enumerate}[(i)]

        \item\label{OT:asm:data:cov} Covariate $X\in\RR^d$ has zero mean $\EE[X]=0$ and bounded second moment $\|X\|_{L_2}\le C$;

        \item\label{OT:asm:data:response} Response $Y\in\RR$ has zero mean $\EE[Y]=0$ and bounded moment $\|Y\|_{L_\infty}\le~C$;

        \item\label{OT:asm:data:estimator} Nuisance functions $\hat\mu$ and $\hat T$ are estimated on auxiliary samples independent of the $n$ observations used in $\PP_n$ such that $\mu$ is Lipshitz continuous with Lipshitz constant $\Lip(\mu) = L_{\mu}$, \( \|\hat\mu\|_{L_{\infty}(\PP_X)}\leq C\) and \(\|\hat\mu\circ \hat T^{-1}\|_{L_{\infty}(\PP_Z)}\leq C\) with probability tending to one.
    \end{enumerate}
    \end{assumption}

    We can decompose the estimation error into two primary sources, namely, the error in estimating the conditional expectation $\mu$ and the error in estimating the transport map $T$.

    \begin{theorem}[Estimation error decomposition]
    \label{thm:error-decom-T}
    Under Assumptions \ref{asm:abs-continuity}, \ref{asm:strong-convexity}, and \ref{OT:asm:data}, if the nuisance estimator $(\hat{\mu},\hat{T})$ of $(\mu,T)$ is independent of $n$ observations of $O=(Z,X,Y)$.
    Denote \(\Delta_j(z):=\eta(z)-\eta_{-j}(z)\).
    Then the DFI estimator \eqref{def:hat-phi-Z} satisfies that
    \begin{align*}
         \hat\phi_{Z_j}(\PP)-\phi_{Z_j}(\PP)
          \;=\;
          (\PP_n-\PP)\{\varphi_{Z_j}(O;\PP)\}
           +\cO(n^{-1/2}\cE_Z + \cE_Z^2),
    \end{align*}
    with probability at least $1-\varepsilon_{\mu}^{2}$.
    Here, the influence function is given by
    \[
        \varphi_{Z_j}(O;\PP) = 2 (Y - \eta(Z))\Delta_j(Z) + \Delta_j(Z)^2 - \phi_{Z_j}(\PP),
    \]
    and the remainder term is given by \(\cE_Z := \|\hat\mu-\mu\|_{L_2(\PP_X)} + \,\| \hat T-T\|_{L_2(\PP_X)}\).
    \end{theorem}
    
    Similar to \Cref{thm:error-decom-eot}, the error term is second-order in the estimation errors of the regression function $\mu$ and the transport map $T$. 
    For optimal transport map estimation, recent work has established rates that satisfy our requirement under analogous smoothness conditions on the transport potential. For example, \citet{manole2024plugin} show that wavelet-based estimators for the Brenier potential achieve a rate for $\|\hat T - T\|_{L_2(\PP_X)}$ that is faster than $n^{-1/4}$ when the potential's smoothness is sufficiently large relative to the dimension $d$.
    EIF can be estimated analoguous as in \eqref{eq:eif-phi-Z-eot-est} with nuisance function replaced by $\hat{\eta}=\hat{\mu}\circ\hat{T}^{-1}$.

    \subsubsection{Estimation of attributed feature importance}

    We first introduce a general result on the efficient influence function of $\phi_{X_l}(\PP)$ for a general transport map $T$, similar to \Cref{prop:eif-eot-weighted}.
    \begin{proposition}[Efficient influence function of $\phi_{X_{l}}(\PP)$]\label[proposition]{prop:eif-loco-weighted}
        Let \(O=(X,Y)\sim \PP\) with \(X\in\RR^d\) and \(Y\in\RR\).  
        Let \(T:\RR^d\to\RR^d\) be a transport map and write \(Z=T(X)\).
        Fix \(j,l\in[d]\) and set
        \[
            \Delta_j(z):=\eta(z)-\eta_{-j}(z),\qquad
            H_{jl}(x):=(\partial_{z_j}(T^{-1})_l(z))^2\mid_{z=T(x)}.
        \]
        Define the parameter
        \[
            \phi_{jl}(\PP) := \PP\{\Delta_j(Z)^2\,H_{jl}(X)\},\qquad \theta_{jl}(\PP) := \PP\{H_{jl}(X)\}.
        \]
        Assume a non-parametric model \(\cM\) in which \(Y\mid X=x\) has finite second moment for \(\PP\)-a.e.\ \(x\), and the efficient influence function $\varphi_{H_{jl}}$ of $\theta(\PP)$ exists.
        Then, \(\phi_{jl}\) is pathwise differentiable at every \(\PP\in\cM\) with efficient influence function
        \begin{align}
            \varphi_{\Delta_j^2H_{jl}}(O;\PP) &= H_{jl}(X)\Delta_j(Z)^2-\phi_{jl}(\PP) + 2 H_{jl}(X) [Y-\eta(Z)] \Delta_j(Z)
             \notag \\
            &\qquad  + (\varphi_{H_{jl}}(O;\PP) - H_{jl}(X) + \theta_{jl}(\PP)) \Delta_j(Z)^2. \label{eq:eif}
        \end{align}
    \end{proposition}

    In general, obtaining a closed-form expression for the efficient influence function of a functional $\theta_{jl}(\PP)$ is challenging. For the class of transport maps introduced in \Cref{app:subsec:ot-dfi}, this task becomes more tractable.
    Below, we restrict attention to the Gaussian setting. 

    \begin{assumption}[Covariate distribution]\label[assumption]{asm:dist-X}
        $X\sim \cN(0,\Sigma)$ and $\lambda^{-1}I_d\preceq\Sigma\preceq \lambda I_d$ for some constant $\lambda>1$.
    \end{assumption}

    Under \Cref{asm:dist-X}, the optimal Brenier map is linear and therefore admits an explicit form.
    By \Cref{thm:brenier}, this map is $\PP_X$-almost surely unique, ensuring that the associated feature-importance measures $\phi_{Z_j}$ and $\phi_{X_l}$ are likewise well defined and unique. 
    We have the following result.

    \begin{theorem}[Estimation error of $\phi_{X_{l}}(\PP)$ under Bures-Wasserstein transport]\label{thm:err-phi-X}
        Under Assumptions \ref{asm:abs-continuity}, \ref{asm:strong-convexity}, \ref{OT:asm:data}, and \ref{asm:dist-X}, consider the Bures-Wasserstein transport map $Z=T(X)=L^{-1}X$ where $L:=\Sigma^{\frac{1}{2}}$, and the nuisance estimator $(\hat{\mu},\hat{L})$ of $(\mu,L)$ is independent of $n$ observations of $O=(Z,X,Y)$.
        The attributed feature importance estimator
        \begin{align}
            \widehat\phi_{X_l}(\PP)
    	&\;:=\;  \frac{1}{2}\sum_{j=1}^d \PP_n\left\{
                 \hat L_{jl}^2 [(Y - \hat{\eta}({Z}^{(j)}) )^2 - ( Y - \hat{\eta}({Z}) )^2] 
            \right\}, \label{def:hat-phi-X}
        \end{align}
        satisfies that:
        \begin{align*}
             \hat\phi_{X_l}(\PP)-\phi_{X_l}(\PP)
              \;=\;
              (\PP_n-\PP)\{\varphi_{X_l}(O;\PP)\}
               +\Op(n^{-1/2}\varepsilon_n + \varepsilon_n^2),
        \end{align*}
        where the efficient influence function is given by
        \[
            \varphi_{X_l}(O;\PP) = 
            \sum_{j=1}^d L_{jl}^2 [2 (Y - \eta(Z))\Delta_j(Z) + \Delta_j(Z)^2 ]- \phi_{X_l}(\PP),
        \]
        and the remaining term satisfies that $\varepsilon_n \;:=\; \|L_{\cdot l}\|_{2} \|\hat\mu - \mu\|_{L_2}  + \tr(\Sigma)^{1/2} \|\hat{L}-L\|_{\oper} $.
    \end{theorem}

\subsection{Proof}\label{app:subsec:proof-ot}
\subsubsection[Proof of Theorem \ref{thm:error-decom-T}]{Proof of \Cref{thm:error-decom-T}}

\begin{proof}
    From Steps 1 and 2 of the proof of \Cref{thm:error-decom-eot}, we have the decomposition \eqref{eq:decomp-cpi-2}:
    \begin{align*}
        \hat\phi_{Z_j}-\phi_{Z_j} & \;=\;  (\PP_n-\PP) \{\varphi(O;\PP)\} + \Op(n^{-1/2}\|\delta\|_{L_2} + \|\delta\|_{L_2}^2).
    \end{align*}
    It remians to bound \(\|\delta\|_{L_2}=\|\hat\eta-\eta\|_{L_2(\PP_Z)}\).

    Recall that $\eta=\mu\circ T^{-1}$, $\hat\eta=\hat\mu\circ\hat T^{-1}$, and decompose for every $z\in\cZ$
    \begin{equation}\label{eq:S3-split}
    \hat\eta(z)-\eta(z)
        = \hat\mu(\hat T^{-1}(z))
                        -\hat\mu(T^{-1}(z))
        + \hat\mu(T^{-1}(z))
                        -\mu(T^{-1}(z)) = : B_1 + B_2.
    \end{equation}

    Term $B_1$ is the transport error.
    Set
    \(X:=T^{-1}(Z)\sim \PP_X\) and
    \(X':=\hat T^{-1}(Z)\sim \PP_{X'}\),
    and define
    \[
        \varepsilon_{\mu}:=\|\hat \mu-\mu\|_{L_{2}(\PP_X)},\qquad  \varepsilon_{T^{-1}}
        :=\|\hat T^{-1}-T^{-1}\|_{L_{2}(\PP_Z)},\qquad  \varepsilon_{T}
        :=\|\hat T-T\|_{L_{2}(\PP_X)}.
    \]
    Introduce the truncation event $E:=\{\|Z\|\le R_{\varepsilon_{\mu}}/\lambda\}$ with $R_{\varepsilon_{\mu}}=\lambda\sqrt{d+2\log(1/\varepsilon_{\mu})}$, for which $\PP(E)\ge1-\varepsilon_{\mu}^{2}$ by a Gaussian tail bound.
    On $E$ both $X$ and $X'$ lie in the ball
    \(B:=\{x\in\RR^{d}:\|x\|\le R_{\varepsilon_{\mu}}\}\).
    By \Cref{prop:Lip-bound}, with probability at least $1-\varepsilon_{\mu}^{2}$,
    \begin{equation}\label{eq:good-Lip-local}
        L_{\mu,\varepsilon_{\mu}}
        :=\Lip(\hat\mu;B)
        \;\le\;
        L_\mu
        +
        C\,L_\mu^{1/2}\lambda^{-1/2}
        [d+2\log(1/\varepsilon_{\mu})]^{-1/4}
        \varepsilon_{\mu}^{1/2},
    \end{equation}
    where $C=4\sqrt2$ is the universal constant from the proposition.
    Event~\eqref{eq:good-Lip-local} is contained in $E$, so it also occurs
    with probability at least $1-\varepsilon_{\mu}^{2}$.
    At this event, we have
    \[
        |B_1|
        =|\hat\mu(X')-\hat\mu(X)|
        \le
        L_{\mu,\varepsilon_{\mu}}\,\|X'-X\|,
        \qquad
        \EE_Z[B_1^{2}]
        \le
        L_{\mu,\varepsilon_{\mu}}^{2}\,\varepsilon_{T^{-1}}^{2} \leq \lambda^2 L_{\mu,\varepsilon_{\mu}}^{2}\,\varepsilon_{T}^{2},
    \]
    where the last inequality is from \Cref{asm:strong-convexity}.

    Term $B_2$ is the regression error. By definition,
    \[
        \EE_Z[B_2^{2}]
        =\|\hat\mu-\mu\|_{L_2(\PP_X)}^{2}
        =:\varepsilon_{\mu}^{2}.
    \]

    Use the inequality $(a+b)^2\le2a^{2}+2b^{2}$ with
    $a=B_1$ and $b=B_2$:
    \[
        \|\hat\eta-\eta\|_{L_2(\PP_Z)}^{2}
        =\EE_Z[(B_1+B_2)^2]
        \le
        2\,\EE_Z[B_1^{2}]
        +
        2\,\EE_Z[B_2^{2}]
        \le
        2\lambda^2 L_{\mu,\varepsilon_{\mu}}^{\,2}\,\varepsilon_T^{2}
        +2\,\varepsilon_{\mu}^{2}.
    \]
    Taking square roots and keeping only the leading factor
    $\sqrt2$ yields
    \begin{equation}\label{eq:eta-bound}
        \|\hat\eta-\eta\|_{L_2(\PP_Z)}
        \;\le\;
        \sqrt2\lambda
        L_{\mu,\varepsilon_{\mu}}\,\varepsilon_T
        +
        \sqrt2\,\varepsilon_{\mu},
    \end{equation}
    with probability at least $1-\varepsilon_{\mu}^2 $.

    Finally, combining \eqref{eq:decomp-cpi-2} with \eqref{eq:eta-bound}, we further have
    \begin{align*}
        \hat{\phi}_{Z_j}-\phi_{Z_j}
      =(\PP_n-\PP)\varphi(O;\PP)
       +\cO(n^{-1/2}(\| \hat\mu-\mu \|_{L_2(\PP_X)}^2+ \| \hat T-T\|_{L_2(\PP_X)}^2)^{1/2} \\
            +  \| \hat\mu-\mu \|_{L_2(\PP_X)}^2 \;+\; \| \hat T-T\|_{L_2(\PP_X)}^2
                  ),
    \end{align*}
    with probability at least $1-\| \hat\mu-\mu \|_{L_2(\PP_X)}^2 $.
    This completes the proof.
\end{proof}

\subsubsection[Proof of Proposition \ref{prop:eif-loco-weighted}]{Proof of \Cref{prop:eif-loco-weighted}}\label{app:sub:prop:eif-loco-weighted}
\begin{proof}
    Throughout the proof, we drop the subscript $j$ for simplicity and put
    \[
        A(O)\;:=\;\Delta(Z)^{2},
        \qquad 
        B(O)\;:=\;H(X),
        \qquad 
        \psi_{A}(\PP)\;:=\;\PP [A(O) ],
        \qquad 
        \psi_{B}(\PP)\;:=\;\theta(\PP)=\PP [B(O) ].
    \]
    With this notation, the target parameter is the expectation of a product,
	\(
		\phi(\PP)=\PP[A(O)B(O)].
	\)

    \noindent\textbf{Step 1: EIF of $A$.}
    From \citet[Lemma 1]{williamson2021nonparametric}, the efficient influence function (EIF) of $\psi_{A}(\PP)=\EE[\Delta(Z)^{2}]$ is
	\[
		\varphi_{A}(O;\PP)
		=
		2\,[Y-\eta(Z)]\,\Delta(Z)\;+\;\Delta(Z)^{2}\;-\;\psi_{A}(\PP).
	\]
	Writing $A_{c}(O)=A(O)-\psi_{A}(\PP)$ we therefore have the decomposition
	\[
		\varphi_{A}(O;\PP)
		=\;
		A_{c}(O)\;+\;\alpha_{A}(O),
		\qquad 
		\alpha_{A}(O):=2\,[Y-\eta(Z)]\,\Delta(Z).
	\]

    \noindent\textbf{Step 2: EIF of $B$.}
    By assumption the (non‑degenerate) parameter $\theta(\PP)=\EE[H(X)]$ is pathwise differentiable with EIF $\varphi_{H}(O;\PP)$.  
    Writing $B_{c}(O)=B(O)-\theta(\PP)$ we set
	\[
		\alpha_{B}(O)\;:=\;\varphi_{H}(O;\PP)\;-\;B_{c}(O)
		\;=\;
		\varphi_{H}(O;\PP)\;-\; \{H(X)-\theta(\PP) \}.
	\]

    \noindent\textbf{Step 3: EIF of the product functional.}
    Apply \Cref{lem:EIF-product} with $A,B,\alpha_{A},\alpha_{B}$ defined above, it follows that
	\[
		\varphi_{{\Delta^2H}}(O;\PP)
		\;=\;
		A(O)B(O)-\phi(\PP)\;+\;B(O)\,\alpha_{A}(O)\;+\;A(O)\,\alpha_{B}(O).
	\]
	Substituting the explicit expressions gives
	\begin{align*}
		\varphi_{\Delta^2H}(O;\PP)
		&=
		H(X)\,\Delta(Z)^{2}
		-\phi(\PP)
		+\;
		H(X)\, \{2\,[Y-\eta(Z)]\,\Delta(Z) \}
		\\
		&\quad
		+\;
		\Delta(Z)^{2}\,
		 \{\varphi_{H}(O;\PP)-H(X)+\theta(\PP) \}.
	\end{align*}
	Re‑arranging the terms completes the proof.
\end{proof}

\subsubsection[Proof of Theorem \ref{thm:err-phi-X}]{Proof of \Cref{thm:err-phi-X}}\label{app:sub:them:err-phi-X}

\begin{proof}
    Before we present the proof, we will introduce some notation for clarity.
    Fix a coordinate $l\in[d]$.  
    For each $j\in[d]$, define
    \[
    	H_{jl}(x)\;:=\;(\partial_{x_l}T_j(x))^2,
    	\qquad 	
    	\Delta_j(z)\;:=\;\eta(z)-\eta_{-j}(z_{-j}),
    	\qquad 	
    	\phi_{jl}(\PP)\;:=\;\EE [H_{jl}(X)\,\Delta_j(Z)^2 ].
    \]
    Under the Bures--Wasserstein (BW) map $T(x)=Lx$ we have  
    $H_{jl}(x)=L_{jl}^{2}$, a constant that depends only
    on $L:=\Sigma^{-\frac{1}{2}}$.  
    The target parameter is the sum
    \(
    	\phi_{X_l}=\sum_{j=1}^{d}\phi_{jl}(\PP).
    \)

    \noindent\textbf{Step 1: Influence function expansion.}
    From \Cref{prop:eif-loco-weighted,prop:EIF-H}, the efficient influence function of $\phi_{X_l}$ is given by
    \[
    	\varphi_{X_l}(O;\PP)
    	=
    	\sum_{j=1}^{d}
    		\left\{
                L_{jl}^2\,
    		 [
    			2\,(Y-\eta(Z))\,\Delta_j(Z) + \Delta_j(Z)^{2}
    		 ]
                - 
                L_{jl} (Z_jZ_l - \delta_{jl}) \Delta(Z)^2 \right\}
    		-
    		\phi_{X_l}.
    \]
    Let $ \{(X_i,Z_i,Y_i) \}_{i=1}^{n}$ be an i.i.d.\ sample independent of the nuisance estimates $(\widehat\eta,\widehat\eta_{-j},\widehat T)$,
    \(
    	\widehat H_{jl}(x)=(\partial_{x_l}\,\widehat{T}_j(x))^2
    \)
    and
    $
    	\widehat\Delta_j(z)=\widehat\eta(z)-\widehat\eta_{-j}(z_{-j})
    $.
    
    Recall that the one-step estimator is defined as
    \begin{align}
        \widehat\phi_{X_l}^{\loco}(\PP) &\;=\; \phi_{X_l}(\hat{\PP}) + \PP_n\{\varphi(O; \hat{\PP})\} \notag\\
            &=  \frac{1}{2}\PP_n\left\{
            \sum_{j=1}^d \hat L_{jl}^2 [(Y - \hat{\eta}({Z}^{(j)}) )^2 - ( Y - \hat{\eta}({Z}) )^2] 
        \right\} + \Op(n^{-1/2}\|\delta\|_{L_2} + \|\delta\|_{L_2}^2) \notag\\
        &=  \widehat\phi_{X_l}(\PP) + \Op(n^{-1/2}\|\delta\|_{L_2} + \|\delta\|_{L_2}^2), \label{eq:diff-philoco-phi}
    \end{align}
    where the second equality follows similarly from Step 2 in the proof of \Cref{thm:error-decom-eot} and the remainder term corresponds to the difference of the LOCO estimator $\widehat\phi_{X_l}^{\loco}(\PP)$ and the CPI estimator $\widehat\phi_{X_l}(\PP)$, weighted by $\hat L_{jl}^2$'s.
    Note that $\hat{\eta}_{-j}(Z_{-j}) =\EE[\hat{\eta}(Z) \mid Z_{-j}]$ and $\hat T(X) = \hat{L} X$ and $\partial_{x_l}\hat T_j(X) = \hat{L}_{jl}$; hence, we only need to estimate two nuisance functions, essentially.

    Because the nuisance functions are estimated from independent samples other than $\PP_n$, the classical von~Mises expansion (see, e.g., \citet[Equation (2.2)]{du2025causal}) reads
    \begin{equation}\label{eq:vM}
    	\widehat\phi_{X_l}^{\loco}-\phi_{X_l}
    	=
    	(\PP_n-\PP) \{\varphi_{X_l}(O;\PP) \}
    	+
    	\underbrace{(\PP_n-\PP) [\varphi_{X_l}(O;\widehat \PP)-\varphi_{X_l}(O;\PP) ]}_{E_n}
        +
    	\underbrace{\PP [\varphi_{X_l}(O;\widehat \PP)-\varphi_{X_l}(O;\PP) ]}_{R_n} .
    \end{equation}
    It therefore suffices to show
    $
    	|E_n|=\Op(n^{-1/2}\,\varepsilon_n),
    $
    and
    $
    	|R_n|=\Op(\varepsilon_n^{2}).
    $

    \noindent\textbf{Step 2: Bounding the empirical process term $E_n$.}
    Conditioning on the auxiliary split, the summands $\varphi_{X_l}(O;\widehat \PP)-\varphi_{X_l}(O;\PP)$ are i.i.d.\ with mean~$0$,
    so that
    \[
    	|E_n|
    	\;=\;
    	 |
    		(\PP_n-\PP)
    		 [
    			\varphi_{X_l}(O;\widehat \PP)-\varphi_{X_l}(O;\PP)
    		 ]
    	 |
    	\le
    	n^{-1/2}\,
    	 \|\varphi_{X_l}(O;\widehat \PP)-\varphi_{X_l}(O;\PP) \|_{L_2}
    	= \Op(n^{-1/2}\,\varepsilon_n)
    \]
    where the last equality is from \Cref{lem:err-varphi-X}.
    
    \noindent\textbf{Step 3: Bounding the bias remainder $R_n$.}
    From \Cref{lem:err-varphi-X}, the bias term is bounded by $|R_n|=\Op(\varepsilon_n^{2})$.
    
    Finally, combining the bounds for $E_n$ and $R_n$ with \eqref{eq:diff-philoco-phi} and \eqref{eq:vM} completes the proof.    
\end{proof}

\subsubsection{Lemmas related to efficient influence functions}

\begin{proposition}[EIF of $\theta(\PP)$ under Bures-Wasserstein transport]\label[proposition]{prop:EIF-H}
    Let the observable $O=X\in\RR^{d}$ satisfy
    \[
      \PP[X]=0,
      \qquad
      \Sigma=\PP[XX^{\top}]\;\succ\;0,
      \qquad
      L:=\Sigma^{-1/2}.
    \]
    For fixed indices $j,l\in[d]$ define
    \[
      T(x):=L x,
      \qquad
      H_{jl}(X):=\partial_{x_l}T_j(X)=L_{jl}.
    \]
    The efficient influence function of the parameter
    \[
      \theta_{jl}(\PP):=\PP [H_{jl}(X)^2 ]
                     =L_{jl}^{2},
    \]
    is given by
    \[
        \varphi_{jl}(O;\PP)
          = -\;L_{jl} (X^{\top}L e_{j}e_{l}^{\top}L X-\delta_{jl}) = -L_{jl} (Z_jZ_l - \delta_{jl}).
    \]
\end{proposition}
\begin{proof}
    Take an arbitrary regular one‑dimensional sub‑model $\{\PP_\varepsilon:\varepsilon\in(-\varepsilon_0,\varepsilon_0)\}$ embedded in the non‑parametric model such that $\PP_{\varepsilon=0}=\PP$ and with score $s_\varepsilon(O)=\partial_{\varepsilon}\log p_\varepsilon(O) |_{\varepsilon=0}$.
    The usual regularity conditions give $\PP[s(O)]=0$ and $\PP[Xs(O)]=0$ (the latter enforces the constraint $\EE[X]=0$ along the path).

    Write $\Sigma(\varepsilon):=\EE_{\PP_\varepsilon}[XX^{\top}]$.
    Differentiating at $\varepsilon=0$,
    \[
      \partial_{\varepsilon}\Sigma(\varepsilon) |_{\varepsilon=0}
         =\PP [
             XX^{\top}\,s(O)
            ]
         \;=\;\PP [
             \{XX^{\top}-\Sigma\}\,s(O)
            ],
    \]
    because $\PP[s(O)]=0$. 
    Hence, the influence function for~$\Sigma$ is
    \[
      U_\Sigma(O)\;:=\;XX^{\top}-\Sigma,
      \qquad
      \PP[U_\Sigma(O)]=0.
    \]

    Treat $\theta_{jl}$ as the composition $\theta_{jl}=h\circ\Sigma$ with $h(M)=(M^{-1/2})_{jl}^{2}$.  
    A matrix calculus identity gives the Fréchet derivative of $M\mapsto M^{-1/2}$:
    \[
      \rd L
        \;=\;
        \rd(\Sigma^{-\frac{1}{2}})
        \;=\;
        -\tfrac12\,L\,(\rd \Sigma)\,L.
    \]
    Because
    $\langle A,B\rangle_{F}:=\tr(A^{\top}B)$, we can rewrite the last term
    as an inner product:
    \[
      \rd\theta_{jl}
        \;=\;
         \langle
           \nabla_{\Sigma}\theta_{jl}(\Sigma),\;
           \rd\Sigma
          \rangle_{F},
      \quad
      \nabla_{\Sigma}\theta_{jl}(\Sigma)
         =-\frac{L_{jl}}{2}\,
           (L e_{j}e_{l}^{\top}L
                 +L e_{l}e_{j}^{\top}L).
    \]

    Insert $\rd\Sigma=
      \partial_{\varepsilon}\Sigma(\varepsilon)|_{\varepsilon=0}
     =\PP[U_\Sigma(O)\,s(O)]$ from Step 1:
    \[
      \partial_{\varepsilon}\theta_{jl}(\PP_\varepsilon)\Big|_{\varepsilon=0}
      \;=\;
       \langle\nabla_{\Sigma}\theta_{jl},\,
                \PP[U_\Sigma(O)\,s(O)] \rangle_{F}
      \;=\;
      \PP
         [
          \langle\nabla_{\Sigma}\theta_{jl},U_\Sigma(O)\rangle_{F}\;s(O)
         ].
    \]
    Therefore, the efficient influence function $\varphi_{jl}$ in the tangent space satisfies that $\partial_{\varepsilon}\theta_{jl}(\PP_\varepsilon)|_{\varepsilon=0} =\PP[\varphi_{jl}(O)\,s(O)]$ for every score $s$ is
    \[
        \varphi_{jl}(O;\PP)
          = \langle\nabla_{\Sigma}\theta_{jl}(\Sigma),
                     U_\Sigma(O) \rangle_{F}
          =-\;L_{jl} \{
              X^{\top}L e_{j}e_{l}^{\top}L X-\delta_{jl}
             \}.
    \]
    Centring holds automatically because
    $\PP[X^{\top}L e_{j}e_{l}^{\top}L X]=\delta_{jl}$, so
    $\PP[\varphi_{jl}(O;\PP)]=0$.
\end{proof}

\begin{lemma}[Estimated influence function for $\phi_{X_l}$]\label[lemma]{lem:err-varphi-X}
    Under the assumptions of \Cref{thm:err-phi-X}, let
  \[
     \varphi_{X_{l}}(O;\PP)
       \;=\;
       \sum_{j=1}^{d}
    		\left\{
                L_{jl}^2\,
    		 [
    			2\,(Y-\eta(Z))\,\Delta_j(Z) + \Delta_j(Z)^{2}
    		 ]
                - 
                L_{jl} (Z_jZ_l - \delta_{jl}) \Delta(Z)^2 \right\}
    		-\;
       \phi_{X_{l}}(\PP),
  \]
  and let $\varphi_{X_l}(O;\hat \PP)$ be the same expression with
  $(\eta, L)$ replaced by $(\hat\eta,\hat L)$ and $\eta_{-j}$ replaced with $\hat\eta_{-j} = \EE_{Z_j}[\hat{\eta}(Z)\mid Z_{-j}]$.
  Then
  \[
      \|\varphi_{X_l}(O;\hat \PP)-\varphi_{X_l}(O;\PP) \|_{L_2}
       \;=\;\Op(\varepsilon_n),
     \qquad
     \PP \{\varphi_{X_l}(O;\hat \PP)-\varphi_{X_l}(O;\PP) \}
       \;=\;\Op(\varepsilon_n^{2}),
  \]
  where $\varepsilon_n = \|L_{\cdot l}\|_{2}\|\hat \mu -\mu\|_{L_2} + \tr(\Sigma)^{1/2} \|\hat{L} - L\|_{\oper}$.
\end{lemma}
\begin{proof}
    For simplicity, we will not explicitly indicate the dependency on \( Z \).
    For $l\in[d]$ fixed and for every $j\in[d]$, we use the following notation $\delta=\hat\eta-\eta$, $\delta_{-j}=\hat\eta_{-j}-\eta_{-j}$, $\vartheta_{j}=\delta-\delta_{-j}$, $\Delta_j = \eta - \eta_{-j}$, $\hat\Delta_j = \hat\eta - \hat\eta_{-j}$, $D_{jl}\;:=\;\hat L_{jl}^{2}\;-\;L_{jl}^{2}$, $U=Y-\eta$, $\hat U=Y-\hat\eta$.

    \noindent\textbf{Step 1: Decomposition.}
    Put
    \[
       A_{j}\;:=\;L_{jl}^{2} [\,2U\Delta_j+\Delta_j^{2} ],
       \qquad
       B_{j}\;:=\;L_{jl}\,(Z_jZ_l-\delta_{jl})\,\Delta_j^{2},
    \]
    and denote by $\hat A_{j},\hat B_{j}$ the corresponding quantities with
    $(L_{jl},\eta,\Delta_j)$ replaced by
    $(\hat L_{jl},\hat\eta,\hat\Delta_j)$.
    Then
    \begin{align}
       \varphi_{X_l}(O;\hat\PP)-\varphi_{X_l}(O;\PP)
       \;=\;
       \sum_{j=1}^{d}(E_{1,j}+E_{2,j})
             \;-\;
              [\phi_{X_l}(\hat\PP)-\phi_{X_l}(\PP) ],
       \label{eq:varphi-X-decomp}
    \end{align}
    where $E_{1,j}:=\hat A_{j}-A_{j}$ and $E_{2,j}:=-\hat B_{j}+B_{j}$.
    Using $\hat\Delta_j=\Delta_j+\vartheta_j$ and
    $U-\delta=Y-\hat\eta$ we obtain
    \begin{align}
       E_{1,j}
       &=\;
          D^{(2)}_{j} [\,2U\Delta_j+\Delta_j^{2} ]
          +\hat L_{jl}^{2} [\,2(U-\delta)\vartheta_j
          +\vartheta_j^{2}-2\delta_{-j}\Delta_j ],
       \label{eq:varphi-X-E1}\\[2pt]
       E_{2,j}
       &=\;
          -D^{(1)}_{j}\,(Z_jZ_l-\delta_{jl})\,\Delta_j^{2}
          -\hat L_{jl}(Z_jZ_l-\delta_{jl})
                [\,2\Delta_j\vartheta_j+\vartheta_j^{2} ],
       \label{eq:varphi-X-E2}
    \end{align}
    with
    \(
       D^{(2)}_{j}:=\hat L_{jl}^{2}-L_{jl}^{2},\;
       D^{(1)}_{j}:=\hat L_{jl}-L_{jl}.
    \)
    
    Because the plug‑in estimator of $\phi_{X_l}$ employs the empirical mean of
    $\hat\Delta_j^{2}$, 
    \begin{equation}
      \phi_{X_l}(\hat\PP)-\phi_{X_l}(\PP)
      \;=\;
      \sum_{j=1}^{d}
          \{
            D^{(2)}_{j}\,\PP[\Delta_j^{2}]
            \;+\;
            \hat L_{jl}^{2}\,\PP [\,2\Delta_j\vartheta_j+\vartheta_j^{2} ]
            \;+\;
            \hat L_{jl}^{2}\,(\PP_n-\PP) [\hat\Delta_j^{2} ]
          \}.
      \label{eq:varphi-X-phi}
    \end{equation}
    
    \noindent
    Substituting \eqref{eq:varphi-X-E1}-\eqref{eq:varphi-X-phi} into
    \eqref{eq:varphi-X-decomp} and collecting like terms yields
    \begin{equation}
      \varphi_{X_l}(O;\hat\PP)-\varphi_{X_l}(O;\PP)
      \;=\;
      T_{1l}+T_{2l}+T_{3l}+T_{4l}, \label{eq:varphi-X-final-split}
    \end{equation}
    where
    \begin{align*}
       T_{1l}
        &:=\sum_{j}D^{(2)}_{j} [\,2U\Delta_j+\Delta_j^{2}-\PP[\Delta_j^{2}] ],
    \\[2pt]
      T_{2l}
        &:=\sum_{j}\hat L_{jl}^{2}
             \{
              2(U-\delta)\vartheta_j
              +\vartheta_j^{2}
              -2\delta_{-j}\Delta_j
              -\PP[\,2\Delta_j\vartheta_j+\vartheta_j^{2} ]
             \},
    \\[2pt]
      T_{3l}
        &:=\sum_{j}
             [
              -D^{(1)}_{j}(Z_jZ_l-\delta_{jl})\,\Delta_j^{2}
              -\hat L_{jl}(Z_jZ_l-\delta_{jl})
                 (2\Delta_j\vartheta_j+\vartheta_j^{2})
             ],
    \\[2pt]
      T_{4l}
        &:=\,-\sum_{j}\hat L_{jl}^{2}\,(\PP_n-\PP)[\hat\Delta_j^{2}].
    \end{align*}

    \noindent\textbf{Step 2: $L_2$-norm bound.}
    We bound $\|T_{kl}\|_{L_2}$ for $k=1,\ldots, 4$, term by term.
    Define $e_{L,l}\;:= \|\hat L_{\cdot l} - L_{\cdot l}\|_{2}$, $e_{\eta} =\|\hat\eta -\eta\|_{L_2}$, and \(e_n= e_{L,l}+\|L_{\cdot l}\|_2 e_{\eta}\).

    \textit{(i) Term $T_{1l}$.}  
    Because $|D^{(2)}_{j}|\le 2C|D^{(1)}_{j}|$ and
    $ \|2U\Delta_j+\Delta_j^{2}-\PP[\Delta_j^{2}] \|_{L_2}\le 10C^{2}$,
    \[
       \|T_{1l}\|_{L_2}
          \le
          20C^{3}\sum_{j}|D^{(1)}_{j}|
          \le
          20C^{3}\,e_{L,l}
          \;=\;\Op(e_n).
    \]

    \textit{(ii) Term $T_{2l}$.}  
    Each factor in the curly braces of $T_{2l}$ contains at least one of \(\vartheta_j\) or \(\delta_{-j}\):
    \[
       \|\vartheta_j\|_{L_2}\le 2(\|\hat\eta-\eta\|_{L_2} + \|\hat\eta_{-j}-\eta_{-j}\|_{L_2}) ,\quad
       \|\delta_{-j}\|_{L_2}\le \|\hat\eta_{-j}-\eta_{-j}\|_{L_2}.
    \]
    Hence, from Cauchy-Schwartz inequality and \Cref{lem:eta-j}, it follows that
    \[
       \|T_{2l}\|_{L_2}
          \le
          \|\hat L_{\cdot l}\|_{L_2}
          (\|\hat\eta-\eta\|_{L_2}
                +\|\hat\eta_{-j}-\eta_{-j}\|_{L_2})
          \;=\;\Op(e_n),
    \]
    by noting that $\|\hat L_{\cdot l}\|_{L_2} \leq \|L_{\cdot l}\|_{2} + e_{L,l}$.
    
    \textit{(iii) Term $T_{3l}$.}  
    For the first piece in $T_{3l}$,
    \(
       \|(Z_jZ_l-\delta_{jl})\Delta_j^{2}\|_{L_2}\le 4C^{3},
    \)
    so
    \(
       \|D^{(1)}_{j}(\,\cdot\,)\|_{L_2}\le 4C^{3}|D^{(1)}_{j}|.
    \)
    Summation gives $\Op(e_{L,l})$.
    For the second piece, note that
    \(
      \|Z_jZ_l-\delta_{jl}\|_{L_2}\le 2C^{2}
    \)
    and
    \(
      \|2\Delta_j\vartheta_j+\vartheta_j^{2}\|_{L_2}
         \lesssim\|\hat\eta-\eta\|_{L_2}
                 +\|\hat\eta_{-j}-\eta_{-j}\|_{L_2} \leq 2 e_{\eta},
    \) from \Cref{lem:eta-j}.
    Hence, by Cauchy-Schwarz inequality, we have
    \[
       \|T_{3l}\|_{L_2}
          =\Op(e_n).
    \]

    \textit{(iv) Term $T_{4l}$.}  
    Conditional on the nuisance estimates,
    \(
       \VV \{(\PP_{n}-\PP)[\hat\Delta_j^{2}] \}
          \le \frac{16C^{4}}{n},
    \)
    so
    \(
       \|(\PP_{n}-\PP)[\hat\Delta_j^{2}]\|_{L_2}
          =\cO(n^{-1/2}).
    \)
    With $|\hat L_{jl}|\le C$,
    \(
       \|T_{4l}\|_{L_2}
          =\Op(n^{-1/2}).
    \)

    Putting the four bounds together, we have
    \[
        \|\varphi_{X_l}(O;\hat\PP)-\varphi_{X_l}(O;\PP) \|_{L_2}
          \le
          \|T_{1l}\|_{L_2}+\|T_{2l}\|_{L_2}+\|T_{3l}\|_{L_2}+\|T_{4l}\|_{L_2}
          \;=\;
          \Op(e_n+n^{-1/2}).
    \]

    \noindent\textbf{Step 3: Expectation bound.} We again bound the expectation term by term.
    
    \textit{(i) Term $T_{1l}$.}
    $T_{1l}$ is of the form
    \(D^{(2)}_{j}\{2U\Delta_j+\Delta_j^{2}-\EE[\Delta_j^{2}]\}\).
    Conditional on the (independent) nuisance estimates, the factor
    $D^{(2)}_{j}$ is fixed, while  
    \(\EE[\,2U\Delta_j\mid Z]=0\) and  
    \(\EE[\Delta_j^{2}]=\EE[\Delta_j^{2}]\) by definition, so the centered bracket has mean~0.  
    Thus, $\EE[T_{1l}]=0$.

    \textit{(ii) Term $T_{2l}$.}
    Each term inside the braces of $T_{2l}$ contains at least one factor
    \(\vartheta_j=\delta-\delta_{-j}\) or $\delta_{-j}$.  Because
    \(\|\delta_{-j}\|_{L_2}\le\|\delta\|_{L_2}=e_{\eta}\), Cauchy-Schwarz inequality gives
    \[
       |\EE[(U-\delta)\vartheta_j]|
            =|\EE[\delta\vartheta_j]|
            \le \|\delta\|_{L_2}\|\vartheta_j\|_{L_2}
            \lesssim e_{\eta}^2,
    \]
    and similarly
    \(|\EE[\delta_{-j}\Delta_j]|\lesssim e_{\eta}^{2}\).
    All other raw moments in the bracket are of this same order.
    Therefore, we have
    \[
       |\EE[T_{2l}]|
         \lesssim
         \|\hat L_{\cdot l}\|_{L_2}^{2}\, e_{\eta}^{2}
         =\Op(e_n^{2}).
    \]

    \textit{(iii) Term $T_{3l}$.}
    Because $\EE[Z_jZ_l-\delta_{jl}]=0$, we have
    \[
        \PP[T_{3l}]
        =-\sum_{j}\hat L_{jl}\,
            \EE [(Z_jZ_l-\delta_{jl})
                       (2\Delta_j\vartheta_j+\vartheta_j^{2}) ].
    \]
    Every summand contains at least one factor \(\vartheta_j\); by
    Cauchy-Schwarz,
    \[
        |\EE[(Z_jZ_l-\delta_{jl})
                (2\Delta_j\vartheta_j+\vartheta_j^{2})] |
       \le
       3C^{2}\|\vartheta_j\|_{L_2}\,
            \max  \{2\|\Delta_j\|_{\infty},\|\vartheta_j\|_{L_2} \}
       =\Op(e_{\eta}^{2}).
    \]
    Multiplying by $|\hat L_{jl}|\le C$ and summing gives
    \[
       |\EE[T_{3l}]|
         \lesssim
         \|\hat L_{\cdot l}\|_{L_2}\,e_{\eta}^{2}
         =\Op(e_n^{2}).
    \]

    \textit{(iv) Term $T_{4l}$.} 
    Because $\EE[(\PP_{n}-\PP)f]=0$ for any square-integrable $f$, we have $\EE[T_{4l}]=0$.

    Combining the above results yields that
    \(
      \PP\{\varphi_{X_l}(O;\hat \PP)-\varphi_{X_l}(O;\PP)\}
        =\Op(e_n^{2}).
    \)
    Finally, we will show that $e_n\lesssim \varepsilon_{n}$.
    By triangle inequality and Cauchy-Schwarz inequality, we have that
    \begin{align*}
        \|\hat\eta - \eta\|_{L_2} &= \| (\hat\eta -\eta)\circ\hat T^{-1} + (\eta \circ \hat T^{-1} - \eta \circ T^{-1}) \|_{L_2}  \\
         &\lesssim \|\hat\mu - \mu\|_{L_2} + \|\hat T - T\|_{L_2} .
    \end{align*}

    Let \(E := \hat L-L\).  Using the cyclic property of the trace,
\[
       \|\hat T-T\|_{L_2(\PP_X)}^{2}
          = \EE[\|EX\|_2^{2}]
          = \tr(E\Sigma E^{\top})
          = \tr(E^{\top}E\,\Sigma).
    \]
    For any two positive-semidefinite matrices \(A,B\succeq0\),
    \(\tr(AB)\le \lambda_{\max}(A)\,\tr(B)\).  Applying this with
    \(A = E^{\top}E\) and \(B = \Sigma\) yields
    \[
       \|\hat T-T\|_{L_2(\PP_X)}^{2}
          \le
          \lambda_{\max}(E^{\top}E)\,\tr(\Sigma)
          = \|E\|_{\oper}^{2}\,\tr(\Sigma)
          = \tr(\Sigma)\,\|\hat L-L\|_{\oper}^{2}.
    \]
    Furthermore, the column-wise error can also be bounded by:
    \[
       \max_{l\in[d]}\|\hat L_{\cdot\,l}-L_{\cdot\,l}\|_{2}
          \le
          \|\hat L-L\|_{\oper}.
    \]
    Thus, we further have $e_n \lesssim   \|\hat{L}_{\cdot l} - L_{\cdot l}\|_{2} + \|L_{\cdot l}\|_{2}[\|\hat \mu -\mu\|_{L_2} + \tr(\Sigma)^{1/2} \|\hat{L} - L\|_{\oper}] \lesssim \varepsilon_n$, which completes the proof.
\end{proof}

\subsubsection{Basic properties of transport maps}\label{app:sub:basic-transport}

\begin{lemma}[Bi-Lipschitz $\boldsymbol{\Longrightarrow}$ Bijective]\label[lemma]{lem:bilipshitz-bijective}
    Let $T:\mathbb{R}^{d}\to\mathbb{R}^{d}$ satisfy the \emph{global bi-Lipschitz} bounds
    \[
        c\,\|x-y\|
        \le
        \|T(x)-T(y)\|
        \le
        C\,\|x-y\|,
        \qquad\forall x,y\in\mathbb{R}^{d},
    \]
    for some constants $0<c\le C<\infty$.
    Then $T$ is a bijection and a bi-Lipschitz homeomorphism of $\mathbb{R}^{d}$ onto itself.    
\end{lemma}
\begin{proof}
We split the proof into two steps.

\noindent\textbf{Step 1: Bijection.}
We verify the properties of $T$ one by one:
\begin{enumerate}[(a)]
    \item \emph{Injectivity.}  If $T(x)=T(y)$ then $c\|x-y\|\le0$, so $x=y$. Hence, $T$ is injective
    
    \item \emph{Properness.}  From the lower bound,
          $\|T(x)\|\ge c\|x\|-\|T(0)\|$, hence $\|x\|\to\infty\implies\|T(x)\|\to\infty$.
          Thus inverse images of bounded sets are bounded (and closed by continuity), hence compact.
    \item \emph{Openness.}  For any ball $B(x,r)$ one has
          $B(T(x),cr)\subset T(B(x,r))\subset B(T(x),Cr)$,
          so $T$ maps open sets to open sets.
    \item \emph{Surjectivity.}  The image $T(\mathbb{R}^{d})$ is non-empty, open (3) and closed
          (properness $\,+$\, continuity).  Since $\mathbb{R}^{d}$ is connected,
          the only such subset is $\mathbb{R}^{d}$ itself.
    \end{enumerate}
    Consequently, $T$ is a bijection.

    \noindent\textbf{Step 2: Lipshitzness of inverse mapping.}
    For any two points \(u,v\in\mathbb R^{d}\) and set  
    \(x:=T^{-1}(u),\;y:=T^{-1}(v)\).
    This gives
    \[
            c\|x-y\|
            \le
            \|T(x)-T(y)\|
            \;=\;
            \|u-v\|.
    \]
    Rearranging yields
    \[
            \|T^{-1}(u)-T^{-1}(v)\|
            \;=\;
            \|x-y\|
            \le
            \frac1c\,\|u-v\|,
            \qquad\forall u,v\in\mathbb R^{d}.
    \]
    Hence \(T^{-1}\) is globally Lipschitz with 
    \[
        \Lip(T^{-1})\le1/c.
    \]
    Consequently, $T$ is a bijection and a bi-Lipschitz homeomorphism of $\mathbb{R}^{d}$ onto itself.
\end{proof}

\begin{lemma}[Bi-Lipschitz gradient $\Longleftrightarrow$ uniform Hessian bounds]\label[lemma]{lem:bilipshitz-potential}
    Let $\Phi:\mathbb{R}^d\to\mathbb{R}$ be a $C^{2}$, convex function and set $T=\nabla\Phi$.
    The following are equivalent:
    \begin{enumerate}[(a)]
    \item {Uniform Hessian bounds:}  
          There exist constants $0<m\le M<\infty$ such that, for all $x\in\mathbb{R}^d$,
          \begin{align}
              m\,I_d \;\preceq\; \nabla^2\Phi(x) \;\preceq\; M\,I_d .\label{eq:hessian}
          \end{align}
    \item {Global bi-Lipschitz gradient:}  
          There exist constants $0<m\le M<\infty$ such that, for all $x,y\in\mathbb{R}^d$,
          \begin{align}
              m\|x-y\|
              \le
              \|\nabla\Phi(x)-\nabla\Phi(y)\|
              \le
              M\|x-y\|.  \label{eq:bi-lipshitz}
          \end{align}
    \end{enumerate}
\end{lemma}
\begin{proof}
    We split the proof into different steps.
    
    \noindent\textbf{Step 1: (a) $\Rightarrow$ (b).}
    For $x,y\in\mathbb{R}^d$ define $\gamma(t)=y+t(x-y)$, $t\in[0,1]$.  
    By the fundamental theorem of calculus,
    \[
      \nabla\Phi(x)-\nabla\Phi(y)
      \;=\;
      \int_{0}^{1} \nabla^2\Phi(\gamma(t))\,(x-y)\rd t .
    \]
    Applying the bounds \eqref{eq:hessian} inside the integral gives
    \[
      m\|x-y\|
      \le
      \|\nabla\Phi(x)-\nabla\Phi(y)\|
      \le
      M\|x-y\|,
    \]
    establishing \eqref{eq:bi-lipshitz}.

    \noindent\textbf{Step 2: (b) $\Rightarrow$ (a).}
    Fix $x\in\mathbb{R}^d$ and a unit vector $v$.  
    For $t>0$ let $y=x+tv$.  Using \eqref{eq:bi-lipshitz},
    \[
      m t
      \le
      \|\nabla\Phi(x+tv)-\nabla\Phi(x)\|
      \le
      M t .
    \]
    Divide by $t$ and send $t\downarrow0$ to obtain
    \(
      m\le\|\nabla^2\Phi(x)v\|\le M.
    \)
    Since $v$ is arbitrary, this yields \eqref{eq:hessian}.
\end{proof}

\begin{proposition}[High–probability Lipschitz bound for the regression estimator]\label[proposition]{prop:Lip-bound}
    Under \Crefrange{asm:abs-continuity}{asm:data}, assume that the estimator $\hat\mu$ satisfies the mean--squared error
    \(\displaystyle
      \varepsilon
      :=\|\hat\mu-\mu\|_{L_{2}(\PP_X)}\in(0,1/\sqrt{2}).
    \)    
    Define the data–dependent radius
    \begin{align}
        R_\varepsilon
        :=\lambda\,
          \sqrt{d+2\log(1/\varepsilon)},\label{eq:Reps}
    \end{align}
    and the ball
    \[B:=\{x\in\RR^{d}:\|x\|\le R_\varepsilon\}.\]
    Then there exists a universal constant $C=4\sqrt2$, such that
    \[
        \PP_{Z}(
            \Lip(\hat\mu; B)
            \;\le\;
            L_\mu
            +
            C\,
            L_\mu^{1/2}\,
            \lambda^{-1/2}\,
            [d+2\log(1/\varepsilon)\,]^{-1/4}\,
            \varepsilon^{1/2}
        )
        \;\ge\;
        1-\varepsilon^{2}.
    \]
\end{proposition}
\begin{proof}
    Throughout the proof, we rely on two arguments, {truncation} and a {covering–number} bound.
    We split the proof into multiple steps.
    
    \noindent\textbf{Step~1:  Truncation of the input domain.}
    
    For $Z\sim\cN(0,I_{d})$ the standard tail bound
    \(\PP(\|Z\|>t)\le\exp(-t^{2}/2)\) implies
    \[
        \PP(\|Z\|>R_\varepsilon/\lambda)
        \;\le\;
        \exp(-\tfrac12[d+2\log(1/\varepsilon)])
        \;=\;
        \varepsilon^{2}.
    \]
    Because $X=T^{-1}(Z)$ and $T^{-1}$ is $\lambda$–Lipschitz,
    \(\|X\|\le\lambda\|Z\|\); hence, with probability at least
    \(1-\varepsilon^{2}\), $X\in B$.
    
    \noindent\textbf{Step~2: A finite $h$–net of the truncated ball.}
    Conditioned on event $X\in B$, and fix a mesh size \(h>0\).  The ball \(B\) admits an $h$–net
    \(\cN_{h}=\{x_{1},\dots,x_{N_{h}}\}\) of cardinality
    \(N_{h}\le(1+2R_\varepsilon/h)^{d}\).
    Write \(\Delta_\mu:=\hat\mu-\mu\).
    By Cauchy–Schwarz,
    \begin{align}
        \max_{x_{i}\in\cN_{h}}
              |\Delta_\mu(x_{i})|
        \;\le\;
        \frac{\|\Delta_\mu\|_{L_{2}(\PP_X)}}
             {\sqrt{\PP_X(B)}}
        \;\le\;
        \frac{\varepsilon}{\sqrt{1-\varepsilon^{2}}}
        \;\le\;
        2\varepsilon . \label{eq:Delta-mu-cover}
    \end{align}

    \noindent\textbf{Step 3: Lipschitz bound on the ball \(B\).}
    For arbitrary \(x,y\in B\) choose \(x_{i},x_{j}\in\cN_{h}\) with
    \(\|x-x_{i}\|\le h\) and \(\|y-x_{j}\|\le h\).
    Insert and remove \(\mu\):
    \begin{align*}
        |\hat\mu(x)-\hat\mu(y)|
        &\le
        |\hat\mu(x)-\hat\mu(x_{i})|
        +|\hat\mu(x_{i})-\hat\mu(x_{j})|
        +|\hat\mu(x_{j})-\hat\mu(y)|                                                  \\
        &\le
        L_\mu(\|x-x_{i}\|+\|x_{j}-y\|)
        +|\Delta_\mu(x_{i})|
        +|\Delta_\mu(x_{j})|                                                           \\
        &\le
        L_\mu(\|x-y\|+2h)+4\varepsilon,
    \end{align*}
    where the second line uses the $L_\mu$–Lipschitzness of $\mu$ and the third invokes \eqref{eq:Delta-mu-cover}.  Dividing by $\|x-y\|$ and taking the supremum over \(x,y\in B\) yields
    \begin{align}
        \Lip(\hat\mu;B)
        \;\le\;
        L_\mu
        +2L_\mu h/R_\varepsilon
        +4\varepsilon/h. \label{eq:lip-hat-mu}
    \end{align}

    \noindent\textbf{Step 4: Optimizing the mesh size.}
    Minimizing the right‐hand side of \eqref{eq:lip-hat-mu} over \(h>0\) gives the
    choice
    \[
        h^{*}=\sqrt{\frac{2\varepsilon R_\varepsilon}{L_\mu}}.
    \]
    Substituting it into \eqref{eq:lip-hat-mu} yields that
    \begin{align}
        \Lip(\hat\mu;B)
        &\le
        L_\mu
        +
        4\sqrt{2}\sqrt{\frac{L_\mu\varepsilon}{R_\varepsilon}}
          \nonumber\\
        &=
        L_\mu
        +
        C\,
        L_\mu^{1/2}\,
        \lambda^{-1/2}\,
        [d+2\log(1/\varepsilon)\,]^{-1/4}\,
        \varepsilon^{1/2},
        \label{eq:B-bound}
    \end{align}
    where the last equality is from the definition of $R_{\varepsilon}$ in \eqref{eq:Reps} and $C=4\sqrt{2}$.
    This finishes the proof.
\end{proof}

\subsection{Knothe-Rosenblatt transport}\label{app:sub:KR}

The classical Brenier map is the unique solution of the quadratic optimal-transport problem  
$c(x,y)=\|x-y\|_{2}^{2}$ \citep{brenier1991polar}.
A series of papers by \citet{carlier2010knothe,bonnotte2013knothe} shows that the triangular Knothe-Rosenblatt rearrangement $T_{\KR}$ may be recovered from Brenier maps by anisotropically rescaling the cost:
\[
   c_{\varepsilon}(x,y)
   \;=\;
   \sum_{j=1}^{d}\lambda_{j}(\varepsilon)\,(x_{j}-y_{j})^{2},
   \qquad
\]
where $\lambda_{1}(\varepsilon)\gg\lambda_{2}(\varepsilon)\gg\cdots\gg \lambda_{d}(\varepsilon)$ and $  \lambda_{j+1}(\varepsilon)\,/\,\lambda_{j}(\varepsilon) \rightarrow0$ as $\varepsilon\to0$.
Typical weights are $\lambda_{j}(\varepsilon)=\varepsilon^{\,j-1}$.
Let $T_{\varepsilon}$ be the $\cW_{2}$-optimal map for $c_{\varepsilon}$.
Then, whenever the source measure $\PP$ is absolutely continuous and the target $\QQ$ is non-degenerate, \citet[Theorem 3.1]{carlier2010knothe} implies that
\[
        T_{\varepsilon}\xrightarrow{L_{2}(\PP)}T_{\KR}
        \quad\text{ as }\;\varepsilon\to 0.
\]
Hence $T_{\KR}$ inherits a genuine optimal-transport interpretation: it is the map one obtains when infinitely prioritising the first
coordinate over the second, the second over the third, and so on.
The resulting continuation path $\varepsilon\mapsto T_{\varepsilon}$ provides a homotopy strategy for numerical solvers that bridges the inexpensive $T_{\KR}$ and the fully isotropic Brenier map \citep{carlier2010knothe}.

\medskip
Below, we recall a self-contained existence lemma.

\begin{lemma}[Knothe-Rosenblatt (K-R) transport]\label[lemma]{lem:KR-map}
    For two Borel probability measures $\PP$ and $\QQ$ on $\RR^{d}$, consider the transport map $T:X\mapsto Z$ such that $T_\sharp\PP=\QQ$:
    \[
       T(x_1,\dots,x_d)
          =\left(T_1(x_1),\;T_2(x_1,x_2),\;\dots,\;
                 T_d(x_1,\dots,x_d)\right)
       \quad\text{with}\quad T_\sharp\PP=\QQ ,
    \]
    where each coordinate is given by:
    \begin{align*}
        T_1(x_1)
           &= F_{Z_1}^{-1}\left( F_{X_1}(x_1) \right),\\[3pt]
         T_j(x_{1:j})
           &= F_{Z_j\mid Z_{1:(j-1)}}
              ^{-1} \left(
                   F_{X_j\mid X_{1:(j-1)}}(x_j\mid x_{1:(j-1)})
                   \,\middle|\,
                   z_{1:(j-1)}
               \right), \quad 2\le j\le d,
    \end{align*}
    and $z_{1:(j-1)}:=T_{1:(j-1)}(x_{1:(j-1)})$.

    If (i) $\PP$ is absolutely continuous and (ii) each one-dimensional \emph{conditional} cdf of $\QQ$ is continuous and strictly increasing, then $T$ is well-defined and unique within the monotone triangular class $\PP$-almost surely.
\end{lemma}
\begin{proof}

We split the proof into 2 steps.

\noindent\textbf{Step 1: Existence.}
The absolute continuity of $\PP$ possesses a joint density $f(x)$; in particular every conditional distribution $F_{X_j\mid X_{1:j-1}}(\,\cdot\,\mid x_{1:j-1})$ is well defined.
Thus, we may factorise $\PP$ sequentially:
\[
   \rd\PP(x)
      = f_{X_1}(x_1)\;
        f_{X_2\mid1}(x_2\mid x_1)\;
        \cdots\;
        f_{X_d\mid1:(d-1)}(x_d\mid x_{1:d-1})\,
        \rd x .
\]
Choose coordinates of $T$ by matching conditional quantiles:
\begin{align*}
    T_1(x_1)
       &= F_{Z_1}^{-1}\left( F_{X_1}(x_1) \right),\\[3pt]
     T_j(x_{1:j})
       &= F_{Z_j\mid Z_{1:(j-1)}}
          ^{-1} \left(
               F_{X_j\mid X_{1:(j-1)}}(x_j\mid x_{1:(j-1)})
               \,\middle|\,
               z_{1:(j-1)}
           \right), \quad 2\le j\le d,
\end{align*}
where $z_{1:(j-1)}:=T_{1:(j-1)}(x_{1:(j-1)})$.
Each step is well-defined because the inner cdf is continuous and the outer inverse cdf exists and is unique by the second assumption.  
Induction shows $T_\sharp\PP=\QQ$, establishing the existence.

\noindent\textbf{Step 2: Uniqueness (within the monotone triangular class).}

Restrict to maps satisfying $x_j\;\mapsto\;T_j(x_{1:j}) $ is non-decreasing for every fixed $x_{1:(j-1)}$.
Let $T$ and $\tilde T$ be two such transports.  
Because both first coordinates push $X_1$ to the same measure $\QQ_1$ and both are monotone, we must have $T_1(x_1)=\tilde T_1(x_1)$ $\PP$-almost surely.
By induction, we have $T_j=\tilde T_j$ for $j=2,\dots,d$ $\PP$-almost surely.
Hence, uniqueness holds in this class.
\end{proof}

When $\QQ=\Unif[0,1]^{\otimes d}$, the following lemma provides efficient influence function for $\theta_{jl}(\PP)= \EE[(\partial_{z_l}S_j(Z))^2]$, where $S:=T^{-1}$.

\begin{lemma}[Knothe-Rosenblatt weight EIF]
\label[lemma]{lem:eif-Rosenblatt-mixed}

Let $X=(X_1,\dots ,X_d)^\top$ have a joint distribution $\PP$
that admits a Lebesgue density. Define
\[
\begin{aligned}
   T(X)
      &=\bigl(
          F_{1}(X_1),\;
          F_{2\mid1}(X_2\mid X_1),\;
          \dots,\\
      &\qquad
          F_{d\mid1:(d-1)}\left(X_d\mid X_{1:(d-1)}\right)
        \bigr)\\
      &=:Z\sim\Unif(0,1)^{\otimes d}.
\end{aligned}
\]
Write $S:=T^{-1}$.
For any pair of indices $(j,l)\in[d]^2$ define $H_{jl}(X)
:=\left\{\partial_{z_l}S_j(Z)\right\}^{2}$ and $ \theta_{jl}(\PP):=\EE_{\PP}\!\left[ H_{jl}(X) \right]$.
\begin{enumerate}[(i)]
    \item If $l>j$ then $H_{jl}(X)\equiv 0$, so $\theta_{jl}(\PP)=0$ and the efficient influence function is identically~$0$.

    \item If $l\le j$, the parameter is non-degenerate and its efficient influence function (EIF) is
    \begin{equation}\label{eq:eif-mixed}
       \varphi_{jl}(x)
          = H_{jl}(x)-\theta_{jl}
            \;-\; 2\,H_{jl}(x)\,
                  \left\{
                      \ind\{X_l\le x_l\}
                      -F_{l\mid1:(l-1)}\!\left(x_l\mid x_{1:(l-1)}\right)
                  \right\},
    \end{equation}
    which attains the semiparametric efficiency bound.
\end{enumerate}
\end{lemma}
\begin{proof}
    If $l>j$, because the inverse K-R map is lower-triangular, $S_j$ depends only on $(z_1,\dots ,z_j)$, hence $\partial_{z_l}S_j\equiv 0$ whenever $l>j$; all assertions are immediate.

    Next, we analyze the second case in two steps.
    First, we introduce some notation.
    Fix $(j,l)$ with $l\le j$.
    Denote the lower-triangular Jacobian $J_{T}(X):=\left[\partial_{x_k}T_j(X)\right]_{j,k\le d}$ and
    \[
    \begin{aligned}
    d_j(X)
        &:=J_{T}(X)_{jj}
          =f_{\,j\mid 1:(j-1)}\!\left(X_j\mid X_{1:(j-1)}\right),\\
    b_{jk}(X)
        &:=J_{T}(X)_{jk}
          =\partial_{x_k}F_{\,j\mid 1:(j-1)}
            \!\left(X_j\mid X_{1:(j-1)}\right),
          \quad k<j.
    \end{aligned}
    \]    
    Because $J_{T}(X)$ is unit-triangular up to the positive diagonal $\{d_j\}$, its inverse $J_{S}(Z)=J_{T}^{-1}(X)=\left[\gamma_{jk}(X)\right]_{j,k\le d}$ is also lower-triangular.

    Set
    \[
       W:=X_{1:(l-1)},\qquad
       V:=X_l,\qquad
       U:=X_{(l+1):d},
    \]
    so that $X=(W,V,U)$ and the factorisation of $\PP$ reads $f_{X} = f_W  f_{V\mid W} f_{U\mid W,V}$.    
    Because $S_j$ depends only on $X_{1:j}$, $H_{jl}$ may be viewed as a measurable functional of $(W,V,U_{1:(j-l)})$; its exact algebraic form, is multiplicatively proportional to $d_{l}(X)^{-2}$:
    \begin{equation}\label{eq:H-factorisation}
          H_{jl}(X)=\gamma_{jl}(X)^{2}=R_{jl}(X)d_l(X)^{-2},
          \qquad R_{jl}(X)=\left\{d_l(X)\,\gamma_{jl}(X)\right\}^{2}.
    \end{equation}
    Alternatively, because the inverse-Jacobian entries satisfy
    \[
        \gamma_{ll}(X)=\frac{1}{d_l(X)},\qquad
        \gamma_{jl}(X)=-\frac{1}{d_j(X)}
                   \sum_{k=l}^{j-1}b_{jk}(X)\,\gamma_{kl}(X),
    \]
    the weight $R_{jl}$ can be computed recursively:
    \[
       R_{ll}(X)=1,\qquad
       R_{jl}(X)=
          \left[
            -\sum_{k=l}^{j-1}
               \frac{b_{jk}(X)}{d_j(X)}
               \sqrt{R_{kl}(X)}
          \right]^{2},\qquad l<j.
    \]
    The factor $R_{jl}$ depends on lower-order conditional cumulative distribution functions only through the values of $(W, V, U)$ and therefore behaves as a fixed, bounded weight in the influence-function calculation.

    \paragraph{Pathwise differentiability.}
    Let $\{\PP_t:|t|<\varepsilon\}$ be a regular sub-model with score
    $\dot\ell(X)$.
    Write  
    $f_{V\mid W,t}$ for the conditional density of $V$ given $W$ under
    $\PP_t$ and set $\dot f_{V\mid W}:=\partial_t f_{V\mid W,t}|_{t=0}$.
    A simple differentiation gives
    \[
       \dot f_{V\mid W}(v\mid w)
            = f_{V\mid W}(v\mid w)
              \left\{\dot\ell(w,v,U)-\EE[\dot\ell(X)\mid W=w]\right\}.
    \]
    Hence, differentiating  
    $\theta_{jl}(t):=\EE_{\PP_t}[H_{jl}(X)]$ and using
    \eqref{eq:H-factorisation},
    \[
       \dot\theta_{jl}
         =\EE\!\left[
             \dot H_{jl}(X)
             \;+\;
             H_{jl}(X)\,\dot\ell(X)
           \right]
         =\EE\!\left[
             \left\{-2H_{jl}(X)\,\eta(X)+H_{jl}(X)-\theta_{jl}\right\}
             \dot\ell(X)
           \right],
    \]
    where
    $\eta(X):=\dot\ell(X)-\EE\!\left[\dot\ell(X)\mid W\right]=\dot\ell(X)-\EE[\dot\ell(X)\mid X_{1:(l-1)}]$.
    Now note the well-known representation  
    \(
        \eta(X)
          =\dfrac{\partial}{\partial v}
           \left\{\ind\{V\le v\}
                  -F_{V\mid W}(v\mid W)
           \right\}\Big|_{v=V};
    \)
    integrating once in $v$ yields  
    \(
        \eta(X)
          = f_{V\mid W}(V\mid W)\,
            \left\{\ind\{V\le x_l\}
                  -F_{V\mid W}(x_l\mid W)
            \right\},
    \)
    so that  
    $-2H_{jl}(X)\eta(X)=
     -2H_{jl}(X)\{\ind\{X_l\le x_l\}-F_{l\mid1:(l-1)}(x_l\mid x_{1:(l-1)})\}$.
    Substituting back proves that \eqref{eq:eif-mixed} is a centered influence function.
    
    
    \paragraph{Efficiency.}
    Because $\varphi_{jl}$ lies in the tangent space and reproduces the pathwise derivative for every sub-model, it is the unique canonical gradient.
    Square integrability follows from $\EE[R_{jl}(X)]<\infty$ and $\EE[f_{V\mid W}^{-4}(V\mid W)]<\infty$ (which holds whenever the joint density is locally bounded and strictly positive on compact sets).
\end{proof}

Some comments on \Cref{lem:KR-map} follow.

Firstly, the proof of \Cref{lem:KR-map} hinges only on the multiplicative decomposition \eqref{eq:H-factorisation}; it does not require an explicit closed form for $R_{jl}$. 
Hence, the lemma extends to any smooth functional of the K-R map whose density-dependence enters solely through
$f_{l\mid1:(l-1)}^{-2}$.

Secondly, the EIF is identically zero for $l>j$, reflecting the fact that the triangular Rosenblatt inverse does not respond to perturbations in later coordinates when reconstructing an earlier component.

\clearpage
\section{Simulation details}\label{app:sec:simu}
    \subsection{Implementation details}\label{app:subsec:implement}
    \subsubsection{Semicontinuous EOT plan estimation}

    Given centered data \(X \in \RR^{n \times d}\) with sample covariance \( \hat\Sigma = X^\top X / n \), we compute the eigendecomposition \(\hat\Sigma = V \Lambda V^\top\).
    When \(\hat\Sigma\) is full rank, define the whitening matrix \( L^{-1} = V \Lambda^{-1/2} V^\top \) and the square-root matrix \( L = V \Lambda^{1/2} V^\top \).
    The whitened data are \( \tilde X = X L^{-1} \).
    When \(n<d\), we instead work in the rank-\(k\) principal subspace of \(X\), where
    \(k=\mathrm{rank}(X)\), using the truncated decomposition \( \hat\Sigma = V_k \Lambda_k V_k^\top \),
    and define \( L^{-1} = V_k \Lambda_k^{-1/2} \), \( L = V_k \Lambda_k^{1/2}, \) so that \(\tilde X \in \RR^{n\times k}\).

    We consider the semi-continuous entropy-regularized optimal transport problem with source and target:
    \[
    \hat\mu = \frac1n \sum_{i=1}^n \delta_{\tilde X_i},
    \qquad
    \nu = \cN_k(0,I_k),
    \]
    under cost \( c(x,z)=\|x-z\|_2^2, \) and entropic regularization parameter \(\varepsilon>0\). Let \(\gamma_\varepsilon \in \Pi(\hat\mu,\nu)\)
    denote the corresponding EOT coupling. Its Gibbs form is
    \[
    \gamma_\varepsilon(\rd x,\rd z)
    =
    \frac1n \sum_{i=1}^n p_{\varepsilon,i}(z)\,\delta_{\tilde X_i}(\rd x)\,\rd z,
    \]
    where
    \[
    p_{\varepsilon,i}(z)
    \propto
    \exp\!\left(
    f_i-\frac{\|\tilde X_i-z\|_2^2}{\varepsilon}
    \right)\varphi_k(z),
    \]
    \(f_1,\ldots,f_n\) are the source Schr\"odinger potentials and \(\varphi_k\) is standard Gaussian density on \(\RR^k\).

    Because the target is Gaussian and the cost is quadratic, the conditional law
    \(Z\mid X=\tilde X_i\) is available in closed form. Completing the square yields
    \[
    Z \mid X=\tilde X_i
    \sim
    \cN_k\!\left(
    s_\varepsilon \tilde X_i,\;
    \tau_\varepsilon I_k
    \right),
    \qquad
    s_\varepsilon := \frac{2}{2+\varepsilon},
    \qquad
    \tau_\varepsilon := \frac{\varepsilon}{2+\varepsilon}.
    \]
    In particular, the conditional mean does not depend on the fitted potentials \(f_i\). Hence
    the forward barycentric map is explicit:
    \[
    T_\varepsilon(x) := \EE_{\gamma_\varepsilon}[Z\mid X=x] = s_\varepsilon x.
    \]
    Accordingly, we define the latent representation by
    \[
    Z_i = T_\varepsilon(\tilde X_i)=s_\varepsilon \tilde X_i,
    \qquad i=1,\ldots,n.
    \]
    If one wishes to generate full draws from the semi-continuous EOT coupling rather than use
    its barycentric projection, one may sample
    \[
    Z_i^{\mathrm{draw}} = s_\varepsilon \tilde X_i + \tau_\varepsilon^{1/2}\xi_i,
    \qquad
    \xi_i \stackrel{\mathrm{iid}}{\sim} \cN_k(0,I_k).
    \]

    Compared with the discrete-target approximation, this semi-continuous formulation requires
    neither Monte Carlo target points nor Sinkhorn iterations on an \(n\times m\) cost matrix.
    The forward latent variables are obtained directly from the closed-form conditional mean.

    In practice, the estimation of \eqref{def:hat-phi-Z} can still be stabilized by Monte Carlo
    integration over the intervened coordinate. More specifically, we define the LOCO-type and
    CPI-type estimators as
    \begin{align*}
        \hat{\phi}_{Z_j}^{\loco}(\PP)
        &=
        \PP_n\!\left[
        \bigl(Y-\PP_{Z_j^{(j)},m}\{\hat\eta(Z^{(j)})\}\bigr)^2
        -
        \bigl(Y-\hat\eta(Z)\bigr)^2
        \right],
    \end{align*}
    and
    \begin{align*}
        \hat{\phi}_{Z_j}^{\cpi}(\PP)
        &=
        \frac12
        \PP_n\!\left\{
        \PP_{Z_j^{(j)},m}
        \left[
        \bigl(Y-\hat\eta(Z^{(j)})\bigr)^2
        -
        \bigl(Y-\hat\eta(Z)\bigr)^2
        \right]
        \right\}.
    \end{align*}
    Here \(\PP_n\) denotes the empirical average over the \(n\) observations, and
    \(\PP_{Z_j^{(j)},m}\) denotes the empirical average over \(m\) draws of the intervened
    coordinate \(Z_j^{(j)}\). In our implementation, we use \(\hat{\phi}_{Z_j}^{\loco}(\PP)\)
    with \(m=50\) by default.

    \subsubsection{Semicontinuous backward attribution}\label{app:subsec:EOT}

    For Gaussian covariates, the Gaussian Schr\"odinger bridge gives a closed-form decoder \(\EE[X\mid Z=z]\). For non-Gaussian covariates, the full decoder is generally nonlinear.
    However, under the empirical-source / Gaussian-target semi-continuous formulation above, the forward conditional \(Z\mid X\) is explicit for any source distribution.
    Although the semi-continuous formulation yields a closed-form forward conditional
    \[
    Z \mid X=x_i \sim \cN_k(s_\varepsilon x_i,\tau_\varepsilon I_k),
    \qquad
    s_\varepsilon=\frac{2}{2+\varepsilon},
    \qquad
    \tau_\varepsilon=\frac{\varepsilon}{2+\varepsilon},
    \]
    it does not imply a simple Gaussian expression for the backward conditional \(X\mid Z=z\)
    in general. Instead, since the source measure is empirical,
    \[
    \hat\mu=\frac1n\sum_{i=1}^n \delta_{\tilde X_i},
    \]
    the backward conditional is a discrete posterior on the observed support points:
    \[
    \PP_{\gamma_\varepsilon^*}(X=\tilde X_i \mid Z=z)
    =
    \alpha_i(z),
    \qquad
    \alpha_i(z)
    =
    \frac{\exp\!\left(f_i-\|\tilde X_i-z\|_2^2/\varepsilon\right)}
    {\sum_{r=1}^n \exp\!\left(f_r-\|\tilde X_r-z\|_2^2/\varepsilon\right)}.
    \]
    Therefore the exact backward barycentric projection is
    \[
    S_\varepsilon(z):=\EE_{\gamma_\varepsilon^*}[X\mid Z=z]
    =
    \sum_{i=1}^n \alpha_i(z)\,\tilde X_i.
    \]
    Moreover, differentiating the softmax weights gives
    \[
    \nabla_z S_\varepsilon(z)
    =
    \frac{2}{\varepsilon}
    \left\{
    \sum_{i=1}^n \alpha_i(z)\,\tilde X_i \tilde X_i^\top
    -
    S_\varepsilon(z)S_\varepsilon(z)^\top
    \right\}
    =
    \frac{2}{\varepsilon}\Cov_{\gamma_\varepsilon^*}(\tilde X\mid Z=z).
    \]
    Hence the exact backward map and its Jacobian are available in closed form at the sample
    level. In practice, however, evaluating \(S_\varepsilon(z)\) and \(\nabla_z S_\varepsilon(z)\)
    for every \(z\) requires the Schr\"odinger potentials \(f_i\) and repeated softmax summation
    over all \(n\) source points. 
    
    To obtain a simpler and more stable summary for downstream
    attribution, we approximate the backward map $S_{\varepsilon}(z)$ by a global linear projection \(M_w\):
    \[
    M_w
    :=
    \argmin_{M\in\RR^{k\times k}}
    \EE_{\gamma_\varepsilon}\!\left[
    \|\tilde X - M Z\|_2^2
    \right]
    =
    \EE_{\gamma_\varepsilon}[\tilde X Z^\top]
    \Bigl(\EE_{\gamma_\varepsilon}[ZZ^\top]\Bigr)^{-1}.
    \]
    Using the closed-form conditional
    \(
    Z\mid X=\tilde X_i \sim \cN_k(s_\varepsilon \tilde X_i,\tau_\varepsilon I_k)
    \),
    we obtain
    \[
    \EE_{\gamma_\varepsilon}[\tilde X Z^\top]
    =
    \frac1n \sum_{i=1}^n \tilde X_i \EE[Z^\top\mid \tilde X_i]
    =
    s_\varepsilon \hat\Sigma_w,
    \]
    where
    \(
    \hat\Sigma_w := \frac1n \sum_{i=1}^n \tilde X_i \tilde X_i^\top
    \).
    Likewise,
    \[
    \EE_{\gamma_\varepsilon}[ZZ^\top]
    =
    \frac1n \sum_{i=1}^n \EE[ZZ^\top\mid \tilde X_i]
    =
    \tau_\varepsilon I_k + s_\varepsilon^2 \hat\Sigma_w.
    \]
    Therefore the backward map admits the closed form
    \begin{equation}\label{eq:Mw-semicont}
        M_w
        =
        s_\varepsilon \hat\Sigma_w
        \bigl(\tau_\varepsilon I_k + s_\varepsilon^2 \hat\Sigma_w\bigr)^{-1}.
    \end{equation}
    Since \(\tau_\varepsilon>0\), the matrix
    \(
    \tau_\varepsilon I_k + s_\varepsilon^2 \hat\Sigma_w
    \)
    is always invertible, even when \(k<d\) or \(\hat\Sigma_w\) is rank-deficient. Thus the
    semi-continuous backward map is numerically stable without any additional debiasing or
    trace normalization.

    The corresponding attribution matrix in the original feature space is
    \(
    W = L M_w.
    \)
    When \(L\in\RR^{d\times k}\) is obtained from the truncated eigendecomposition, we have
    \(W\in\RR^{d\times k}\), and the original-feature importance is mapped back by
    \(
    \hat\phi_X = (W\odot W)\,\hat\phi_Z,
    \) where \(\odot\) denotes elementwise multiplication.

    \subsection[Model M3]{Detailed calculation under Model \hyperref[M3]{(M3)}}\label{app:subsec:M3}

    We derive the theoretical DFI values for the latent features $Z$ and original features $X$ under the specified model.
    For a Gaussian $X$, the optimal transport map (EOT with $\varepsilon=0$) to $Z \sim \cN_5(0, I_5)$ is the linear transformation $Z = \Sigma^{-1/2}X$. The inverse map is $X = LZ$, where $L = \Sigma^{1/2}$. Given the block-diagonal structure of $\Sigma$, $L$ is also block-diagonal:
    \begin{align*}
        X_1 &= aZ_1 + bZ_2 & X_4 &= aZ_4 + bZ_5 \\
        X_2 &= bZ_1 + aZ_2 & X_5 &= bZ_4 + aZ_5 \\
        X_3 &= Z_3
    \end{align*}
    where $a = \frac{1}{2}(\sqrt{1+\rho}+\sqrt{1-\rho})$ and $b = \frac{1}{2}(\sqrt{1+\rho}-\sqrt{1-\rho})$. Note that $a^2+b^2=1$, $ab=\rho/2$.
    
    The regression function $\eta(Z) = \EE[Y \mid Z]$ becomes:
    \begin{align*}
        \eta(Z) &= \eta_1(Z_1, Z_2) \ind_{\{Z_3>0\}} + \eta_2(Z_4, Z_5) \ind_{\{Z_3<0\}},
    \end{align*}
    where $\eta_1(Z_1, Z_2) = 1.5\left[\tfrac{\rho}{2}(Z_1^2+Z_2^2) + Z_1Z_2\right]$ and $\eta_2(Z_4, Z_5) = \tfrac{\rho}{2}(Z_4^2+Z_5^2) + Z_4Z_5$.

    \noindent\textbf{DFI for Latent Features ($\phi_{Z_j}$)}
    The latent importance score is $\phi_{Z_j} = \EE[\VV(\eta(Z) \mid Z_{-j})]$. 
    \begin{itemize}
        \item For $Z_1$ and $Z_2$: By symmetry, $\phi_{Z_1} = \phi_{Z_2}$. The importance arises from the variance of $\eta_1$ conditional on one of its components, averaged over the indicator.
        \begin{align*}
            \phi_{Z_1} = \phi_{Z_2} = \frac{1}{2}\left(\EE[\eta_1^2] - \EE[(\EE[\eta_1 \mid Z_1])^2]\right) = \frac{9}{16}\rho^2 + \frac{9}{8}
        \end{align*}
    
        \item For $Z_3$: The importance is driven by the switching between the two regression components.
        \begin{align*}
            \phi_{Z_3} = \EE\left[\frac{1}{4}(\eta_1 - \eta_2)^2\right] = \frac{1}{4}\left(\VV(\eta_1-\eta_2) + (\EE[\eta_1-\eta_2])^2\right) = \frac{7}{8}\rho^2 + \frac{13}{16}
        \end{align*}
    
        \item For $Z_4$ and $Z_5$: By symmetry, $\phi_{Z_4} = \phi_{Z_5}$ arises from the variance of $\eta_2$.
        \begin{align*}
            \phi_{Z_4} = \phi_{Z_5} = \frac{1}{2}\left(\EE[\eta_2^2] - \EE[(\EE[\eta_2 \mid Z_4])^2]\right) = \frac{1}{8}\rho^2 + \frac{1}{4}
        \end{align*}
    \end{itemize}
    
    \noindent\textbf{DFI for Original Features ($\phi_{X_l}$)}
    The importance is attributed back to the original features via $\phi_{X_l} = \sum_{j=1}^5 L_{lj}^2 \phi_{Z_j}$.
    
    \begin{itemize}
        \item For $X_1$ and $X_2$: The importance is a weighted average of importances of $Z_1$ and $Z_2$.
        \begin{align*}
           \phi_{X_1} = a^2\phi_{Z_1} + b^2\phi_{Z_2} = (a^2+b^2)\phi_{Z_1} = \phi_{Z_1} = \frac{9}{16}\rho^2 + \frac{9}{8} \\
           \phi_{X_2} = b^2\phi_{Z_1} + a^2\phi_{Z_2} = (a^2+b^2)\phi_{Z_1} = \phi_{Z_1} = \frac{9}{16}\rho^2 + \frac{9}{8}
        \end{align*}
    
        \item For $X_3$: The feature is independent and its importance is preserved.
        \begin{align*}
            \phi_{X_3} = 1^2 \cdot \phi_{Z_3} = \frac{7}{8}\rho^2 + \frac{13}{16}
        \end{align*}
    
        \item For $X_4$ and $X_5$: Similar to $X_1, X_2$, the total importance within the correlated block is preserved and distributed.
        \begin{align*}
            \phi_{X_4} = a^2\phi_{Z_4} + b^2\phi_{Z_5} = (a^2+b^2)\phi_{Z_4} = \phi_{Z_4} = \frac{1}{8}\rho^2 + \frac{1}{4} \\
            \phi_{X_5} = b^2\phi_{Z_4} + a^2\phi_{Z_5} = (a^2+b^2)\phi_{Z_4} = \phi_{Z_4} = \frac{1}{8}\rho^2 + \frac{1}{4}
        \end{align*}
    \end{itemize}

In summary, the DFI values as a function of $\rho$ is given by \Cref{tab:M3}.

\begin{table}[!ht]
    \centering
    \begin{tabular}{c|ccccc}
        \toprule
        $j$ &  1 & 2 & 3 & 4 & 5\\\midrule
        $\phi_{Z_j}$ & $\frac{9}{16}\left(\rho^2 + 2\right)$ & $\frac{9}{16}\left(\rho^2 + 2\right)$ & $\frac{1}{16}\left(14\rho^2 + 13\right)$ & $\frac{1}{4}\left(\rho^2 + 2\right)$ & $\frac{1}{4}\left(\rho^2 + 2\right)$ \\
        $\phi_{X_j}$ & $\frac{9}{16}\left(\rho^2 + 2\right)$ & $\frac{9}{16}\left(\rho^2 + 2\right)$ & $\frac{1}{16}\left(14\rho^2 + 13\right)$  & $\frac{1}{8}\left(\rho^2 + 2\right)$  & $\frac{1}{8}\left(\rho^2 + 2\right)$ \\\bottomrule
    \end{tabular}
    \caption{Theoretical DFI values under model \hyperref[M3]{(M3)} with $\varepsilon=0$.}
    \label{tab:M3}
\end{table}

For particular values of $\rho$ as in \Cref{fig:simu-3}, we have the following:
\begin{itemize}
    \item For $\rho = 0.2$: $\phi_{Z_1} = \phi_{Z_2} = \phi_{X_1} = \phi_{X_2} = 1.1475$, $\phi_{Z_3} = \phi_{X_3} =  0.8475$, and $\phi_{Z_4} = \phi_{Z_5} = \phi_{X_4} = \phi_{X_5} =0.255$.

    \item For $\rho = 0.8$: $\phi_{Z_1} = \phi_{Z_2} = \phi_{X_1} = \phi_{X_2}  = 1.485$, $\phi_{Z_3} = \phi_{X_3} = 1.3725$, and $\phi_{Z_4} = \phi_{Z_5} = \phi_{X_4} = \phi_{X_5} = 0.33$.
\end{itemize}

\clearpage
\subsection{Extra simulation results}\label{app:subsec:extra-simu}
    \begin{figure}[!ht]
        \centering
        \includegraphics[width=0.5\linewidth]{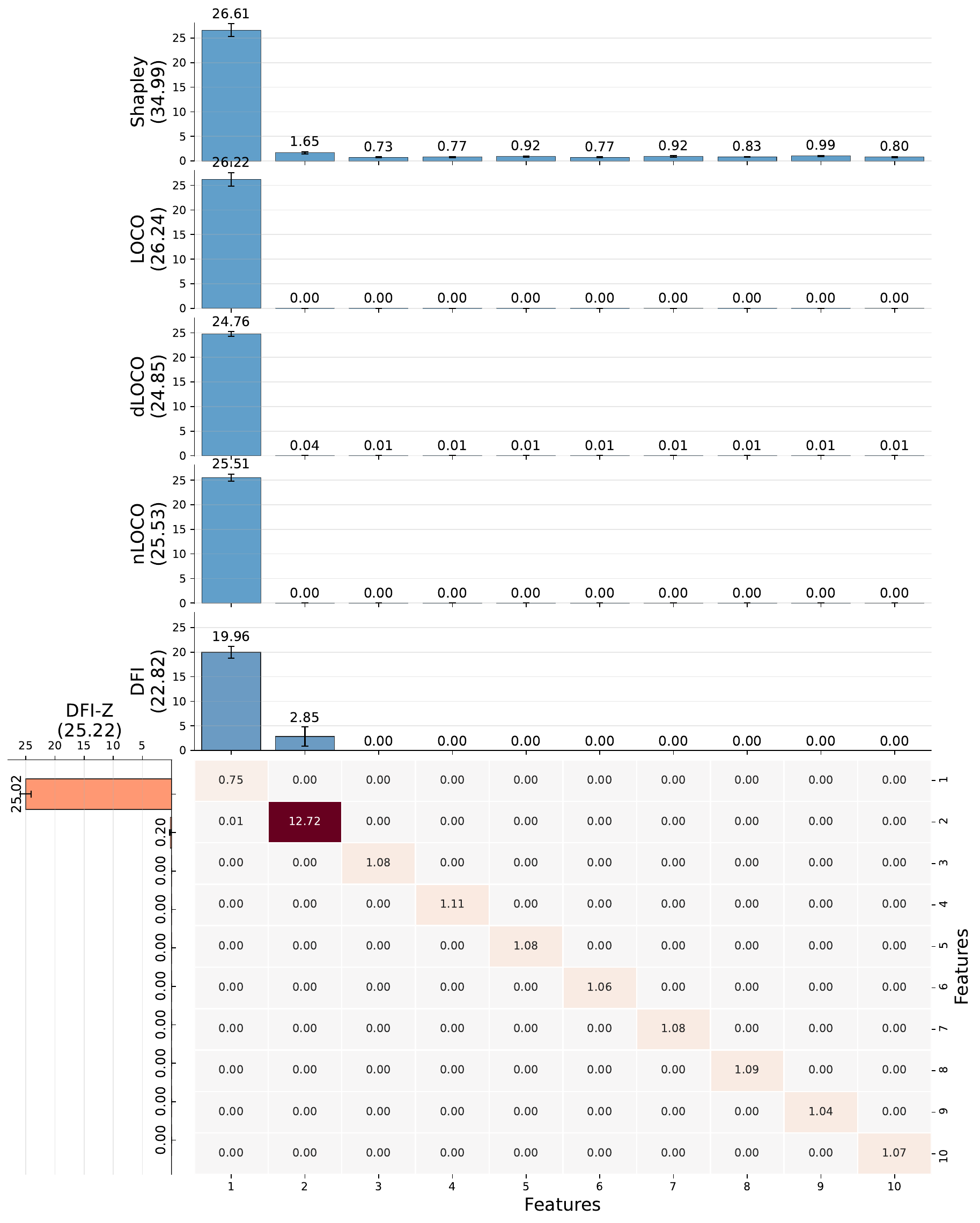}
        \caption{Simulation results for model \hyperref[M4]{(M4)}.
        The interpretation of the bars, error bars, and heat-map is identical to \Cref{fig:simu-1}.
        The number in parentheses is the total estimated importance, which should be close to the signal variance $\VV[\EE[Y\mid X]] = 25$.
        }
        \label{fig:simu-4}
    \end{figure}

    \begin{figure}[!ht]
        \centering
        \includegraphics[width=0.38\linewidth]{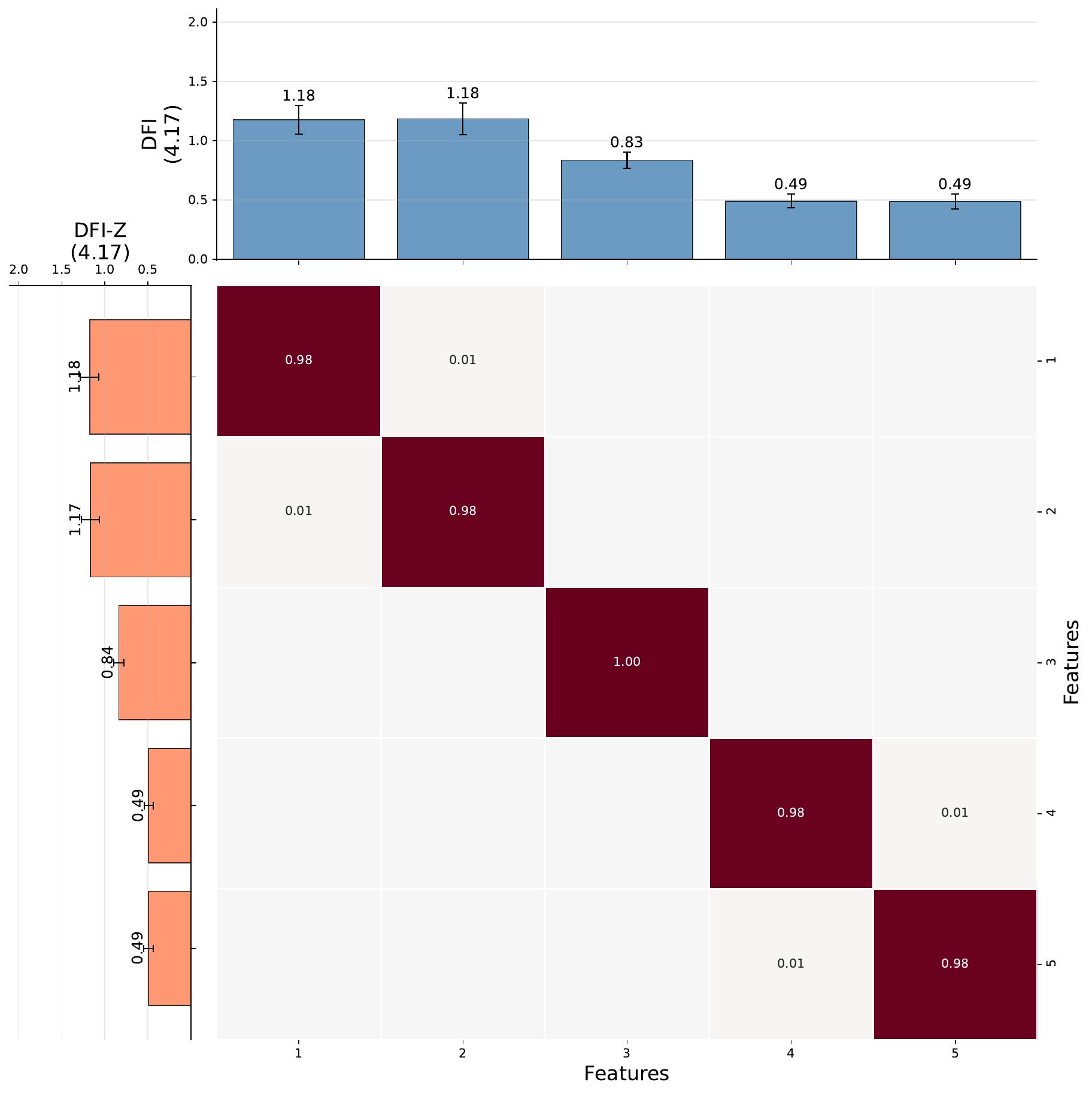}
        \includegraphics[width=0.38\linewidth]{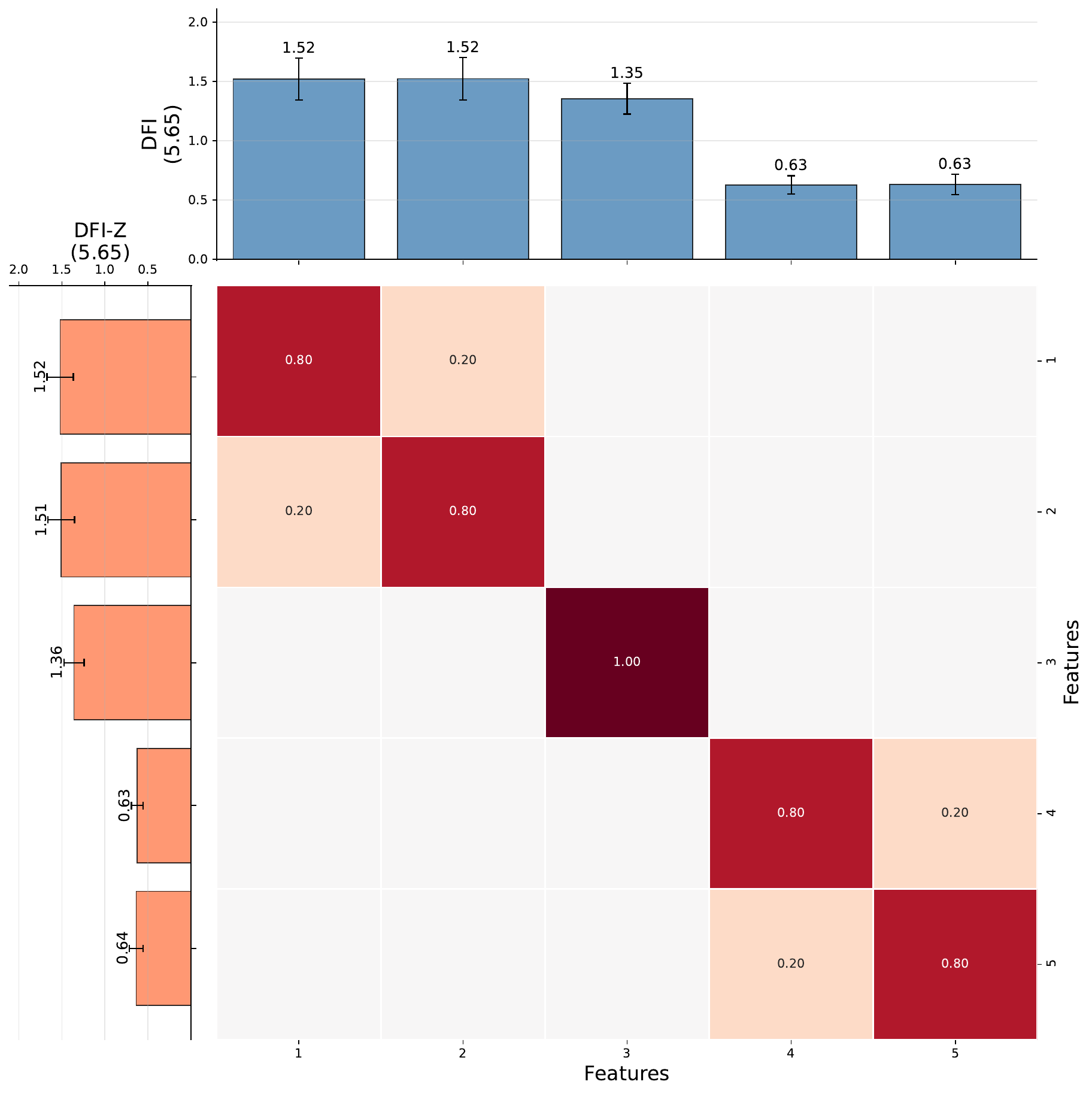}
        \caption{Simulation results for model \hyperref[M3]{(M3)} when the regression function $\mu$ is known.  
        (a) weak correlation ($\rho = 0.2$);  
        (b) strong correlation ($\rho = 0.8$).  
        The interpretation of the bars, error bars, and heat-map is identical to \Cref{fig:simu-1}.
        }
        \label{fig:simu-3-supp}
    \end{figure}

    \clearpage

    \begin{figure}[!ht]
        \centering
        \includegraphics[width=0.95\linewidth]{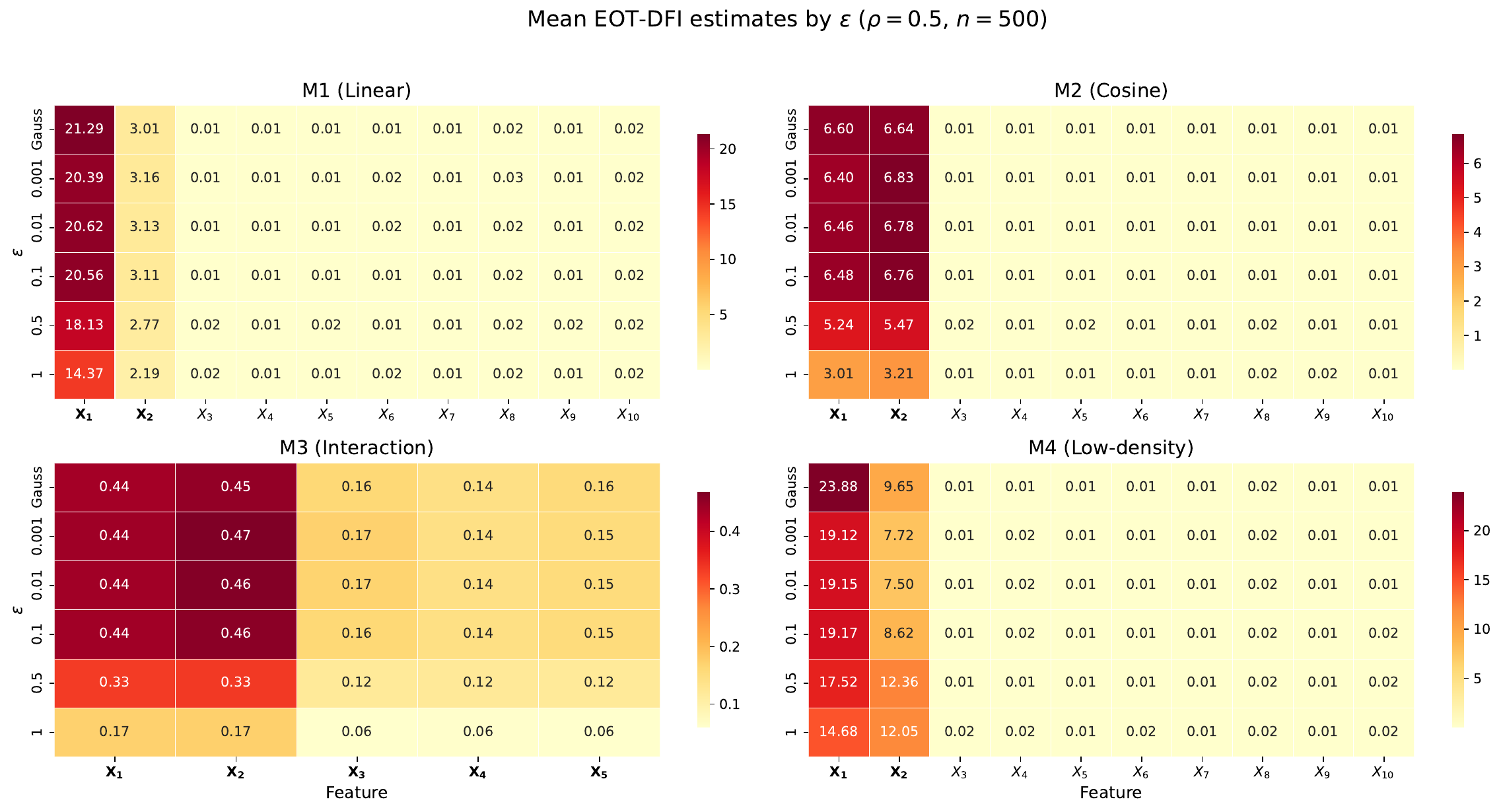}
        \caption{
            Sensitivity of EOT-based attribution to the regularization parameter $\varepsilon$ under model \hyperref[M3]{(M3)} with $\rho = 0.5$. The attribution weights are relatively stable across a range of $\varepsilon$, with the oversampled target method showing less sensitivity than the analytical solution.
        }\label{fig:sensitivity-epsilon}
    \end{figure}

    \begin{figure}[!ht]
        \includegraphics[width=\textwidth]{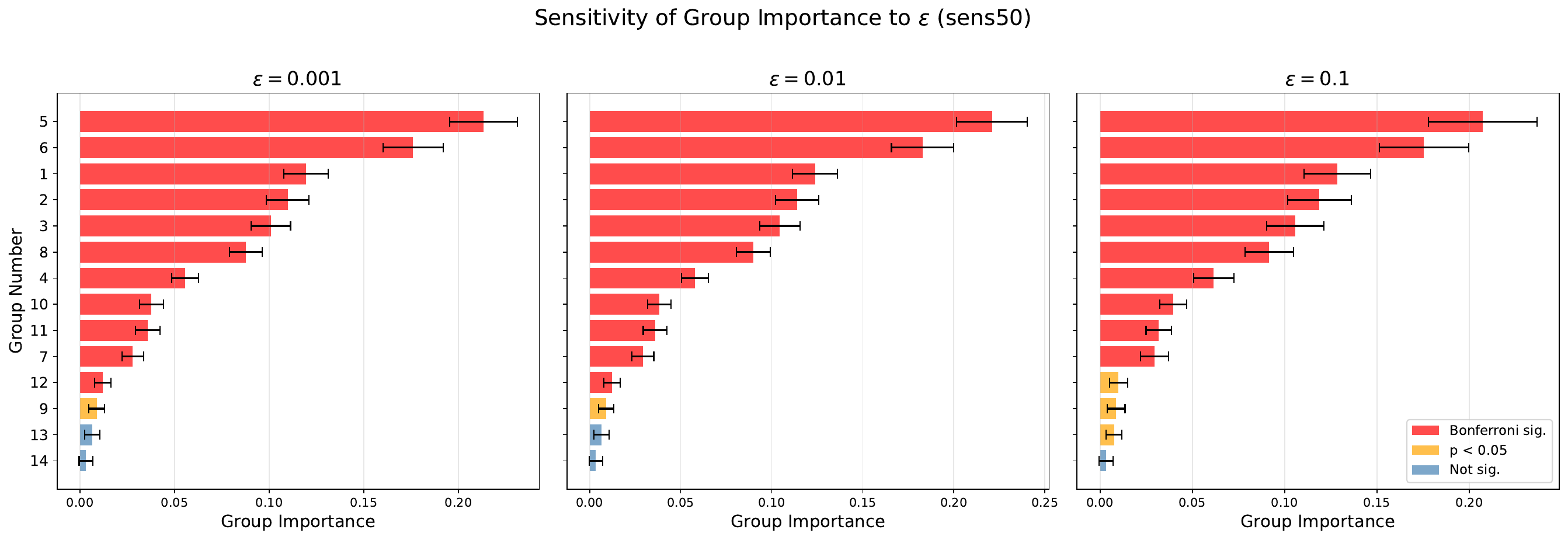}
        \caption{
            Sensitivities of EOT-based attribution to the regularization parameter $\varepsilon$ of importance on antibody group features.
            The ranks are relatively stable across a range of $\varepsilon$.
        }\label{fig:sensitivity-epsilon-sens50}
    \end{figure}

\clearpage
\bibliographystyle{apalike}
\bibliography{ref} 
\end{document}